\documentclass[twoside,11pt]{article}

\usepackage{jmlr2e}

\usepackage{booktabs}       \usepackage{amsfonts}       \usepackage{nicefrac}       

\usepackage{stmaryrd}

\usepackage{thmtools, thm-restate} \usepackage{mdframed} \usepackage{multirow} \usepackage{amsmath}
\usepackage{nccmath} 

\usepackage{float}
\usepackage{subcaption}

\usepackage{vwcol}

\usepackage{pifont} \newcommand{\cmark}{\ding{51}}

\usepackage{algorithm}
\usepackage{algorithmic}

\usepackage{graphicx}
\usepackage[export]{adjustbox} \usepackage{graphbox} \usepackage{array}

\usepackage{enumitem}

\usepackage{wrapfig}
\usepackage{tocloft}
\usepackage{mathtools} \usepackage{bbold} 

\usepackage{xargs} 

\usepackage{tikz}

\usepackage{cleveref}

\newlength{\bibitemsep}
\newlength{\bibparskip}
\let\oldthebibliography\thebibliography
\renewcommand\thebibliography[1]{\oldthebibliography{#1}\setlength{\parskip}{\bibitemsep}\setlength{\itemsep}{\bibparskip}}

        \newtheorem{property}[theorem]{Property}
\newtheorem{inequality}{Inequality}
\newtheorem{assumption}{Assumption}
\newtheorem{model}{Model}
\newtheorem{algo}{Algorithm}
\newtheorem{takeaway}{Take-away}

\crefname{algo}{Algorithm}{Algorithms}
\crefname{model}{Model}{Models}
\crefname{lemma}{Lemma}{Lemmas}
\crefname{fact}{Fact}{Facts}
\crefname{theorem}{Theorem}{Theorems}
\crefname{corollary}{Corollary}{Corollaries}
\crefname{exemple}{Exemple}{Examples}
\crefname{remark}{Remark}{Remarks}
\crefname{property}{Property}{Properties}
\crefname{inequality}{Inequality}{Inequalities}
\crefname{claim}{Claim}{Claims}
\crefname{example}{Example}{Examples}
\crefname{problem}{Problem}{Problems}
\crefname{definition}{Definition}{Definitions}
\crefname{figure}{Figure}{Figures}
\crefname{assumption}{Assumption}{Assumptions}

\crefname{subsection}{Subsection}{Subsections}
\crefname{section}{Section}{Sections}
\crefname{algorithm}{Algorithm}{Algorithms}
\crefname{algocf}{alg.}{algs.}
\Crefname{algocf}{Algorithm}{Algorithms}
\crefname{proposition}{Proposition}{Propositions}

\newcommand{\gs}{\vspace{-0.5em}}

\newcommand{\kcarre}{{\otimes 2}}

\DeclareMathOperator*{\argmin}{arg\,min}

\DeclareUnicodeCharacter{2628}{\lorraine}

\newcommand{\ffrac}[2]{\ensuremath{\frac{\displaystyle #1}{\displaystyle #2}}}

\newcommand{\Artemis}{\texttt{Artemis}}

\newcommand{\Cquant}{\C_\mathrm{q}}
\newcommand{\Cstabquant}{\C_{\mathrm{sq}}}
\newcommand{\Cspars}{\C_{\mathrm{s}}}
\newcommand{\Crand}{\C_{\mathrm{rd}1}}
\newcommand{\Crandh}{\C_{\mathrm{rd}h}}
\newcommand{\Csketch}{\C_{\Phi}}
\newcommand{\Cpp}{\C_{\mathrm{PP}}}

\newcommand{\aniac}{\mathfrak{C}_{\mathrm{ania}}}

\newcommand{\qqquad}{\qquad\qquad}
\newcommand{\I}{\mathcal{I}}

\newcommand{\FF}{\mathcal{F}}
\newcommand{\GG}{\mathcal{G}}
\newcommand{\HH}{\mathcal{H}}

\newcommand{\E}{\mathbb{E}}
\newcommand{\N}{\mathbb{N}}
\newcommand{\R}{\mathbb{R}}

\newcommand{\WW}{\R^d}

  \newcommand{\lnrm}{\left \|} \newcommand{\rnrm}{\right \|} \newcommand{\normeucl}[1]{\| #1 \|_2}

\newcommand{\vertiii}[1]{{\vert\kern-0.25ex\vert\kern-0.25ex\vert #1 
    \vert\kern-0.25ex\vert\kern-0.25ex\vert}}

\newcommand{\FullExpec}[1]{\E \left[#1\right]} \newcommand{\fullexpec}[1]{\E [#1]} 

\newcommand{\Expec}[2]{\E \left[#1~\middle|~#2\right]} \newcommand{\expec}[2]{\E [#1~|~#2]} 

 \newcommand{\Cov}[1]{\mathrm{Cov}\left[#1\right]}
\newcommand{\PdtScl}[2]{\left\langle#1,#2\right\rangle}  \newcommand{\pdtscl}[2]{\langle#1,#2\rangle}  \newcommand{\SqrdNrm}[1]{ \lnrm #1\rnrm^2} 
\newcommand{\sqrdnrm}[1]{ \| #1\|^2} \newcommand{\bigpar}[1]{\left( #1 \right)}

\newcommand{\omgC}{\omega}\newcommand{\omgCOne}{\Omega}
\newcommand{\omgCTwo}{(\omgC + 1)}

\newcommand{\ws}{w_*}
\newcommand{\g}{\textsl{g}} \newcommand{\gw}{g} 
\newcommand{\wkm}{w_{k-1}}

\newcommand{\gwki}{\g_{k}^i(\wkm)}

\newcommand{\gwkstar}{\gw_{k,*}}

\newcommand{\xik}{\xi_k(\etakm)}

\newcommand{\ustar}{u_k}

\newcommand{\xikstaru}{\xi_k^{\mathrm{add}}}\newcommand{\xikstaruk}{\xi_k^{\mathrm{add}}}

\newcommand{\xikmultiplicatif}{\xi^{\mathrm{mult}}}
\newcommandx*\xikmultFNM[2][1=,2=\ustar,usedefault]{\xikmultiplicatif_k \bigpar{#1, u_k, #2}}

\newcommandx*\xikmultFN[2][1=,2=\ustar,usedefault]{\xikmultiplicatif_k \bigpar{#1}}

\newcommand{\etakm}{\eta_{k-1}}
\newcommand{\etabarkm}{\overline{\eta}_{K-1}}

\newcommand{\xCov}{H}
\newcommand{\epsiXCov}{\xCov}

\newcommandx*\compCov[3][1=,2=,3=\xCov,usedefault]{\mathfrak{C}(\C_{#2}^{#1}, p_{{#3}_{#1}})}
\newcommandx*{\compCovPoly}[3][1=,2=,3=,usedefault]{\mathfrak{C}_{\R[\cdot]}(\C_{#3}, p_{#1\xCov_{#2}})}

\newcommand{\Fhess}{H_{F}}
\newcommand{\Ftrace}{R_{F}}

\newcommandx*\FLcompCov[2][1=,2=,usedefault]{\overline{\mathfrak{C}(\C_{#2}, (p_{#1\epsiXCov_{i}})_{i=1}^N)}}
\newcommandx*{\FLcompCovPoly}[1][1=,usedefault]{\overline{\mathfrak{C}_{\R[\cdot]}(#1(\xCov_i)_{i=1}^N)}}

\newcommand{\bigdim}{d}
\newcommand{\subdim}{h}

\newcommand{\iN}{{i=1}^N}
\newcommand{\OneToN}{[N]}\newcommand{\OneToK}{[K]}

\newcommand{\C}{\mathcal{C}}

\newcommand{\boundAdd}{\mathcal{A}}\newcommand{\boundMult}{\mathcal{M}_2}\newcommand{\boundMultPrime}{\mathcal{M}_1}

\newcommand{\ShaA}{\Sha_{\mathrm{add}}}
\newcommand{\ShaM}{\Sha_{\mathrm{mult}}}

\newcommand{\Tr}[1]{\mathrm{Tr}\bigpar{#1}}

\newcommand{\Diag}[1]{\mathrm{Diag}\bigpar{#1}}

\newcommand{\eig}[1]{\mathrm{eig}(#1)}
\newcommand{\Id}{\mathrm{I}}

\newcommand{\cst}{C}

\newcommand{\Flast}{\FF}
\newcommand{\Fx}{\GG}
\newcommand{\Fy}{\HH}
\newcommand{\Fsto}{\mathcal{I}}

\newcommand{\sign}{\mathrm{sign}}
\newcommand{\Bern}{\mathrm{Bern}}

\DeclareFontFamily{U}{wncy}{}
\DeclareFontShape{U}{wncy}{m}{n}{<->wncyr10}{}
\DeclareSymbolFont{mcy}{U}{wncy}{m}{n}
\DeclareMathSymbol{\Sha}{\mathord}{mcy}{"58}

\makeatletter
\newcommand{\oset}[3][0ex]{\mathrel{\mathop{#3}\limits^{
    \vbox to#1{\kern-2\ex@
    \hbox{$\scriptstyle#2$}\vss}}}}
\makeatother

\def\lesssim{\oset[-.2ex]{\sim}{\le}}
\def\precsim{\oset[-.3ex]{\sim}{\preccurlyeq}} 
\definecolor{brickred}{rgb}{0.8, 0.25, 0.33}

\definecolor{tabblue}{HTML}{1F77B4}
\definecolor{taborange}{HTML}{FF7F0e}
\definecolor{tabred}{HTML}{d62728}
\definecolor{tabgreen}{HTML}{2ca02c}
\definecolor{tabgray}{HTML}{7f7f7f}

\definecolor{forestgreen}{rgb}{0.13, 0.55, 0.13}

\definecolor{carmine}{rgb}{0.59, 0.0, 0.09}
\newcommandx{\warning}[1]{\textcolor{carmine}{#1}}

 \defcitealias{krizhevsky_learning_2009}{Kri09}
\defcitealias{caruana_kdd-cup_2004}{CTL04} 
\usepackage{lastpage}
\jmlrheading{25}{2024}{1-\pageref{LastPage}}{08/23}{09/24}{23-1040}{Constantin Philippenko and Aymeric Dieuleveut}

\ShortHeadings{Distributed, compressed and averaged least-squares regression}{Application to federated learning}
\firstpageno{1}

\begin{document}

\title{Compressed and distributed least-squares regression: convergence rates with applications to federated learning}

\author{\name Constantin Philippenko \\ \email constantin.philippenko@polytechnique.edu \\\addr École polytechnique, Institut Polytechnique de Paris, CMAP \AND  
\name Aymeric Dieuleveut \\ \email aymeric.dieuleveut@polytechnique.edu \\ \addr École polytechnique, Institut Polytechnique de Paris, CMAP
} 

\editor{Lorenzo Rosasco}

\maketitle

\addtocontents{toc}{\protect\setcounter{tocdepth}{0}} 

\begin{abstract}

In this paper, we investigate the impact of compression on stochastic gradient algorithms for machine learning, a technique widely used in distributed and federated learning.  
We underline differences in terms of convergence rates between several unbiased compression operators, that all satisfy the same condition on their variance, thus going beyond the classical worst-case analysis.
To do so, we focus on the case of least-squares regression (LSR) and analyze a general stochastic approximation algorithm for minimizing quadratic functions relying on a random field. 
We consider weak assumptions on the random field, tailored to the analysis (specifically, expected Hölder regularity), and on the noise covariance, enabling the analysis of various randomizing mechanisms, including compression. We then extend our results to the case of federated learning.

More formally, we highlight the impact on the convergence of the covariance $\aniac$ of the \textit{\underline{a}dditive \underline{n}oise \underline{i}nduced by the \underline{a}lgorithm}.
We demonstrate despite the non-regularity of the stochastic field, that the limit variance term scales with $\Tr{\aniac \Fhess^{-1}}/K$ (where $\Fhess$ is the Hessian of the optimization problem and $K$ the number of iterations) generalizing the rate for the vanilla LSR case where it is $\sigma^2 \Tr{\Fhess \Fhess^{-1}} / K = \sigma^2 d / K$~\citep{bach_non-strongly-convex_2013}. Then, we analyze the dependency of $\aniac$ on the compression strategy and ultimately its impact on convergence, first in the centralized case, then in two heterogeneous FL frameworks.
\end{abstract}


\section{Introduction}
Large-scale optimization \citep{bottou_tradeoffs_2007} has become ubiquitous in today's learning problems due to the incredible growth of data collection. It becomes computationally extremely hard to process a full dataset or even, to store it on a single device \citep{abadi_tensorflow_2016,seide_cntk_2016,caldas_leaf_2019}. This led practitioners to either process each observation only once in a streaming fashion, or to design distributed algorithms. This paper is part of this line of work and considers in particular stochastic federated algorithms \citep[][]{konecny_federated_2016,mcmahan_communication-efficient_2017} that use a central server to orchestrate the training over a network of $N$ in $\N^*$ clients.

A well-identified challenge in this framework is the communication cost of the learning process \citep{seide_1-bit_2014,chilimbi_project_2014,strom_scalable_2015} based on stochastic gradient algorithms. Indeed, iteratively  exchanging gradient or model information between the local workers and the central server generates a huge computational and bandwidth bottleneck.
To reduce this communication cost, two strategies have been widely implemented and analyzed:
performing local updates~\citep[see e.g.~][]{mcmahan_communication-efficient_2017,karimireddy_scaffold_2020}, or reducing the size  of the exchanged messages by passing them through a compression operator, on the uplink channel \citep{seide_1-bit_2014,alistarh_qsgd_2017,alistarh_convergence_2018,mishchenko_distributed_2019,karimireddy_error_2019,wu_error_2018,horvath_natural_2022,mishchenko_distributed_2019,li_acceleration_2020,richtarik_ef21_2021}, or on both uplink and downlink channels \citep{harrane_reducing_2018,tang_doublesqueeze_2019,liu_double_2020,zheng_communication-efficient_2019,philippenko_artemis_2020,philippenko_preserved_2021,gorbunov_linearly_2020,sattler_robust_2019,fatkhullin_ef21_2021}. These two strategies, although typically analyzed independently, are often combined. We focus on compression; to reduce the cost of exchanging a vector, three techniques are combined: (1) sending the message to only a few clients, (2) sending only a fraction of the coordinates, (3) sending low-precision updates.

Most analyses of the impact of compression schema rely on generic assumptions on the compression operator $\C$, typically either \textit{contractive}, i.e.~for any $z$ in $\R^d$, $\|\C(z)-z\| < (1-\delta)\|z\|$ with $\delta \in ]0;1[$ \citep[almost surely or in expectation, see for instance][]{seide_1-bit_2014,stich_sparsified_2018,karimireddy_error_2019,ivkin_communication-efficient_2019,koloskova_decentralized_2019,gorbunov_linearly_2020,beznosikov_biased_2020}, or unbiased with bounded variance increase, i.e., for any $z$ in $\R^d$,  $\E [\C(z)] = z $ and $ \E[\|\C(z)-z\|^2] \leq \omgC \|z\|^2$ for a parameter $\omgC > 1$ \citep[see among others][]{alistarh_qsgd_2017,wu_error_2018,mishchenko_distributed_2019,chraibi_distributed_2019,gorbunov_unified_2020,reisizadeh_fedpaq_2020,horvath_natural_2022,kovalev_lower_2021,philippenko_artemis_2020,philippenko_preserved_2021,haddadpour_federated_2021,li_canita_2021,khirirat_distributed_2018}. Unlike biased---and often deterministic---operators, unbiased operators typically benefit from a variance reduction proportional to the number of clients (e.g., \citealp{gorbunov_linearly_2020} vs \citealp{horvath_stochastic_2019}). 

In parallel, a line of work has thus focused on the design of compression schemes satisfying one of these two assumptions~\citep{bernstein_signsgd_2018,dai_hyper-sphere_2019,beznosikov_biased_2020,horvath_natural_2022,xu_optimal_2020,leconte_dostovoq_2021,gandikota_vqsgd_2021,ramezani_nuqsgd_2021,horvath_natural_2022}.
Two fundamental strategies are typically combined: (1) quantization~\citep{rabbat_quantized_2005,gersho_vector_2012,alistarh_convergence_2018}, and (2) random projection~\citep{vempala_random_2005,rahimi_random_2008,nesterov_efficiency_2012,nutini_coordinate_2015}.
These methods are compared based on (1) the number of bits required for storing or exchanging a $d$ dimensional vector and (2) the resulting  variance increase $\omgC$ or contractiveness constant $\delta$.
Consequently, convergence results are \textit{worst-case} results over the class of compression operators: two compression operators satisfying the same variance assumption are regarded as producing the same convergence rate. 

 \ \\ 
\noindent
\fbox{
\begin{minipage}{\dimexpr\textwidth-4\fboxsep\relax}
The goal of this paper is to provide an in-depth analysis of compression within a fundamental learning framework, namely least-squares regression \citep[LSR,][]{legendre_nouvelles_1806}, 
in order  to highlight the differences in convergence between several unbiased compression schemes having the \textit{same}  variance increase. 
\end{minipage}
}

More precisely, we aim at analyzing updates on a sequence of models $(w_k)_{k\in \N}$ of the form $w_k = \wkm - \frac{\gamma}{N} \sum_\iN \C^i(\g_k^i(\wkm))$, where $\gamma$ is the step-size and $\g_k^i$ is a stochastic oracle on the gradient of the least-squares objective function of client $i$ (see \Cref{ex:cent_comp_LMS,ex:dist_comp_LMS}).

To the best of our knowledge, this study is the first to \emph{compare compressors that are in the same class}, i.e. satisfying
the same variance assumption. Especially, this analysis will highlight the impact of (1) the  compression scheme's regularity (Lipschitz in squared expectation or not) and of (2) the correlation between the compression of the different coordinates.
We highlight three examples of possible take-aways from our analysis, that will be detailed in \Cref{sec:application_compressed_LSR}.
\begin{takeaway}
Quantization-based compression schemes do not have \emph{Lipschitz in squared expectation} regularity but satisfy a Hölder condition. Because of that, their convergence is degraded, yet they asymptotically achieve a rate comparable to projection-based compressors, in which the limit covariance is similar.
\end{takeaway}
\begin{takeaway}
Rand-$h$ and partial participation with probability $(h/d)$ satisfy the same variance condition. Yet the convergence of compressed least mean squares algorithms for partial participation is more robust to ill-conditioned problems.
\end{takeaway}

\begin{takeaway}
The asymptotic convergence rate is expected to be at least as good for quantization than for sparsification or randomized coordinate selection, \emph{if} the features are standardized. On the contrary, if the features are independent and the feature vector is normalized, then quantization is worse than  sparsification or randomized coordinate selection.
\end{takeaway}

We consider a random-design LSR framework and make the following assumption on the input-output pairs distribution
\begin{model}[Federated case]
\label{model:fed}
We consider $N$ clients. Each client $i$ in~$\OneToN$ accesses $K$ in~$\N^*$ i.i.d.~observations $(x_k^i,y_k^i)_{k\in\OneToK}\sim \mathcal{D}_i^{\otimes K}$, such that  there exists a well-defined client-dependent model~$\ws^i$: 
\begin{align}
\label{eq:model}
\forall k \in \OneToK, \quad y_k^i = \PdtScl{x_k^i}{\ws^i}  + \varepsilon_k^i, \qquad \text{with}~~\varepsilon_k^i \sim \mathcal{N}(0, \sigma^2)~\,, 
\end{align}
for an i.i.d.~sequence~$\bigpar{(\varepsilon_k^i)_{k\in \OneToK, i\in\OneToN}}$ independent of $\bigpar{(x_k^i)_{k\in\OneToK,i\in\OneToN}}$. 
We use the generic notation $(x^i, y^i,\varepsilon^i)$ for such an input-output-noise triplet on client $i$.
Moreover, we assume that the inputs' second moment\footnote{In the following, we may refer to this matrix $\xCov$ as the \textit{covariance} (in the case of centered features, covariance is equal to the second moment)} is bounded to define $\E [x^i \otimes x^i] = \xCov_i$ and $\E [\|{x^i}\|^2] = R^2_i$; such that $\E [\|{x^i}\|^2 x^i \otimes x^i] \preccurlyeq R_i^2 \xCov_i$. For any $i \in \OneToN$, we consider the expected squared loss on client $i$ of a model $w$ as $F_i(w) := \frac{1}{2} \E_{(x^i, y^i)\sim\mathcal{D}_i}[(\PdtScl{x^i}{w} - y^i)^2]$.
\end{model}

\begin{remark}[Almost surely bounded features]
\label[remark]{rem:as_bounded_features}
    In the case of linear compressors, we will also assume that for each client $i$ in $\OneToN$, features are almost surely bounded by $R_i^2$.
\end{remark}
 
\noindent
This model is classical in the single worker case \citep[e.g.][]{hsu_random_2012,bach_non-strongly-convex_2013}: 
\begin{model}[Centralized case]
\label{model:centralized}
    We consider \Cref{model:fed} with $N=1$ client. For simplicity, we then omit the $i$ superscript.
\end{model}

We focus on the problem of minimizing the global expected risk $F: \R^d \rightarrow \R$, thus finding the optimal model $\ws$ in $\R^d$ such that:
\begin{align}
\label{eq:def_wstar}
    \ws = \argmin_{w \in \R^d} \left\{ F(w) := \frac{1}{N} \sum_\iN F_i(w) \right\}~~ \tag{OPT}
\end{align}
Note that we assume that $\mathrm{Span}\{\mathrm{Supp}(x^i), i \in {\OneToN}\} =\R^d$ to ensure the existence and uniqueness of~$\ws$.

The empirical version of the risk minimization admits an explicit formula, yet is computationally too expensive to compute for large problems. This is why, in practice, LSR is solved using iterative stochastic algorithms, for example Stochastic Gradient Descent \citep[SGD, see][]{robbins_stochastic_1951}. SGD for LSR is often referred to as the \emph{Least Mean Squares} (LMS) algorithm~\citep{bershad_analysis_1986,macchi_adaptative_1995}. Analysis of LMS~\citep{gyorfi_averaged_1996,bach_non-strongly-convex_2013} and its variants received a lot of interest over the last decades. Indeed despite its simplicity, LSR  is a model of choice for practitioners because of its efficiency to train good and interpretable models \citep[see e.g.][chapter 5.1]{molnar_interpretable_2020}. Moreover, its simplicity enables to isolate and analyze challenges faced in specific configurations, for instance, non-strong convexity \citep{bach_non-strongly-convex_2013}, interaction between acceleration and stochasticity \citep{dieuleveut_harder_2017,jain_accelerating_2018,varre_accelerated_2022}, non-uniform iterate averaging~\citep{jain_parallelizing_2018,neu_iterate_2018,muecke_beating_2019}, infinite-dimensional frameworks~\citep{dieuleveut_nonparametric_2016}, or over-parametrized regimes and double descent phenomena \citep{belkin_reconciling_2019}. 

Our approach follows this line of work: our goal is to analyze the impact of \textit{compression} in FL algorithms, by providing a careful study of compressed LMS, based on a fine-grained analysis of Stochastic Approximation (SA) under weak assumptions on the random field. 
More precisely, we consider linear stochastic approximation recursion, to find a zero of the linear mean-field $\nabla F$.

\begin{definition}[Linear Stochastic Approximation, LSA]
\label[definition]{def:class_of_algo}
Let $w_0 \in \R^d$ be the initialization, the linear\footnote{While in LSA literature, both the mean-field  $\nabla F$ and the noise-field $(\xi_k)$ are linear, we do not here consider the noise fields to be linear.
} stochastic approximation recursion is defined as:
\begin{equation}
\tag{LSA}\label{eq:LSA}
w_{k} = \wkm - \gamma \nabla F (\wkm) + \gamma \xi_k(\wkm - \ws),\ \  k \in \N ,
\end{equation}
where $\gamma > 0$ is the step-size and 
$(\xi_k)_{k \in \N^*}$ is a sequence of i.i.d. zero-centered random fields that characterizes the stochastic oracle on $\nabla F(\cdot)$.
For any $k\in \N^*$, we denote $\Flast_k = \sigma\bigpar{\xi_1, \dots, \xi_k}$, such  that the filtration $(\Flast_k)_{k\geq0}$ is adapted to $(w_k)_{k \geq 0}$. 

We assume that $F$ is quadratic, we denote $\Fhess$ its Hessian, $\Ftrace^2 := \Tr{\Fhess}$ its trace and $\mu$ its smallest eigenvalue. For any $k$ in $\N$, with $\eta_k = w_k - \ws$, we get equivalently:
\begin{equation*} 
\eta_k = (\Id- \gamma \Fhess) \eta_{k-1} + \gamma \xi_k(\eta_{k-1}), \  \ k \in \N .
\end{equation*}
\end{definition}
As underlined by \citet{bach_non-strongly-convex_2013},  \eqref{eq:LSA} corresponds to a homogeneous Markov chain. A study of stochastic approximation using results and techniques from the Markov chain literature can be found for instance in \citet{freidlin_random_1998} or more recently in \citet{dieuleveut_bridging_2020}. 

\eqref{eq:LSA} encompasses three examples of interest, the first one is the classical LMS algorithm.
Indeed, with the observations  in \Cref{model:fed,model:centralized}, for any client $i\in\OneToN$, any iteration $k$ in $\OneToK$, any model $w\in \R^d$, 
\begin{equation}\label{eq:def_oracle}
\g_k^i(w) := (\PdtScl{x_k^i}{w} - y_k^i) x_k^i 
\end{equation}
is an unbiased oracle of $\nabla F_i(w)$. This can be used to define the following three algorithms.

\begin{algo}[LMS]\label{ex:LMS}
For LMS algorithm, with a single worker (\Cref{model:centralized}), we have for all $k \in \N$,  $w_k = \wkm - \gamma \g_k(\wkm) = \wkm - \gamma (\PdtScl{x_k}{\wkm} -y_k)x_k$, thus equivalently, we have~$\xi_k(\cdot) = (\E[x_1 x_1^\top] - x_k x_k^\top) (\cdot) + \epsilon_k x_k $. Indeed, for any $w$ in $\R^d$, $\xi_k(w - w_*) = \nabla F(w)  - \g_k(w) = \E[x_1 x_1^\top] (w - \ws) - (\PdtScl{x_k}{w} - y_k) x_k = (\E[x_1 x_1^\top] - x_k x_k^\top) (w - \ws) - (\PdtScl{x_k}{w_*} - y_k) x_k$.
\end{algo}
Second, the case of a single client compressed LMS algorithm. 
\begin{algo}[Centralized compressed LMS]\label{ex:cent_comp_LMS}
A single client ($N=1$) observes at any step $k \in\OneToK$ an oracle $\g_k(\cdot)$ on the gradient of the objective function $F$, and applies a random compression mechanism $\C_k(\cdot)$. Thus, for any step-size $\gamma>0$ and any $k \in \N^*$, the resulting sequence of iterates $(w_k)_{k \in\N}$ satisfies: $w_k = \wkm - {\gamma}\C_k(\g_k(\wkm)) \,.$
\end{algo}

And finally, the extension to the distributed case.
\begin{algo}[Distributed compressed LMS]\label{ex:dist_comp_LMS}
In our motivating example,  each client $i \in \OneToN$ observes at any step $k \in \OneToK$ an oracle $\g_k^i(\cdot)$ on the gradient of the local objective function $F_i$, and applies a random compression mechanism $\C_k^i(\cdot)$. Thus, for any step-size $\gamma>0$ and any $k \in \N^*$, the resulting sequence of iterates $(w_k)_{k \in\N}$ satisfies:
$w_k = \wkm - \frac{\gamma}{N} \sum_\iN \C_k^i(\g_k^i(\wkm)) \,$ (we consider the randomization made on clients $(\C_k^i(\cdot))_{i \in \{1 \cdots, N\}}$ to be independent)
\end{algo}

\begin{remark}
    The analysis naturally covers any randomized postprocessing $\C_k^i(\cdot)$, beyond the compression case.
\end{remark}

\paragraph{Challenges, contributions and structure of the paper.}
Although there is abundant literature on the study of \eqref{eq:LSA}, the application to~\Cref{ex:cent_comp_LMS,ex:dist_comp_LMS} poses novel challenges. Especially, most analyses of LSA \citep[][]{blum_multidimensional_1954,ljung_analysis_1977,ljung_theory_1983} assume that the field $\xi_k$ is linear \citep[i.e.~for any $z,z' \in \R^d, \ \xi_k(z)-\xi_k(z')=\xi_k(z-z')$, see][]{konda_linear_2003,benveniste_adaptive_2012,leluc_sgd_2022}. More general non-asymptotic results on SA with a Lipschitz mean-field (i.e. SGD with a smooth objective) also assume that the noise-field is Lipschitz-in-squared-expectation i.e.~for any $z,z' \in \R^d,\E[\|\xi_k(z)-\xi_k(z')\|^2] \leq C \|z-z'\|^2$
\citep[][]{moulines_nonasymptotic_2011,bach_adaptivity_2014,dieuleveut_bridging_2020,gadat_optimal_2023}. One major specificity and bottleneck in the case of compression is the fact that the resulting field \textbf{does not} satisfy such an assumption.
The rest of the paper is thus organized as follows:
\begin{enumerate}[topsep=2pt,itemsep=1pt,leftmargin=0.5cm]
\item In \Cref{sec:theoretical_analysis}, we provide a non-asymptotic analysis of \eqref{eq:LSA} under weak regularity assumptions of the noise field $(\xi_k)_k$. 
We show that the asymptotically dominant term depends  on the covariance matrix $\aniac$ of the \textit{\underline{a}dditive \underline{n}oise \underline{i}nduced by the \underline{a}lgorithm}, as expected from the classical asymptotic literature ~\citep{polyak_acceleration_1992}. The backbone results of our paper are \Cref{thm:bm2013_with_nonlinear_operator_compression,thm:bm2013_with_linear_operator_compression} which generalize the results from \cite{bach_non-strongly-convex_2013} for \Cref{ex:LMS}. The limit convergence rate term scales with  $\Tr{\aniac \Fhess^{-1}}/K$, which highlights the interaction between the Hessian of the optimization problem $\Fhess$, and the additive noise's covariance $\aniac$.

 \item In \Cref{sec:application_compressed_LSR}, we prove that assumptions made in \Cref{sec:theoretical_analysis} are valid for \Cref{ex:cent_comp_LMS} with classical compression schemes.
 Although this single-client case is a simple configuration, it enables to describe the impact of the compressor choice on the dependency between the features' covariance $\xCov$ (which is also the Hessian $\Fhess$ of the optimization problem) and the additive noise's covariance $\aniac$.
 Contrary to \Cref{ex:LMS}, for which the noise is said to be \textit{structured}, i.e.~the additive noise's covariance is proportional to the Hessian $\Fhess$, applying a random compression mechanism on the gradient breaks this  structure.
This phenomenon is noteworthy: for an ill-conditioned $\Fhess$, it  may lead to a drastic increase in $\Tr{\aniac \Fhess^{-1}}$ and thus, to a degradation in convergence. 
By calculating the additive noise's covariance for various compression mechanisms, we identify differences that classical literature was unable to capture.

\item In \Cref{sec:federated_learning}, we study the distributed \Cref{ex:dist_comp_LMS} with heterogeneous clients. 
We examine two different sources of heterogeneity for which we show that \Cref{thm:bm2013_with_nonlinear_operator_compression,thm:bm2013_with_linear_operator_compression} remain valid. First, the case of heterogeneous features' covariances $(\xCov_i)_\iN$ in \Cref{subsec:fl_sigma}; second, the case of heterogeneous local optimal points $(\ws^i)_\iN$ in \Cref{subsec:fl_wstar}. 
\end{enumerate}

These results are validated by numerical experiments which help to get an intuition of the underlying mechanisms. The code is provided on our GitHub repository: \url{https://github.com/philipco/structured_noise}.
We summarize hereafter the structure of the paper in \Cref{fig:chart_summary_paper}.

\begin{figure}
    \centering
    \includegraphics[trim={8.2cm 13.1cm 7.2cm 5.5cm},clip, width=0.85\linewidth]{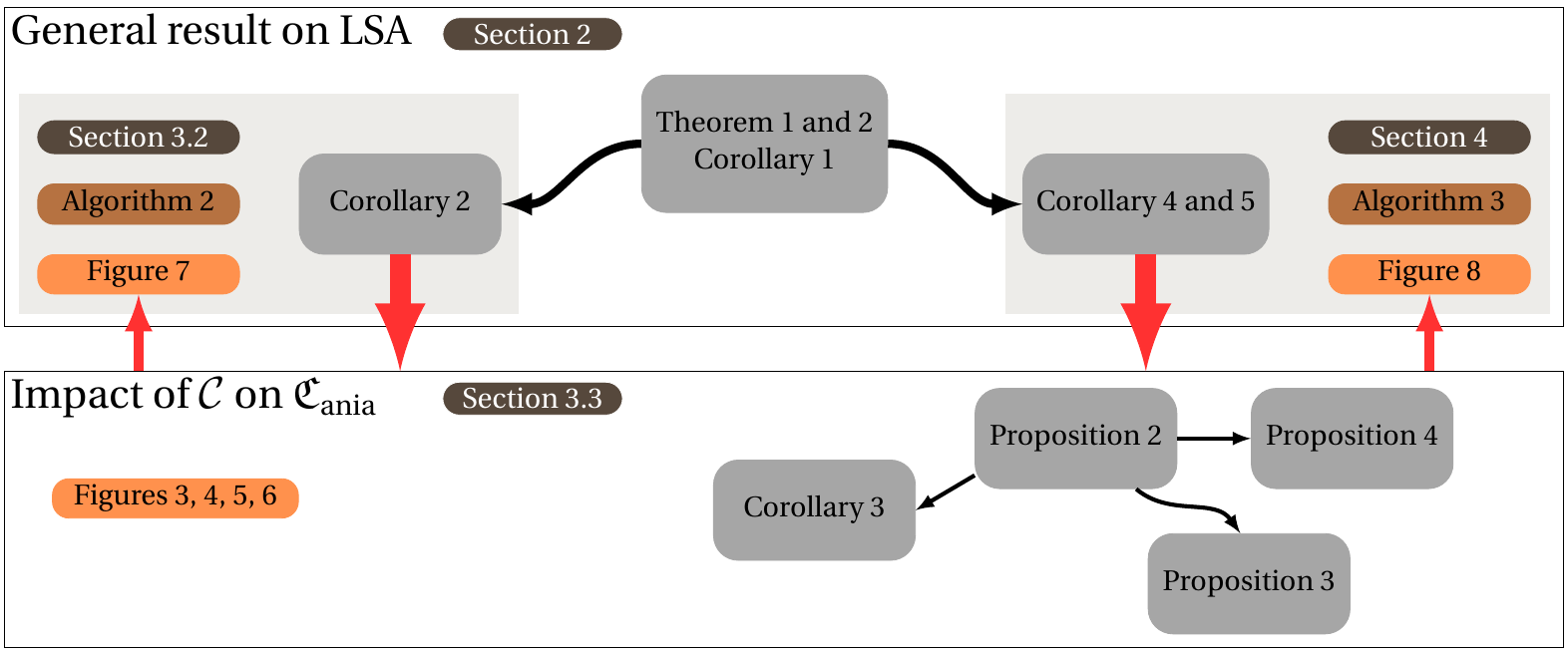}
    \caption{Flow chart summarizing our results.\gs\gs}
    \label{fig:chart_summary_paper}
\end{figure}

\paragraph{Notations.} We denote by $\preccurlyeq$ the order between self-adjoint operators, i.e., $A \preccurlyeq B$ if and only if $B-A$ is positive semi-definite (p.s.d.) and $A \precsim B$ if $A \preccurlyeq B$ and $A = B + O(\frac{1}{d})$. We denote by $A^{1/2}$ the p.s.d.~square root of any symmetric p.s.d.~matrix $A$. 
For two vectors $x, y$ in $\R^d$, the Kronecker product is defined as $x \otimes y := x y^\top$, the element-wise product is denoted as $x\odot y$, and the Euclidean norm is $\| x \|^2 := \sum_{i=1}^d x_i^2$.  For any rectangular matrix $A$ in $\R^{n\times m}$ s.t. $AA^\top$ is inversible, we denote $A^\dagger := A^\top (A A^\top)^{-1}$ the Moore–Penrose pseudo inverse. For $x,y$ in $\R^d$, we use $x\wedge y$ for the minimum between two values, and $x\lesssim y$ if $x \leq y $ and $x = y + O(\frac{1}{d})$. For any sequence of vector $(x_k)_{k \in \{0,\dots, K \}}$ we denote $\overline{x}_{K-1} =\sum_{k=0}^{K-1} x_k / K$.
We use $e_i$ to denote the vector in $\R^d$ with zero everywhere except at coordinate $i$, and $\mathcal{O}_d(\R)$ the group of orthogonal matrices. 
Finally, 
all random variables are defined on a probability space $(\Omega, \mathcal{A}, \mathbb{P}), \mathbb{E}$ is the expectation associated with the probability $\mathbb{P}$ and $\mathcal{A}$ is a $\sigma$-algebra. 
We define the set of probability distribution function $\mathcal P_M$ 
whose second moment is equal to $M$ in $\R^{d \times d}$:
$
\mathcal P_M = \{\text{probability distribution } p_M \text{~over $\R^d$ s.t.}, \E_{\varepsilon \sim p_M} [\varepsilon^{\kcarre}] = M \}\,.
$
Any such distribution  $p_M$ is indexed with its matrix of covariance. 

\section{Non asymptotic convergence result for \texorpdfstring{\eqref{eq:LSA}}{LSA}}
\label{sec:theoretical_analysis}

\subsection{Definition of the additive noise's covariance and assumptions on the random fields}
\label{subsec:def_ania_asu_field}
For any $k$ in $\N^*$, we define the additive noise $\xi_k^\mathrm{add}$ and the multiplicative noise $\xi_k^\mathrm{mult}(\cdot)$.
\begin{definition}[Additive and multiplicative noise]
\label[definition]{def:add_mult_noise}
Under the setting of \Cref{def:class_of_algo}, for any $k$ in $\N^*$, we define:
\[
\xi_k^\mathrm{add} := \xi_k(0) \qquad \text{and} \qquad \xi_k^\mathrm{mult} : z \in \R^d \mapsto \xi_k(z) - \xi_k^\mathrm{add} \,.
\]
\end{definition}
\begin{remark}
\label[remark]{remark:main:add_noise_indep_of_k}
Observe that $(\xi_k^\mathrm{add})_{k\in\N^*}$ is an i.i.d. sequence of random variables and $(\xi_k^\mathrm{mult})_{k\in\N^*}$ is an i.i.d. sequence of random field. The following assumptions, made for $k=1$, are thus equivalently valid for any $k\geq 1$.
\end{remark}

\begin{assumption}[Second moment]
\label{asu:main:bound_add_noise}
${\xi_1^\mathrm{add}}$ admits a second order moment. We note $\boundAdd \geq 0$ such that $\E [\|\xi_1^\mathrm{add}\|^2] ~\leq~\boundAdd$.
\end{assumption}
\Cref{asu:main:bound_add_noise,remark:main:add_noise_indep_of_k} enable us to define the covariance of the \underline{a}dditive \underline{n}oise \underline{i}nduced by the \underline{a}lgorithm. 
\begin{definition}[Additive noise's induced by the algorithm's covariance.]
\label[definition]{def:ania}
Under the setting of \Cref{def:class_of_algo}, we define the \emph{additive noise's covariance} as the covariance of the additive noise:
$\aniac = \E[\xi_1^\mathrm{add} \otimes \xi_1^\mathrm{add}]\,.
$
\end{definition}

Secondly, we state our assumptions on the multiplicative part of the noise, especially its regularity around 0 (note that $\xi_1^\mathrm{mult}(0)=0$).
\begin{assumption}[Second moment of the multiplicative noise]
\label{asu:main:second_moment_noise}
There exist two constants $\boundMultPrime,\boundMult > 0$ such that, for any $\eta$ in $\R^d$, the following hold:
\begin{enumerate}[label={\textbf{A\ref*{asu:main:second_moment_noise}.\arabic*:}}, ref={\theassumption.\arabic*}, leftmargin=*, align=left]
\item \label[assumption]{asu:main:bound_mult_noise}
    $\E[\|{\xi_1^\mathrm{mult}(\eta)}\|^2] \leq 2 \boundMult \|{\Fhess^{1/2} \eta}\|^2 + {4 \boundAdd\, }.$
\item \label[assumption]{asu:main:bound_mult_noise_holder}
    $\E[\|{\xi_1^\mathrm{mult}(\eta)}\|^2] \leq  \boundMultPrime \|\Fhess^{1/2} \eta \| + 3 \boundMult \|\Fhess^{1/2} \eta \|^2.$
\end{enumerate}
\end{assumption}

The main originality of this section is the analysis under \Cref{asu:main:bound_mult_noise_holder}.  This Hölder-type condition will appear naturally for  compression in \Cref{sec:application_compressed_LSR}. Up to our knowledge, \eqref{eq:LSA} has not been analyzed under this particular condition.

Under these assumptions, asymptotic results from \citet{polyak_acceleration_1992} can be applied. Especially, we establish the asymptotic normality of $(\sqrt{K} \etabarkm)_{K >0}$, with an asymptotic variance equal to $\Fhess^{-1} \aniac \Fhess^{-1}$. 
\begin{proposition}[CLT for \eqref{eq:LSA}]
\label[proposition]{prop:asymptotic_normality}
Under \Cref{asu:main:bound_add_noise,asu:main:second_moment_noise}, consider a sequence $(w_k)_{k\in\N^*}$ produced in the setting of \Cref{def:class_of_algo} for a step-size $(\gamma_k)_{k\in \N^*}$ s.t. $\gamma_k =k^{-\alpha}$, $\alpha\in ]0,1[$. Then $(\sqrt{K} \etabarkm)_{K>0}$ is asymptotically normal and converge in distribution to $\mathcal{N}(0, \Fhess^{-1} \aniac \Fhess^{-1})$.
\end{proposition}
The proof of this result is almost straightforward and is recalled in \Cref{app:subsec:clt}. In the following, we establish non-asymptotic results in \Cref{thm:bm2013_with_linear_operator_compression,thm:bm2013_with_nonlinear_operator_compression}, that highlight the impact of \Cref{asu:main:bound_mult_noise_holder}.

\subsection{Convergence rates for \texorpdfstring{\eqref{eq:LSA}}{LSA}, general case}
\label{subsec:theorems}

In this section, we present non-asymptotic convergence rates for \eqref{eq:LSA} under the  assumptions above. These results build upon the work of \citet{bach_non-strongly-convex_2013}.  Our first result is the main result, under the Hölder assumption on the noise field, it is demonstrated in \Cref{app:sec:nonlinear_bm}.

\begin{theorem}[Non-linear multiplicative noise]
\label{thm:bm2013_with_nonlinear_operator_compression}
Under \Cref{asu:main:second_moment_noise,asu:main:bound_add_noise}, consider a sequence $(w_k)_{k\in\N^*}$ produced in the setting of \Cref{def:class_of_algo} for a constant step-size $\gamma$ such that $ \gamma( \Ftrace^2 + 2 \boundMult) \leq 1/2$. Then for any horizon $K$, we have:
\begin{align*}
    \E[ F(\overline{w}_{K-1}) - F(\ws) ] &\leq \frac{1}{2K} \Bigg(\ffrac{\|\Fhess^{-1/2} \eta_0\|}{\gamma \sqrt{K}} \wedge \frac{\|\eta_0 \|}{\sqrt{\gamma}} + \sqrt{\Tr{\aniac \Fhess^{-1} }}  + \bigpar{10\boundAdd \gamma}^{1/4}\sqrt{\boundMultPrime \mu^{-1}}  \\
    &\qqquad + \bigpar{30\boundAdd \gamma}^{1/2} \sqrt{\boundMult\mu^{-1}} \Bigg)^2 \,.
\end{align*}
\end{theorem}

The first two terms of the RHS correspond respectively to the impact of the initial condition~$\eta_0$ and the impact of the additive noise. The dependency on these two terms is similar to the one established in \citet{bach_non-strongly-convex_2013} in the case of LMS.   Note that following~\citet{defossez_averaged_2015}, we improve the dependency on the initial condition to $\frac{\|\eta_0 \|^2}{\gamma K} \wedge \frac{\|\Fhess^{-1/2}\eta_0 \|^2}{\gamma^2 K^2}$. Regarding the noise term, the dependency on $\frac{\Tr{\aniac \Fhess^{-1} }}{2K}$ corresponds to the classical asymptotic noise term in CLT for Stochastic Approximation \citep[e.g., ][]{delyon1996general,duflo_random_1997,gyorfi_averaged_1996}. 
In fact, for a sequence of step sizes  $\gamma_t$ decreasing to zero, we recover the variance from \Cref{prop:asymptotic_normality}. 
Remark that in \citep{bach_non-strongly-convex_2013} and several follow up works, the algorithm under consideration is LMS (\Cref{ex:LMS}, which enables to  ensure that $\aniac\preccurlyeq \sigma^2 \Fhess$: the variance term thus scales as $\sigma^2d/K$. On the contrary,  \Cref{ex:cent_comp_LMS,ex:dist_comp_LMS} do not always satisfy $\aniac\preccurlyeq \sigma^2 \Fhess$: in such case, $\Tr{\aniac \Fhess^{-1} }$ may scale as $1/\mu$.

The third and fourth term, that scale respectively as $\sqrt{\gamma}/K$ and $\gamma/K$, are asymptotically negligible for $\gamma = o(1)$. Those term  are proportional to the Hölder-regularity constants $\boundMultPrime,\boundMult$, and also increase with $\mu^{-1}$. The dominant term is $\frac{\boundMultPrime\sqrt{10\boundAdd \gamma}}{\mu K}$. Interestingly, when $\gamma$ is constant (not decreasing with $K$), then the limit variance of the algorithm is affected. 
Moreover, contrary to \citep{bach_non-strongly-convex_2013}, we do not recover a convergence rate independent of $\mu$.  This dependency is un-avoidable as the multiplicative noise is only controlled around $\ws$: without strong-convexity, the iterates may not converge to $\ws$.
While these additional terms in the variance may be considered as a drawback, it can be mitigated by taking a step-size $\gamma$ proportional to $1/K^\alpha$ with $\alpha> 0$ small ($\gamma$ is horizon dependent, but constant). 

\begin{corollary}
\label[corollary]{cor:bm2013_with_nonlinear_operator_compression}
Under the assumptions of \Cref{thm:bm2013_with_nonlinear_operator_compression}, 
with  $\gamma = 1 / K^\alpha$,  and $\alpha\in ]0,1/2[$,  we have:
\begin{eqnarray*}
 \E[ F(\overline{w}_{K-1}) - F(\ws) ] &   \leq & \frac{60}{K} \Bigg({\Tr{\aniac \Fhess^{-1} }}  +  \ffrac{\|\Fhess^{-1/2} \eta_0\|^2}{K^{(1-2\alpha)}} + \frac{{\boundMultPrime}  \sqrt{\boundAdd }}{\mu K^{\alpha/2}}+\frac{{\boundMult}  \boundAdd}{\mu K^{\alpha}} \Bigg) \,.
\end{eqnarray*}
\end{corollary}
The decrease of the second order terms is then optimized for $\alpha=2/5$. To highlight the impact of the non-linearity in compression schemes, we provide for comparison the result for a linear multiplicative noise.

\subsection{Convergence rates for \texorpdfstring{\eqref{eq:LSA}}{LSA}, linear case}
\label{subsec:theorems_lin}

Alternatively, to cover the particular case of a linear multiplicative noise (e.g., to recover LMS or projection-based compressed LMS) we make the following stronger hypothesis:
\begin{assumption}
\label{asu:main:bound_mult_noise_lin}
The multiplicative noise is linear i.e.~there exists a random matrix $\Xi_1$ in $\R^{d\times d}$ s.t.~for any $\eta$ in~$\R^d$, we have a.s.  $\xi_1^\mathrm{mult}(\eta) = \Xi_1 \eta$. Moreover $\E[\|{\xi_1^\mathrm{mult}(\eta)}\|^2] \leq  \boundMult \|\Fhess^{1/2} \eta \|^2$.
\end{assumption}

\begin{remark}
Note that $\Xi_1$ is not necessarily symmetric (in \Cref{ex:dist_comp_LMS,ex:cent_comp_LMS}, this results from the compression). 
\end{remark}

In addition to \Cref{asu:main:bound_mult_noise_lin}, in the case of linear multiplicative noise, we also consider the following assumption.
\begin{assumption}
\label{asu:main:baniac_lin}
The following hold.
\begin{enumerate}[label={\textbf{A\ref*{asu:main:baniac_lin}.\arabic*:}},ref={\theassumption.\arabic*}]
\item \label[assumption]{asu:main:bound_cov_add_noise_lin} There exists a constant\footnote{This letter $\Sha$ is the Russian upper letter ``sha''.} $\ShaA>0 $ s.t. $\aniac \preccurlyeq \ShaA \Fhess$.
\item \label[assumption]{asu:main:bound_cov_mult_noise_lin}  There exists a constant $\ShaM >0 $, such that $\FullExpec{\Xi_1 \Xi_1^\top} \preccurlyeq  \ShaM \Fhess$. 
\end{enumerate}
\end{assumption}

\begin{remark}[Link between \Cref{asu:main:baniac_lin,asu:main:second_moment_noise,asu:main:bound_add_noise}]
 \Cref{asu:main:bound_add_noise} (resp. \Cref{asu:main:second_moment_noise}) corresponds to an assumption on the second order moment of the additive noise (resp. multiplicative), while \Cref{asu:main:bound_cov_add_noise_lin}  (resp. \Cref{asu:main:bound_cov_mult_noise_lin}) is a (stronger) assumption on its covariance.  
\end{remark}

\begin{theorem}[Linear multiplicative noise]
\label{thm:bm2013_with_linear_operator_compression}
Under \Cref{asu:main:bound_add_noise,asu:main:bound_mult_noise_lin,asu:main:baniac_lin}, i.e., with a linear multiplicative noise. Consider a sequence $(w_k)_{k\in\N^*}$ produced in the setting of \Cref{def:class_of_algo}, for a constant step-size $\gamma$ such that $\gamma ( \Ftrace^2 + \boundMult) \leq 1$ and $4 \ShaM \gamma  \leq 1$.  Then for any horizon $K$, we have
\begin{align*}
    \E[ F(\overline{w}_{K-1}) - F(\ws) ]\leq \frac{1}{2K} \bigpar{\frac{\|\eta_0 \|}{\sqrt{\gamma}} + \sqrt{\Tr{\aniac \Fhess^{-1}}} + 2  \sqrt{\gamma d \ShaA \ShaM}}^2 \,.
\end{align*}
\end{theorem}

\Cref{thm:bm2013_with_linear_operator_compression} generalizes Theorem 1 from \citet{bach_non-strongly-convex_2013}. It also highlights the impact of additive noise's covariance, and the comparison between \Cref{thm:bm2013_with_nonlinear_operator_compression} and \Cref{thm:bm2013_with_linear_operator_compression} shows the advantage of linear compression schemes. Indeed the variance scales as  $K^{-1}(\Tr{\aniac \Fhess^{-1}} + 4   \gamma d \ShaA \ShaM) $. As before, the first term $\Tr{\aniac \Fhess^{-1}}$ corresponds to the asymptotic variance given in~\Cref{prop:asymptotic_normality}, and the second term is negligible: (i) for all $4 \ShaM \gamma  \leq 1$ it can be upper bounded by $d \ShaA$, and for LMS \citep[see][]{bach_non-strongly-convex_2013}, the variance term is $\Tr{\aniac \Fhess^{-1}}=d \sigma^2 $, which is thus at least as large, (ii) it scales with $\gamma$ thus is asymptotically negligible as $\gamma$ tends to $0$.  
Overall,  depending on  $\aniac$, the algorithm may or may not suffer from the lack of  strong-convexity ($\mu$ tending to $ 0$).
More precisely, in the case of linear multiplicative noise, we can obtain a $O(K^{-1})$ rate independent of $\mu$  if and only if $\aniac\preccurlyeq a \Fhess$, with $a$ in $\R$. The proof of \Cref{thm:bm2013_with_linear_operator_compression} is given in \Cref{app:sec:linear_bm}, and follows the line of proof of \cite{bach_non-strongly-convex_2013}.

\textbf{Conclusion:} we established rates for \eqref{eq:LSA} for both the Hölder-noise case and the linear noise case. In the former, convergence requires strong convexity while in the latter, we can achieve $O(K^{-1})$ for  $\aniac\preccurlyeq a \Fhess$. In both cases, the dominant term for an optimal choice of $\gamma$ scales as $\frac{\Tr{\aniac \Fhess^{-1} }}{K}$. 

In the following section, we turn to the analysis of \Cref{ex:cent_comp_LMS}: we show how the choice of the compression impacts both the linearity of the  noise and the structure of $\aniac$. 

\section{Application to Algorithm~\ref{ex:cent_comp_LMS}:  compressed LSR on a single worker}
\label{sec:application_compressed_LSR}

In this section, we analyze  \Cref{ex:cent_comp_LMS}, i.e.~compressed LSR.
In \Cref{subsec:compressors}, we introduce the  compression operators of interest and verify in \Cref{subsec:applicability_lsa} that \Cref{thm:bm2013_with_linear_operator_compression,thm:bm2013_with_nonlinear_operator_compression} can be applied.
Then, in \Cref{subsec:impact_cov_on_additive_noise}, we provide {explicit formulas of $\Tr{\aniac H^{-1}}$} for various compression schemes.
Finally, in \Cref{subsec:numerical_experiments}, we validate our findings with numerical experiments.

\subsection{Compression operators}
\label{subsec:compressors}
Our analysis applies to most unbiased compression operators.

\begin{definition}[Compression operators]
\label[definition]{def:operators_compression}
Let $z\in\R^d$.
\begin{enumerate}[itemsep=1pt,topsep=0pt]
    \item 
\textbf{$1$-quantization} is defined as $\Cquant (z) := \| z \| \sign(z) \odot \chi  \text{~with~} \chi \sim \otimes_{i=1}^d (\mathrm{Bern}(|z_i| / \|z\|_2)).$

\item \textbf{Stabilized $1$-quantization} is defined as $\Cstabquant (z) := U^\top \C_q(Uz)$, with $U \in \mathrm{Unif}(\mathcal{O}_d)$.

\item \textbf{Rand-$h$} is defined as $\Crandh(z) := \frac{d}{h} B(S) \odot z$ with $S\sim \mathrm{Unif}(\mathcal{P}_h([d]))$ and $B(S)_i = \mathbb{1}_{i \in S}$.

\item \textbf{Sparsification} is defined as $\Cspars (z) := \frac{1}{p} B \odot z \in \R^d \text{~with~}B \sim \otimes_{i=1}^d\bigpar{\mathrm{Bern}(p)} \,.$

\item \textbf{Partial participation} is defined $\Cpp (z) := \frac{b_0 }{p} z$ with $b_0 \sim \mathrm{Bern}(p)$.

\item \textbf{Random Projection}, also referred to as \emph{sketching},  is defined as $\Csketch (z) := \frac{1}{p} \Phi^\dagger \Phi z$, where $\subdim \ll \bigdim\in\N$, $p= \subdim / \bigdim$ and $\Phi \in \R^{\subdim \times \bigdim}$ is a random projection matrix onto a lower-dimension space \citep{vempala_random_2005,li_very_2006}. 
In the following, we consider Gaussian projection, where each element $i, j\in\llbracket 1, \subdim \rrbracket \times \llbracket 1, \bigdim \rrbracket$ follows an independent zero-centered normal distribution.  

\end{enumerate}
\end{definition}

We refer to the introduction for related work on compression. Operators $\Cquant,\Cstabquant$ are quantization-based schemes while $\Crand,\Cspars, \Cpp, \Csketch$ are projection-based. Indeed  sparsification can be seen as a random projection (for $\subdim \ll \bigdim$, $p= \subdim / \bigdim$   and $\subdim$ randomly sampled  coordinates $\I$ from $\llbracket 1, d \rrbracket$ such that for any $i\in\I$, the $i^{th}$ lines of $\Phi$ are equal to $e_i\in\R^\bigdim$, and equal to zero otherwise).
For $\Cpp$, the motivation is distributed settings, in which the intermittent availability of clients prevents them from systematically participating in the training. This can be modeled through \textit{partial participation}: clients only participate in a fraction $p$ of the training steps. In theoretical analyses, this can be handled as a compression scheme $\Cpp$, in which the compression of a vector $z$ is either $z/p$ or 0. Observe that in the centralized case, this is slightly artificial as it actually means that no update is performed at most steps and that the step-size is scaled at the other steps. 
Finally, we denote $\C_{\Id_d} : z \in \R^d \mapsto z$ the operator that does not carry out any compression.

\begin{remark}
The analysis of random projection is related to Random features \citep{rahimi_random_2008}, usually used for Kernel learning in infinite dimensions. Nyström method \citep[introduced by][]{kumar_ensemble_2009} is another similar technique of compression often used in this setting, it consists of removing a subset $\mathcal S \subset \{1, \cdots, d\}$ of lines and columns in the kernel matrix $K$. 
Both techniques have been extensively studied in the context of linear and non-linear kernel learning \citep{rudi_less_2015,rudi_falkon_2017,rudi_generalization_2017,lin_optimal_2017}.
Recently, the combination of SGD and random features has been analyzed by \citet{carratino_learning_2018}. However, their results cannot be directly applied to our setting for two reasons. Firstly, their analysis is for infinite dimensions, where they obtain a $O(1/\sqrt{K})$ rate of convergence. Secondly, the compressions used in their approach are not independent at each iteration.
\end{remark}

\begin{remark}
   Diffusion LMS (i.e. distributed learning without a central server) has also been studied from the perspective of low-cost training by \citet{arablouei_analysis_2015,harrane_reducing_2018}, but using only  clients' partial participation or sparsification. Contrary to our work they use biased compression and an adaptive correction step to compensate for the induced error. They provide results guarantying asymptotic convergence \citep[][see Equations (28)-(37)]{harrane_reducing_2018}.
\end{remark}

\subsection{Applicability of the results on~\texorpdfstring{\eqref{eq:LSA}}{LSA}~from~Section~\ref{sec:theoretical_analysis}}
\label{subsec:applicability_lsa}

We first show that our results from \Cref{sec:theoretical_analysis} can be applied for \Cref{ex:cent_comp_LMS} with a random compression operator $\C$, in the case of \Cref{model:centralized}. 

\begin{lemma}
\label[lemma]{lem:compressor}
For any compressor $\C \in \{ \Cquant, \Cstabquant, \Crandh, \Cspars, \Csketch, \Cpp \}$, there exists constants $\omgC, \omgCOne \in \R^*_+$, such that the random operator $\C$ satisfies the following properties for all $z, z'\in\R^d$.
\begin{enumerate}[label={\textbf{L.\arabic*:}},ref={L.\arabic*},noitemsep]
\item \label{item:urvb} $\E [\C(z)] = z$ and $\E [ \| \C(z) - z\|^2] \leq \omgC \|z\|^2$ (unbiasedness and variance relatively bounded),
\item \label{item:holder_compressor} $\E [ \| \C(z) - \C(z')\|^2] \leq  \omgCOne  \min(\|z\|, \|z'\|) \|z - z'\|+ 3 \omgCTwo \|z- z' \|^2 \text{(Hölder-type bound),}$
\end{enumerate}
with $\omgC = \sqrt{d}$ and $\omgCOne
 = 12 \sqrt{d}$ (resp.  $\omgC = (1-p)/p$ and $\omgCOne
 = 0$) for $\Cquant$ and $\Cstabquant$ (resp. $\Crandh, \Cspars$, $\Csketch, \Cpp$). 
\end{lemma}

We note $\mathbb{C}$ the set of unbiased compressors verifying \Cref{lem:compressor}.  \Cref{item:urvb} is frequently established in the literature and corresponds to the worst-case assumption, see the introduction for references. On the other hand, \Cref{item:holder_compressor} is the Hölder-type bound, which is not used in the literature up to our knowledge. The expected squared distance between the compression of two nearby points scales with the \emph{non-squared} norm of the distance. Moreover, the distance is multiplied by an unavoidable coefficient scaling with $z, z'$. Remark that in \Cref{item:holder_compressor}, we assume the compression randomness to be the same for the compression of $z$ and $z'$: formally, we control $\mathcal W_2(\C(z), \C(z'))^2$, with $\mathcal W_2$ the Wassertein-$2$ distance. This lemma is demonstrated in \Cref{app:subsec:var_cov_compression}.

\begin{remark}
    For a given $\omgC$, note that the \emph{communication cost} $c$ for quantization-based and projection-based compressors is not always equivalent. For $1$-quantization we have $c \approx \frac{3}{2} \sqrt{d} \log_2{d} + 32$ while for projection-based we have $c \approx 32 \sqrt{d}$, for $\sqrt{d}$-quantization we have $c \approx 3d + 32 $ while for projection-based we have $c=16d$.
\end{remark}

\Cref{lem:compressor} enables to show that \Cref{thm:bm2013_with_linear_operator_compression,thm:bm2013_with_nonlinear_operator_compression}, and \Cref{ex:cent_comp_LMS} are valid in the context of \Cref{model:centralized}.

\begin{corollary}
\label[corollary]{cor:value_constants_in_thm}
Consider \Cref{ex:cent_comp_LMS} in the context of \Cref{model:centralized}, with a compressor $\C \in \{ \Cquant, \Cstabquant, \Crandh, \Cspars,$ $\Csketch, \Cpp \}$. With  \Cref{lem:compressor} above, \Cref{asu:main:second_moment_noise,asu:main:bound_add_noise} on the resulting random field~$(\xi_k)_{k \in \N^*}$ are valid, with in particular $\Fhess = \xCov$, $\Ftrace^2 = R^2$, $\boundAdd = (\omgC +1) R^2 \sigma^2$, $\boundMult = (\omgC + 1)R^2$, $\boundMultPrime = \omgCOne R^2 \sigma$. Therefore, it follows that \Cref{thm:bm2013_with_nonlinear_operator_compression} holds.

Moreover for any linear compressor $\C \in \{\Crandh, \Cspars,$ $\Csketch, \Cpp \}$, under \Cref{rem:as_bounded_features}, we also have that \Cref{asu:main:bound_mult_noise_lin,asu:main:baniac_lin} are valid with $\ShaA = \sigma^2 \Sha_\xCov$ and $\ShaM = R^2 \Sha_\xCov$, with $\Sha_\xCov$ given below.
Therefore, it follows that \Cref{thm:bm2013_with_linear_operator_compression} holds. 

\begin{center}
    \begin{tabular}{lllll}
    \toprule
Compressor & $\Crandh$ & $\Cspars$ & $\Cpp$ & $\Csketch$ \\
\midrule
    $\Sha_\xCov$ & $\frac{h-1}{p(d-1)} + (1- \frac{h-1}{d-1}) \frac{\tau}{p}$&  $1 + \frac{(1- p) \tau}{p}$ &$\frac{1}{p}$ & $\frac{\alpha - \beta}{p} + \frac{\beta \tau}{p}$ \\ 
    $\Sha_\xCov$ (if $H$ diagonal) & $\frac{1}{p}$&  $\frac{1}{p}$  &$\frac{1}{p}$& $\frac{\alpha - \beta}{p} + \frac{\beta \tau}{p}$\\
    \bottomrule
\end{tabular}
\end{center}
Where $p = h/d$, $\tau = \Tr{\xCov} / \mu$, and for sketching $\alpha = \frac{\subdim +2}{d+2}$ and $\beta = \frac{d - \subdim}{(d - 1) (d + 2)}$.
\end{corollary}
This corollary is proved in \Cref{app:subsec:validity_asu_random_fields}. We observe that a first difference in terms of convergence exists between quantization-based compression and projection-based: for the former, \textit{only} \Cref{thm:bm2013_with_nonlinear_operator_compression} can be applied and the lower-order terms always have a \textit{poorer dependency on $\mu$} while for the latter, \Cref{thm:bm2013_with_linear_operator_compression} is applicable and lower-order terms do not necessarily depends on $\mu$. Indeed, the constants $\Sha_\xCov$ do not depend on $\mu$ for $\Cpp$, and for $\Crand, \Cspars$, \textit{when} the features' covariance $\xCov$ is diagonal. On the contrary, there is always a dependency on $\mu $ for $\Csketch$, and for $\Crand, \Cspars$ when $\xCov$ is not diagonal.
In practice, this means that, among projection-based compressors, regarding lower-order terms, the convergence is expected to be slower for random Gaussian projection.  

We now turn to the analysis of the impact of the choice of the compression on the dominant asymptotic term $\Tr{\Fhess^{-1}\aniac}$.

\subsection{Impact of the compression on the additive noise covariance}
\label{subsec:impact_cov_on_additive_noise}

In this section, we illustrate how distinct compressors lead to different covariances for the additive noise.
This shows how $\Tr{\Fhess^{-1}\aniac}$ is impacted  by the choice of a compressor. 

First recall that for \Cref{ex:cent_comp_LMS} in the context of \Cref{model:centralized}, with any compressor $\C $, the additive noise writes for any $k\in\OneToK$, as: $$\xikstaru \overset{\text{def.}~\ref{def:add_mult_noise}}{=} \xi_k(0)  \overset{\text{algo}~\ref{ex:cent_comp_LMS}}{=} \nabla F (\ws) - \C_k(\g_k(\ws)) \overset{\text{eq.}~\ref{eq:def_oracle}}{=}   - \C_k((\PdtScl{x_k}{\ws} - y_k) x_k  ) \overset{\text{model}~\ref{model:centralized}}{=} \C_k(\varepsilon_k x_k)\,.$$  
Also recall that $\aniac$ is defined as $\aniac := \E[(\xikstaru)^{\otimes 2}]= \E[\C(\varepsilon_k x_k)^{\otimes 2}] $. Moreover, note that $\C(\varepsilon_k x_k) \overset{\text{a.s.}}{=} \varepsilon_k \C(x_k)$ for all operators under consideration (this is immediate for linear operators and results from the scaling for quantization-based ones). Consequently 
\begin{equation}
\label{eq:aniac_comp_cov}
\aniac = \E[\varepsilon_k^2 \C( x_k)^{\otimes 2}]= \sigma^2 \E[ \C( x_k)^{\otimes 2}],
\end{equation}
as $\E[\varepsilon_k^2|x_k]=\sigma^2$. Ultimately, we have to study the covariance of $ \C( x_k)$, for $  x_k$ a random variable with second-moment $\xCov$.

We thus generically study the covariance of $\C(E)$, for~$E$ a random vector with distribution $p_M$ with second moment\footnote{Remark that we do not assume $\E[E]=0$. Indeed, all computations only depend on the\textit{ second-order moment} $M$ of $E$, not on its variance (and the convergence depends of the second\textit{-order moment} $\xCov$ of $x$, not its variance). It is clear,  that $ \E [\C(E)^{\kcarre}] $ does not depend on the fact that $E$  is centered: indeed, for $R$ a Rademacher $1/2$ independent of $E$, we have $\E [\C(E)^{\kcarre}] = \E [R^2] \E[ \C(E)^{\kcarre}] \overset{\perp}{=} \E [(R \C(E))^{\kcarre}] = \E [\C(R E)^{\kcarre}] $ and $R E$ is (1) centered (2) has the same second-moment as $E$. Remark that centering the covariates before learning does impact $\xCov$: indeed $\xCov=\E [(x)^{\kcarre}]= \E [(x-\E[X])^{\kcarre}]  + (\E[X])^{\otimes 2}$). Centering subtracts $(\E[X])^{\otimes 2}$ to the second moment, which is a rank-1 matrix, typically does not affect the smallest eigenvalue, but it can affect the top-eigenvalue.}   $\E[E^{\otimes 2}] = M$.

\begin{definition}[Compressor' covariance on $p_M$]
\label[definition]{def:cov_of_compression}
We define the following operator $\mathfrak{C}$ which returns \emph{the covariance of a random mechanism $\C$} acting on a distribution $p_M \in \mathcal{P}_M$,
\[
\mathfrak{C}:
	\begin{array}{lcl}
		\mathbb{C} \times \mathcal{P}_M & \to & \R^{d\times d} \\
		(\C, \ \ p_M) & \xmapsto & \E [\C(E)^{\kcarre}] \,,
	\end{array}
\]
where $E\sim p_M$ and the expectation is over  the joint randomness of $\C$ and $E$, which are considered independent, that is $\E [\C(E)^{\kcarre}] = \int_{\mathbb R^d} \E [\C(e)^{\kcarre}]  \mathrm{d} p_M(e)$.
\end{definition}

\begin{wrapfigure}[7]{r}{0.38\textwidth}
\vspace{-1.2cm}
  \begin{center}
    \includegraphics[width=1\linewidth]{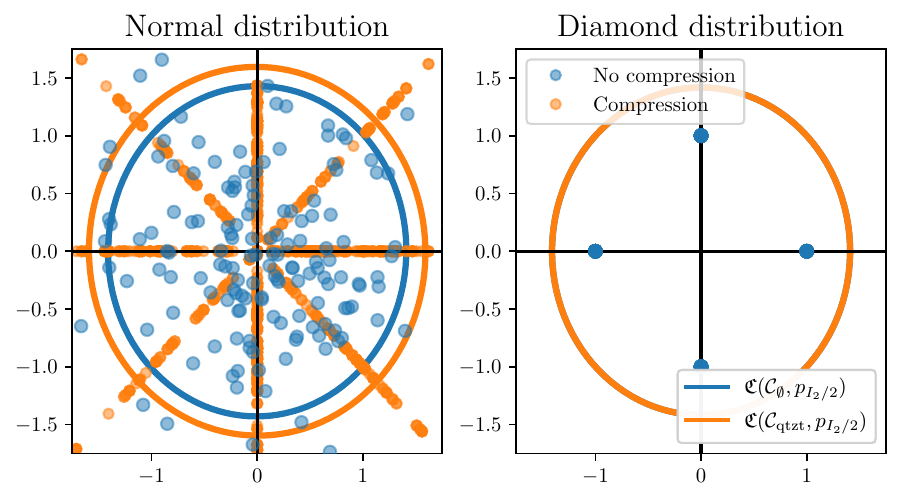}
  \end{center}
  \vspace{-0.8cm}\caption{Illustration of Remark~\ref{rem:covCpm}}
  \label{fig:normal_vs_diamond}
\end{wrapfigure}

Using a compressor $\C\in\mathbb{C}$, we therefore have by \Cref{eq:aniac_comp_cov}:
\begin{equation}
\label{eq:aniac_to_C}
    \aniac = \sigma^2 \mathfrak{C}(\C, p_\epsiXCov),
\end{equation}
where $ p_\epsiXCov$ is the marginal distribution of $x_k$ (for any $k$). 

\begin{remark}[Dependence on $p_M$, not only $M$]
\label[remark]{rem:covCpm}
Note that, for $\C=\Cquant$, there exist two distributions $p_M, p'_M$ with the same covariance $M$, such that $\mathfrak{C}(\C, p_M) \neq \mathfrak{C}(\C, p'_M)$. 
This is why we cannot simply denote $\mathfrak{C}(\C, M) $.
\end{remark}

Indeed, consider $d=2$ and (1) a normal distribution $\textcolor{tabblue}{E_1}\sim\mathcal{N}(0, I_2/2)$, vs (2) a \emph{diamond} distribution $\textcolor{tabblue}{E_2}\sim \mathbb{P}_\diamond$, such that
$\mathbb{P}_\diamond\{(1,0)\} = \mathbb{P}_\diamond\{(-1,0)\} = \mathbb{P}_\diamond\{(0,1)\} = \mathbb{P}_\diamond\{(0,-1)\}$ $ = 1/4 \,,
$
and thus $\Cov{\textcolor{tabblue}{E_1}}=\Cov{\textcolor{tabblue}{E_2}}=  I_2/2$. Then $\Cov{\textcolor{tabblue}{E_1}} \prec \Cov{\textcolor{taborange}{\Cquant(E_1)}} $, but  $\textcolor{taborange}{\Cquant(E_2)}\overset{\mathrm{a.s.}}{= } \textcolor{tabblue}{E_2}$ thus  $\Cov{\textcolor{tabblue}{E_2}} = \Cov{\textcolor{taborange}{\Cquant(E_2)}} $. We illustrate this  on \Cref{fig:normal_vs_diamond}: we represent $\textcolor{tabblue}{E_i}$ in blue and $\textcolor{taborange}{\Cquant(E_i)}$ in orange for $i=1$ (left) and $i=2$ (right). We also represent the covariance matrices by plotting the ellipses $\textcolor{tabblue}{\mathcal{E}_{\Cov{E_i}}}$ and   $\textcolor{taborange}{\mathcal{E}_{\Cov{\Cquant(E_i)}} } $, where $\mathcal{E}_M =\{ x \in \R^d, x^\top M^{-1} x = 4 \}$ (see \Cref{def:covariance_ellipse})\footnote{The constant 4 is chosen so that for Gaussian distributions, the expected fraction of points within the ellipse is $86,4\% \simeq 1-F_{\chi^2(2)}(4)$ }. 

We now compute for the compression operators, the value or an upper bound on  $\compCov[][]$.

\begin{proposition}[Compression and covariance]
\label[proposition]{prop:covariance_formula}
The following formulas hold: 
\begin{equation*}
\renewcommand{\arraystretch}{1.3}
    \begin{array}{ll}
     \compCov[][\Id_d][M] &= M\\
     \compCov[][q][M] &\preccurlyeq \widetilde{\mathfrak{C}}(\Cquant, M): =  M + \sqrt{\Tr{M} } \sqrt{\Diag{M}} - \Diag{M} \\ 
     & \text{(with equality if $\|E\|$ is a.s. constant under $  p_M$)}\\
     \compCov[][s][M] &= M + {(1 - p)}{p^{-1}} \Diag{M}\\
     \compCov[][\Phi][M] &= p^{-1} \bigpar{(\frac{h+1}{d+2} + \delta_{hd}) M + \bigpar{1 - \frac{h-1}{d-1}} \frac{\Tr{M} }{d+2}\Id_d},~\text{with }  \delta_{hd} =\frac{h-1}{(d-1)(d+2)} = O(\frac{1}{d})\\ 
     \compCov[][\mathrm{rd}h][M] &= p^{-1}\bigpar{\frac{h-1}{d-1} M + \bigpar{1 - \frac{h-1}{d-1}}  \Diag{M} }\\
     \compCov[][\mathrm{PP}][M]  &= p^{-1} M\,.
\end{array}
\end{equation*}
\end{proposition}

\textbf{Conclusion and interpretation.} Most compression operators induce \textit{both} a \textit{structured} noise \citep{flammarion_from_2015} which covariance scales with $H$ and an \textit{unstructured} noise, which covariance scales with $\Diag{H}$ or $\Id_d$---thus corresponding to an \textit{isotropic} noise.  

From the convergence standpoint, the asymptotic convergence rate scales with the trace $\Tr{\aniac H^{-1}}~=~\sigma^2 \Tr{\compCov H^{-1}}$. Therefore, the un-structured part in the noise is problematic  as $\Tr{ \aniac H^{-1})}$ will strongly depends on the smallest eigenvalue~$\mu$.  This comes from the fact that the compression induces a significant noise in directions in which the Hessian curvature is very limited (thus directions onto which the contraction towards the optimum in the algorithm is weak).

A particular case is when $H$ is diagonal (e.g. the features are \textit{centered} and \textit{independent}), we get the following corollary.

\begin{corollary}[Compression and covariance, diagonal case]
\label[corollary]{prop:covariance_formula_diagonal_case}
If $M$ is diagonal, then \Cref{prop:covariance_formula} is simplified to the following (with the same $\delta_{hd}$): 
\begin{equation*}
\begin{array}{llll}
 \compCov[][\Id_d][M] & = M & \compCov[][\Phi][M] & = p^{-1} \bigpar{(\frac{h+1}{d+2} + \delta_{hd}) M + (1 - \frac{h-1}{d-1}) \frac{\Tr{M} }{d+2}\Id_d} \\
 \compCov[][q][M] & \preccurlyeq \sqrt{\Tr{M} } \sqrt{M}  \hspace{0.6cm}\ &   \compCov[][\mathrm{rd}h][M] & = p^{-1} M\\
 \compCov[][s][M] & = p^{-1} M & \compCov[][\mathrm{PP}][M] &  = p^{-1} M.\\
\end{array}
\end{equation*}
\end{corollary}

\begin{remark}[Composition of compressors]
\label[remark]{rem:composition_compressors}
    For all compression schemes but $\Cquant$, we observe that $\compCov[][][M]$ is a function of $M$, which complements \Cref{rem:covCpm}. In that particular case, we can then denote $\mathfrak C (C ,M)$. This means that the lemma can be extended to any composition of compression schemes, for example to compute  
    $\mathfrak{C} ({C_1 \circ C_2, M})=\mathfrak {C} ({C_1 ,\mathfrak C ( C_2 ,M)}).$
\end{remark}

From  Proposition \ref{prop:covariance_formula} and Corollary \ref{prop:covariance_formula_diagonal_case} we can deduce certain generic comparisons between the asymptotic convergence rates, depending on the compression operator (for compression operators having the same variance bound). They are proven in \Cref{app:subsec:proof_particular_case}. In the following,  for any $a,b \in \R$, we use the notation $a\lesssim b$, to denote a \textit{systematic inequality} (i.e., $a\le b$) with a negligible difference as $d\to \infty$ (i.e., $a=b+O(1/d)$), and similarly for any two symmetric matrices  $A,B \in \mathcal S_d(\R)$, $A\precsim B$, for $A\preccurlyeq B$ and $A=B+O(1/d)$ as $d\to \infty$.

\begin{proposition}[Comparison between ${\Cpp, \Cspars, \Crandh, \Csketch}$, $\omega=d/h- 1$] 
\label[proposition]{prop:particular_cases} We consider $\C \in \{\Cpp, \Cspars$,  $\Crandh, \Csketch\}$ with $p=h/d$, such that $\C$ always satisfies \Cref{lem:compressor} with $\omgC = d/h- 1$. For any matrix~$M~\in~\R^{d \times d}$:
\begin{enumerate}[topsep=2pt,itemsep=1pt,leftmargin=0.5cm]
    \item \label{item:eq_trace_diag}  If $M$ is diagonal, then:
    \begin{itemize}
        \item $\compCov[][\mathrm{PP}][M]  = \compCov[][\mathrm{s}][M]  = \compCov[][\mathrm{rd}h][M]  = \frac{d}{h} M$, 
        \item $\mathrm{Tr} \big( \compCov[][\mathrm{PP}/\mathrm{s}/\mathrm{rd}h][M] M^{-1} \big) \leq \Tr{\compCov[][\Phi][M] M^{-1}} $.
    \end{itemize}
    This means that the asymptotic convergence rate does not depend on the choice of the  compressor between $\Cpp, \Cspars, \Crand$ in the diagonal case.
    \item \label{item:ineq_trace_diag} Moreover,  for any matrix $M$ with a \emph{constant diagonal} (e.g., we standardize\footnote{That means we center and rescale to get a variance of one for each feature.} the data in the pre-processing step, such that  $\Diag{M}=\Id_{d}$), we have:
    $$\mathrm{Tr}(\compCov[][\mathrm{PP}][M] M^{-1}) \le \mathrm{Tr}(\compCov[][\Phi][M] M^{-1}) \le \mathrm{Tr}(\compCov[][s][M] M^{-1})\le \mathrm{Tr}(\compCov[][\mathrm{rd}h][M] M^{-1}) \,,$$
    with strict inequalities if $M$ is not proportional to $\Id_d$. This means that we expect the asymptotic convergence rate to be faster for PP than Sparsification, Sketching, or Rand-$h$ (illustrated in experiments).
\end{enumerate}
\end{proposition}

In the next proposition, we compare compressors $\Cspars, \Cpp$ to $\Cquant$ for equal $\omgC = \sqrt{d}$ (we exclude $\Crandh$ and $\Csketch$ for which  $h$ must be an integer).
  
\begin{proposition}[Comparison between ${\Cpp, \Cquant,  \Cspars}$, $\omega=\sqrt{d}$ ] 
\label[proposition]{prop:particular_cases_quant}
We consider that $\C$ is in $\{\Cpp, \Cquant, \Cspars\}$ with $p=(\sqrt{d} + 1)^{-1}$, such that $\C$ always satisfies \Cref{lem:compressor} with $\omega = \sqrt{d}$. 
\begin{enumerate}[topsep=2pt,itemsep=1pt,leftmargin=0.5cm]
\item \label{item:ineq_trace_diag_quantiz} For any symmetric matrix $M$ diagonal, we have:
$$\Tr{\compCov[][\mathrm{PP}][M] M^{-1}} = \Tr{\compCov[][s][M] M^{-1}} \overset{\text{possib. } \ll }{\leq } \left(1 + \frac{1}{\sqrt{d}}\right)\Tr{ \widetilde{\mathfrak{C}}(\Cquant, M)M^{-1}}\,.$$
\item \label{item:ineq_trace_diag_quantiz_sparsif} If $M$ is not necessarily diagonal  but with a \emph{constant diagonal} (e.g., after standardization), then 
\begin{itemize}
    \item $\widetilde{\mathfrak{C}}(\Cquant, M) \preccurlyeq \compCov[][\mathrm{s}][M]$ 
    \item $\Tr{\compCov[][\mathrm{PP}][M] M^{-1}} \leq  \left(1 + \frac{1}{\sqrt{d}}\right)\Tr{ \widetilde{\mathfrak{C}}(\Cquant, M)M^{-1}} $ 
\end{itemize}
\end{enumerate}
This means that sparsification  is expected to always result in a poorer asymptotic convergence rate than quantization. Moreover, the  \textit{upper bound} on the covariance $\widetilde{\mathfrak{C}}(\Cquant, M)$ for  quantization itself leads to a worst bound than for PP.\footnote{Note that the behavior for quantization, apart from the upper bound $\widetilde{\mathfrak{C}}(\Cquant, M)$ is not quantified, it is thus possible that quantization performs even better than PP.} 
\end{proposition}

We now propose a detailed illustration of the results of  \Cref{prop:covariance_formula} and \Cref{prop:covariance_formula_diagonal_case}, first in a low-dimensional setting ($d=2$) and then in higher dimension  on synthetic and real datasets.

\subsubsection{Illustration of Proposition \ref{prop:covariance_formula} and Corollary \ref{prop:covariance_formula_diagonal_case} in dimension 2.}
\label{subsubsec:illustration_dim_2}

In order to  build intuition, we illustrate  \Cref{prop:covariance_formula} and  \Cref{prop:covariance_formula_diagonal_case}  in \Cref{fig:compression_scatter_plot,fig:compression_scatter_plot_diag}, showing how compression affects the additive noise covariance, in a simple 2-dimensional case,  for both a non-diagonal matrix $M$ (\Cref{fig:compression_scatter_plot}) and a diagonal one (\Cref{fig:compression_scatter_plot_diag}).

\begin{figure}
\vspace{-1cm}
  \begin{center}
    \includegraphics[width=1\linewidth]{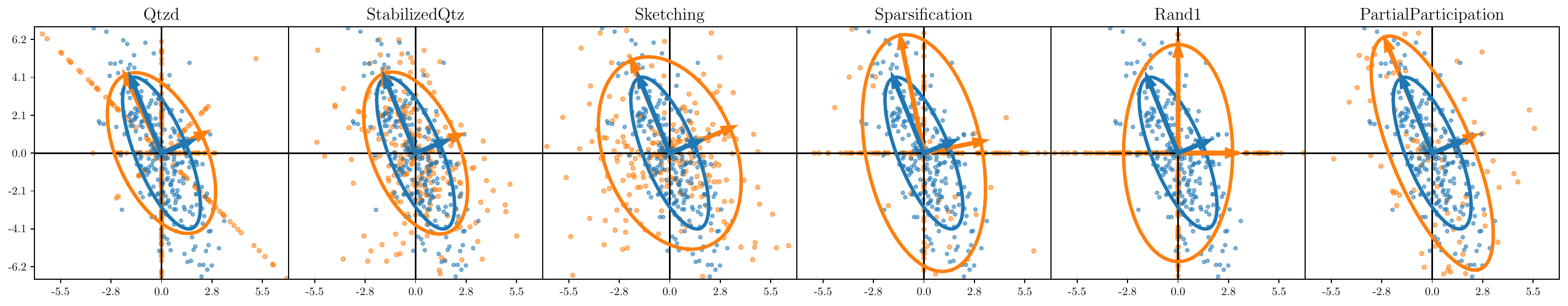}
  \end{center}
  \vspace{-0.5cm} 
  \caption{$\xCov$ not diagonal. Scatter plot of \textcolor{tabblue}{$(x_k)_{i=1}^K$}/ \textcolor{taborange}{$(\C(x_k))_{i=1}^K$} with its ellipse \textcolor{tabblue}{$\mathcal{E}_{\Cov {x_k}}$}/\textcolor{taborange}{$\mathcal{E}_{\Cov {\C (x_k)}}$}. \vspace{-0.25cm}}
  \label{fig:compression_scatter_plot}
\end{figure}

\begin{figure}
  \begin{center}
    \includegraphics[width=1\linewidth]{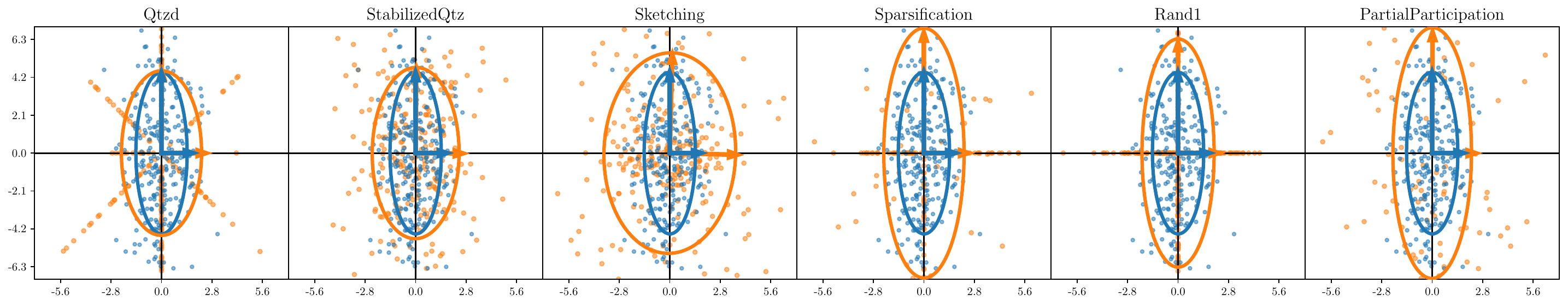}
  \end{center}
  \vspace{-0.5cm} 
  \caption{$\xCov$ diagonal. Scatter plot of \textcolor{tabblue}{$(x_k)_{i=1}^K$}/\textcolor{taborange}{$(\C(x_k))_{i=1}^K$} with its ellipse \textcolor{tabblue}{$\mathcal{E}_{\Cov {x_k}}$}/\textcolor{taborange}{$\mathcal{E}_{\Cov {\C (x_k)}}$}.}
  \label{fig:compression_scatter_plot_diag} \vspace{-0.5cm}
\end{figure}

More specifically, we consider features $(x_k)_{k \in \OneToK}$ sampled from $\mathcal{N}(0, M)$ where $M = Q D Q $, $D= \Diag{1, 10}$ and $Q$ is rotation matrix with angle $\pi/8$ (resp. $0$) in \Cref{fig:compression_scatter_plot} (resp.~\ref{fig:compression_scatter_plot_diag}). 
We represent the values of \textcolor{tabblue}{$x_k$} and \textcolor{taborange}{$\C (x_k)$}, unit-ellipses of the corresponding covariance matrices \textcolor{tabblue}{$\mathcal{E}_{\Cov{x_k}}$} and \textcolor{taborange}{$\mathcal{E}_{\Cov {\C (x_k)}}$} (see \Cref{def:covariance_ellipse}---recall that $\textcolor{tabblue}{\mathcal{E}_{\Cov{x_k}}} \subset \textcolor{taborange}{\mathcal{E}_{\Cov {\C (x_k)}}} \Leftrightarrow \textcolor{tabblue}{{\Cov{x_k}}} \preccurlyeq \textcolor{taborange}{{\Cov {\C (x_k)}}} $), as well as their two eigenvectors; we take $p = (1 + \sqrt{d})^{-1} = 0.41$, hence for $\C\in\{ \Cquant, \Cstabquant, \Cspars, \Cpp\}$ we have $\omgC = 1.41$ but for sketching and rand-$1$, we have $p = 1/2$ and $\omgC = (1-p)/p = 1$.

We make the following observations:\begin{itemize}[itemsep=1pt,leftmargin=1cm,noitemsep]
    \item[{[Qtz]}] For quantization and stabilized quantization, in the non-diagonal case, the eigenvectors of \textcolor{tabblue}{$\mathcal{E}_{\Cov{x_k}}$} and \textcolor{taborange}{$\mathcal{E}_{\Cov {\C (x_k)}}$} are slightly\footnote{On the figure, there are nearly aligned, but actually differ.} different (as $\sqrt{\Diag{M}}$ and $M$ are not jointly diagonalizable, as well as if $\Diag M$ is constant, although this case  is not presented here, but in \Cref{app:fig:compression_scatter_plot_std} in \Cref{app:subsec:proof_particular_case}). They are equal for the diagonal case (as $\sqrt{\Diag{M}}$ and $M$ are both diagonal so  the eigenvectors are aligned with the axis).  In both cases, the eigenvalue decay is reduced (from $\lambda_2/\lambda_1= 1/10$ without compression to  $1/\sqrt{10}$ with compression, which visually corresponds to a ``wider'' ellipse).

    This slower eigenvalue decay results from the \textit{unstructured-noise}, i.e., large noise on the weak-curvature direction, which is particularly visible on \Cref{fig:compression_scatter_plot_diag}. This is critical as it results in a potentially much larger limit rate, as  $\Tr{\compCov[][q][M] M^{-1}}\simeq \Tr{M^{-1/2}}$.
        
    \item[{[Skt]}] For sketching, the eigenvectors remain the same for \textcolor{tabblue}{$\mathcal{E}_{\Cov{x_k}}$} and \textcolor{taborange}{$\mathcal{E}_{\Cov {\C (x_k)}}$} (as $\Id_2$ and $M$ are  jointly diagonalizable, see \Cref{prop:covariance_formula_diagonal_case}), both in the diagonal and non-diagonal case.  However, the isotropic noise with covariance $\Id_2$ is visible (wide ellipse), also drastically impacting $\Tr{\compCov[][\mathrm{PP}][M] M^{-1}}\varpropto \Tr{M^{-1}}$.
    \item[{[Sp]}] For $p$-sparsification, eigenvectors are not aligned with the ones of $M$ in the non-diagonal case, but are in the diagonal case. In this latter case, the covariance $\compCov[][s][M] $ is proportional~to~$M$.
    \item[{[Rd]}] Same remarks hold for Rand-$1$ than for sparsification. We see that $\compCov[][\mathrm{rd}1][M] $ is diagonal, as expected.  Again, both operators induce an unstructured-noise in the non-diagonal case. 
    \item[{[PP]}] For PP, the covariances are always proportional (with factor $p^{-1}$), i.e., the ellipses have the same axis and \textcolor{taborange}{$\mathcal{E}_{\Cov {\C (x_k)}}$} is a scaled  version of  \textcolor{tabblue}{$\mathcal{E}_{\Cov{x_k}}$}.
\end{itemize}
We highlight the following points regarding pairwise comparisons: 
\begin{itemize}[topsep=0pt,itemsep=1pt,leftmargin=1cm,noitemsep]
     \item In the diagonal case, as stated by \Cref{item:eq_trace_diag} in \Cref{prop:particular_cases}, \textcolor{taborange}{${\Cov {\Cspars (x_k)}}$} and \textcolor{taborange}{${\Cov {\Cpp (x_k)}}$} are identical. \textcolor{taborange}{${\Cov {\Crand (x_k)}}$} would have been identical too if $p = 1 /d$  (but here we observe $\compCov[][\mathrm{rd}1][M] \preccurlyeq \compCov[][s/\mathrm{PP}][M]$ because the variance of rand-$1$ is smaller that for sparsification/PP). 
     \item In the non-diagonal case, from \Cref{item:ineq_trace_diag} in \Cref{prop:particular_cases}, we have $\Tr{\compCov[][\mathrm{PP}][M] M^{-1}} \leq \Tr{\compCov[][s][M] M^{-1}}$, however we do not have $ \compCov[][\mathrm{PP}][M] \preccurlyeq \compCov[][s][M]$, hence we can not conclude anything on \textcolor{taborange}{${\Cov {\Cpp (x_k)}}$} and \textcolor{taborange}{${\Cov {\Cspars (x_k)}}$}.
    \item In the non-diagonal scenario, we observe on \Cref{fig:compression_scatter_plot}, that $\compCov[][q][M] \preccurlyeq \compCov[][s][M]$ (as in \Cref{item:ineq_trace_diag_quantiz_sparsif} in \Cref{prop:particular_cases_quant}). 
\end{itemize}

\subsubsection{Illustration of Proposition \ref{prop:covariance_formula} and Corollary \ref{prop:covariance_formula_diagonal_case} in dimension \texorpdfstring{$d>2$}{d>2 }}

\label{subsec:computation_baniac}

\begin{figure}
    \centering     
    \begin{subfigure}{0.32\linewidth}
        \includegraphics[width=\linewidth]{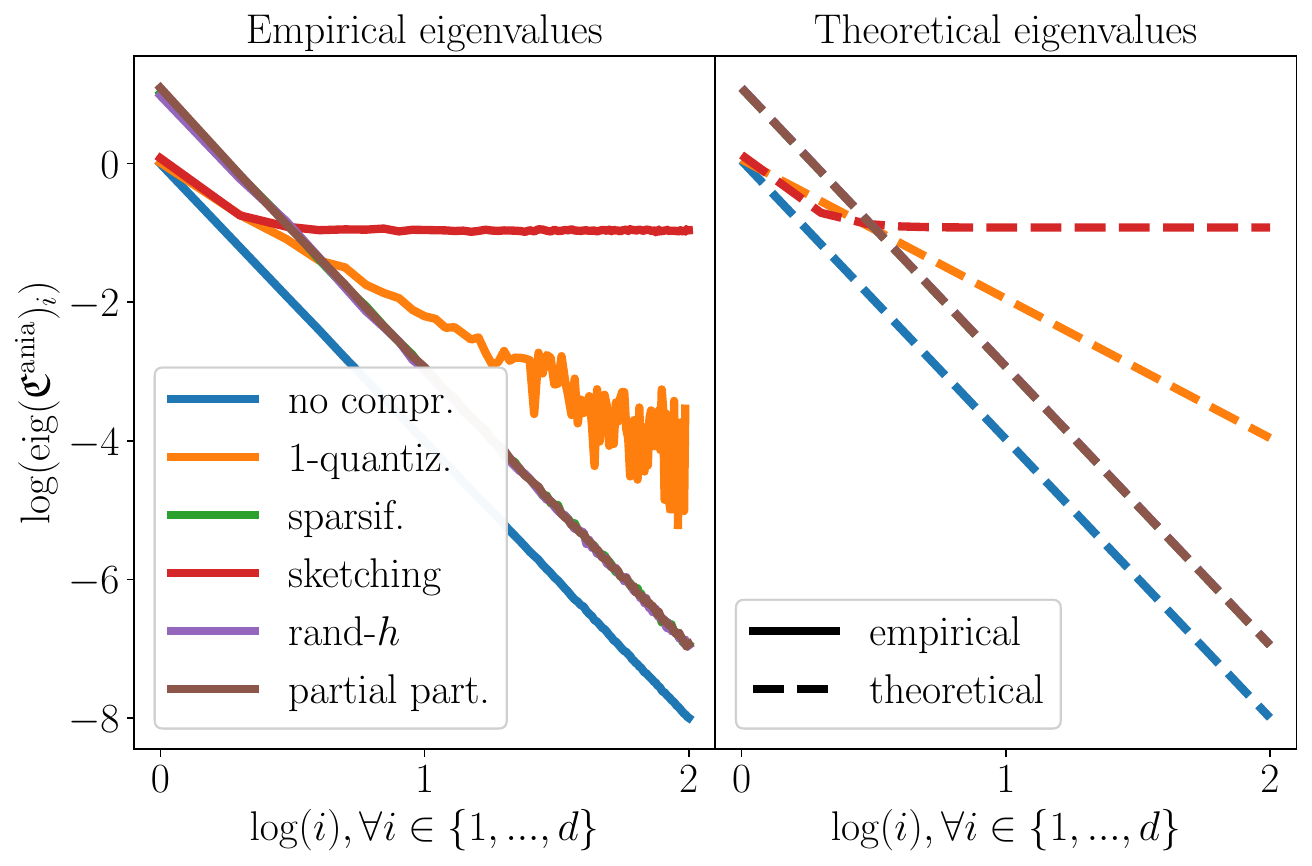}
        \caption{ \label{fig:eig_diag} $M$ diagonal, $d=100$}
    \end{subfigure}
    \begin{subfigure}{0.32\linewidth}
       \includegraphics[width=\linewidth]{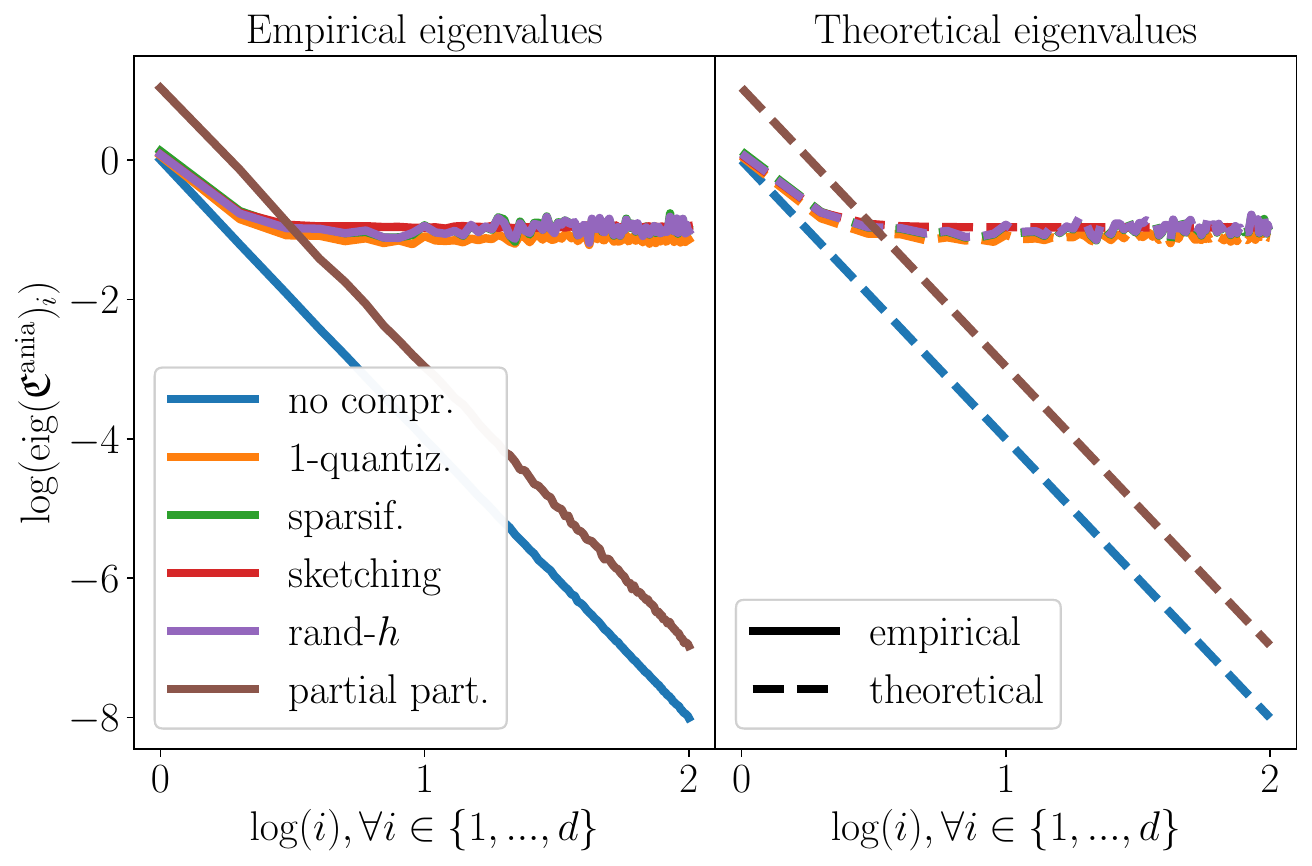}
        \caption{\label{fig:eig_ortho} $M$ non-diagonal, $d=100$}
    \end{subfigure}
    \begin{subfigure}{0.32\linewidth}
       \includegraphics[width=\linewidth]{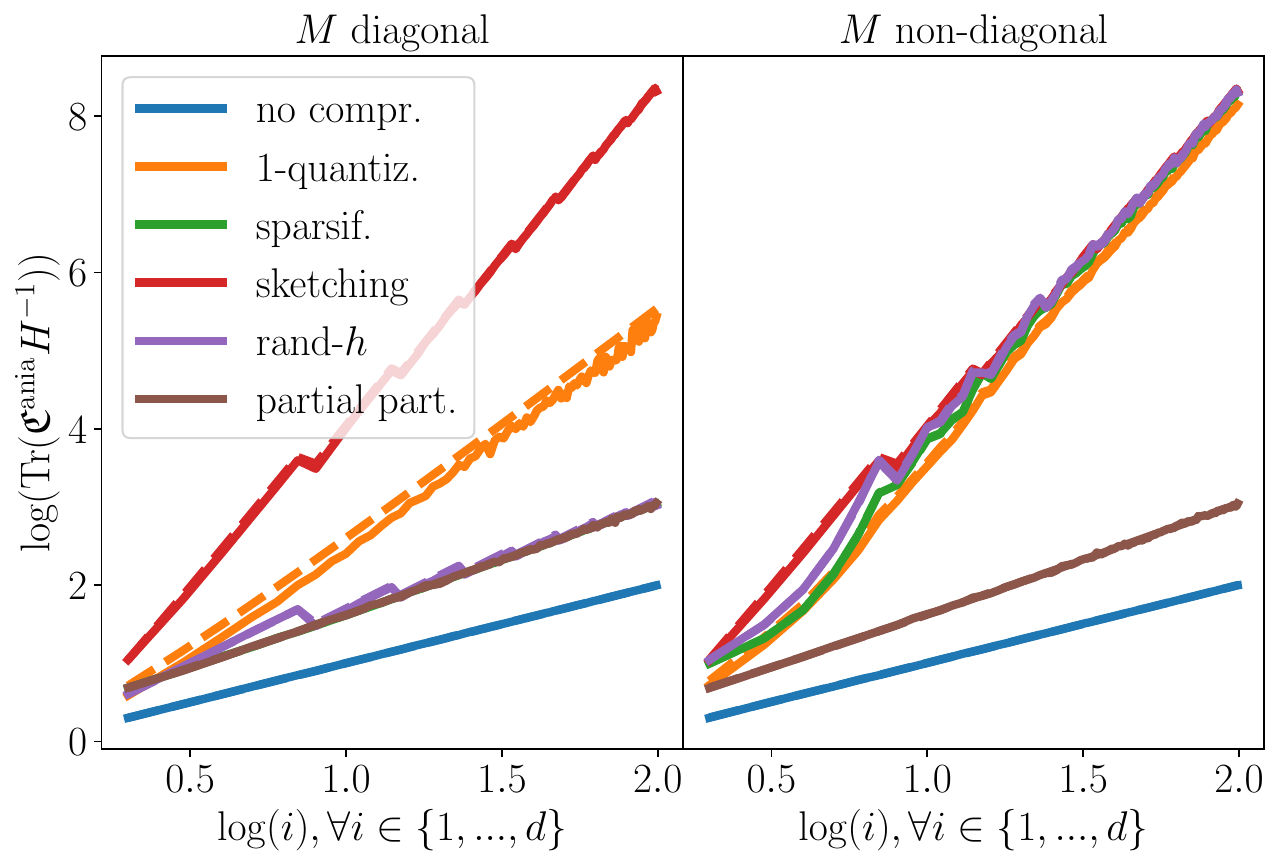}
        \caption{\label{fig:trace} $\Tr{\aniac M^{-1}}$, $d \in \llbracket 2, 100 \rrbracket$}
    \end{subfigure}
    \caption{Figures \ref{fig:eig_diag}~\&~\ref{fig:eig_ortho}: Eigenvalues of $\compCov[][][M]$.   \Cref{fig:trace}: $\Tr{\compCov[][][M] M^{-1}}$.~$K~=~10^4, \omgC = 10$, $M = Q \Diag{(1/i^4)_{i=1}^d} Q^T$ and $Q=\Id_d$ (on \ref{fig:eig_diag}~\&~\ref{fig:trace}-l) or $Q\sim \mathrm{Unif}(\mathcal{O}_d)$ (on \ref{fig:eig_ortho}~\&~\ref{fig:trace}-r). Plain lines: empirical values; dashed lines: theoretical formula or upper bound given by \Cref{prop:covariance_formula}.}
    \label{fig:eigenvalues}
\end{figure}

Another way of visualizing the structured and isotropic parts of the noise is by plotting the eigenvalues of $\compCov[][][M]$ in dimension $d=100$. This is done in \Cref{fig:eigenvalues}, in which we plot the eigenvalues in decreasing order for both $M$ and $\compCov[][][M]$, with Gaussian $p_M= \mathcal N (0, M)$ and $\mathrm{Sp}(M) = \{(1/i^4)_{i=1}^d\}$. We see that in the diagonal case, in \Cref{fig:eig_diag}, all operators but $\Cquant,\Csketch$ have a covariance proportional to $M$ (thus a slope $-4$ on a log/log scale), while $\Cquant$ is proportional to $\sqrt{M}$ (thus a slope $-2$) and $\Csketch$ has an isotropic component (thus eigenvalues not decreasing to 0). In \Cref{fig:eig_ortho} we see that only $\Cpp$ has a covariance proportional to $M$ while all other ones have an isotropic component (thus eigenvalues not decreasing to 0). We plot both empirical values and the ones obtained in \Cref{prop:covariance_formula}, which shows that the upper bound on quantization is reasonable in practice and acts as a safety check for other compression schemes.

We plot on \Cref{fig:trace} the theoretical and empirical $\Tr{\compCov[][][M] M^{-1} }$ again in two cases, diagonal and non-diagonal.
In the diagonal case, PP, sparsification, and rand-$h$ have the same behavior; their traces have the smallest value among all compressors. 
However, in the general case of non-diagonal features' covariance, all compression operators have similar slow performance except for PP. For $d=100$, all the compressors \emph{have $\omgC = 10$, but $\Tr{\compCov[][][M] M^{-1} }$  varies by several orders depending on the compressor}, illustrating again that compressors satisfying \Cref{lem:compressor} with the same $\omgC$ may have vastly different behaviors.

\begin{figure}
    \centering     
    \begin{subfigure}{0.49\linewidth}
        \label{fig:quantum}
        \includegraphics[width=\linewidth]{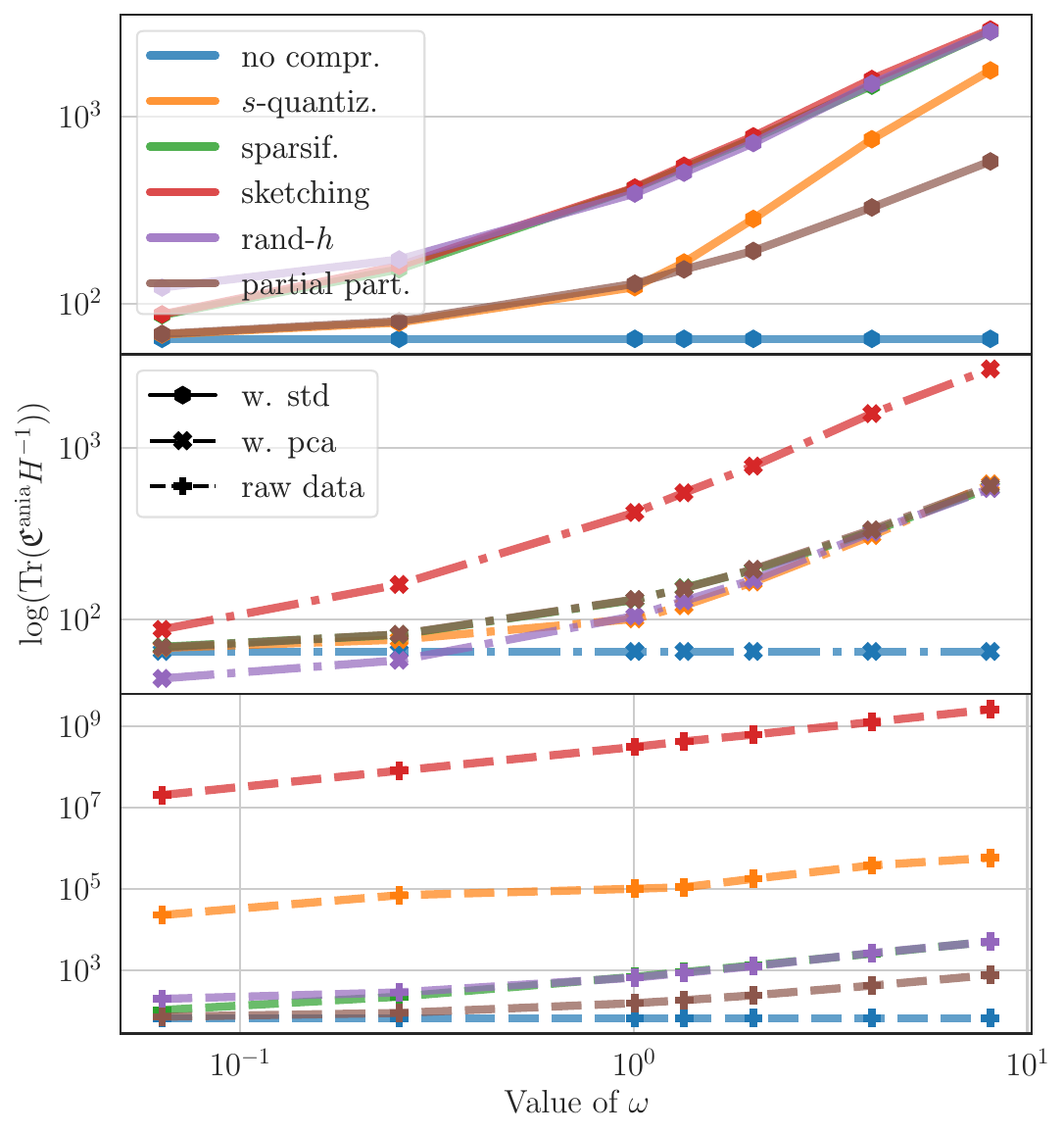}
        \caption{quantum (tabular dataset)}
    \end{subfigure}
    \begin{subfigure}{0.49\linewidth}
        \label{fig:cifar10}
        \includegraphics[width=\linewidth]{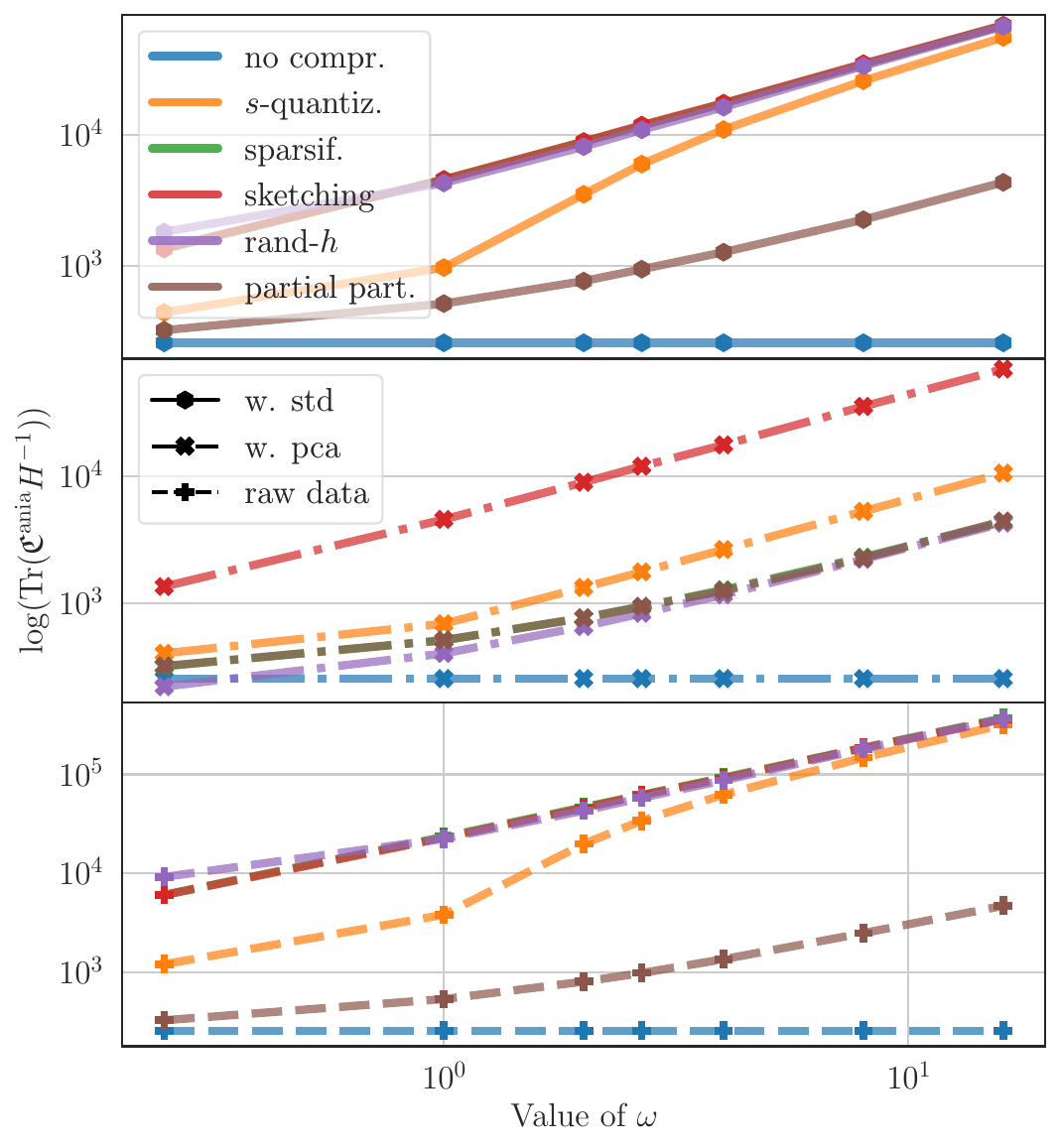}
        \caption{cifar10 (images)}
    \end{subfigure}
    \caption{$\Tr{\compCov[][][M] M^{-1} }$ w.r.t the level of $\omgC$ for quantum and cifar10. X/Y-axis are in log scale. Note that the plots may have different magnitudes.}
    \label{fig:real_dataset}
\end{figure}

Lastly, we perform the same experiments on $\Tr{\compCov[][][M] M^{-1}}$, but on non-simulated datasets, namely \texttt{quantum} \citep{caruana_kdd-cup_2004} and \texttt{cifar-10} \citep{krizhevsky_learning_2009}: in \Cref{fig:real_dataset} we plot $\Tr{\compCov[][][M] M^{-1} }$ w.r.t.~the worst-case-variance-level $\omgC$ of the compression in three scenarios: \textbf{(top-row)}---with data standardization, thus $\Diag{M}$ is constant equal to $1$; \textbf{(middle-row)}---with a PCA, thus with a diagonal covariance $M$ (note that this is for illustration purpose:  performing a PCA would be more expensive computationally than running \Cref{ex:cent_comp_LMS}); and \textbf{(bottom-row)}---without any data transformation.
As a pre-processing, we have resized images of the \texttt{cifar-10} dataset to a $16 \times 16$ dimension.  We adjust $h\in\Crandh,\Csketch$, and  $p\in\Cpp, \Cspars$ to make $\omgC$ vary. Besides, in order to also adjust $\omgC$ for quantization, we use the $s$-quantization (\Cref{def:s-quantization}) schema which generalizes $1$-quantization.

\begin{definition}[$s$-quantization operator]
\label[definition]{def:s-quantization}
Given $z \in \WW$, the $s$-quantization operator $\C_s$ is defined by $\C_s(z):= \mathrm{sign}(z) \times \| z \|_2 \times \frac{\chi}{s}$.
$\chi \in \WW$ is a random vector with $j$-th element defined~as: $ 
\chi:=~\left\{
    \begin{array}{ll}
        l + 1 & \mbox{with probability } s \frac{|z_j|}{\|z\|_2}  - l \,,\\
        l & \mbox{otherwise}\,
    \end{array}
\right.
$
where the level $l$ is such that $\ffrac{s|z_j|}{\lnrm z \rnrm_2} \in \left[l, l+1 \right[$.
\end{definition}

The $s$-quantization scheme verifies Assumption~\ref{item:urvb} with $\omgC= \min(d/s^2, \sqrt{d}/s)$. Proof can be found in \cite[][see Appendix A.1]{alistarh_qsgd_2017}. We do not compute $\boundMultPrime, \boundMult$ and the covariance~$\aniac$.

\textbf{Interpretation.} 
\textbf{(Top-row):} with standardization, the order predicted from \Cref{prop:particular_cases}.\ref{item:ineq_trace_diag} (large $\omega$), and  \Cref{prop:particular_cases_quant}.\ref{item:ineq_trace_diag_quantiz_sparsif} (low $\omega$) is obtained for both \texttt{quantum} and \texttt{cifar-10}: $ \Cpp \leq \Cquant \le  \Cspars  \simeq \Crandh \simeq \Csketch$.
For quantization, we observe two regimes: 1) when $\omgC$ tends to zero, quantization and PP outperform sketching, sparsification, and rand-$h$, that are equivalent.
2) when $\omgC$ increases, quantization changes from scaling as PP to scaling as the second group. 
\textbf{(Middle-row):}  in the diagonal regime, comments made for \Cref{fig:trace}-l are still valid. 
\textbf{(Bottom-row):} We observe that for a generic matrix $M$ (obtained from raw-data) there is no systematic order between compression schemes. This is un-avoidable as  the order for a ``$M$ diagonal''   and  ``$M$ with constant-diagonal'' is \textit{not} the same. We observe that: 
\begin{itemize}[noitemsep,topsep=0pt]
    \item for \texttt{quantum}, \ \  $ \Cpp \leq  \Cspars  \lesssim  \Crandh  \ll \Cquant \ll   \Csketch $ 
\item for \texttt{cifar-10}, $ \Cpp \ll \Cquant \ll  \Cspars  \simeq \Crandh \simeq \Csketch$.
\end{itemize}
We also observe that $\Csketch$, which is the only operator to always induce an isotropic component, may be much worse than all other compressors (e.g., on \texttt{quantum}). Ultimately, the order depends on the covariance matrix $M$. Here we observe that the raw-data behavior is close for \texttt{cifar-10} to the standardized version, while for \texttt{quantum} the order between compressors is the same for raw-data and diagonal (although the ratios are different). In \Cref{app:subsec:plot_cov_matrix} (\Cref{app:fig:cov_aniac}), we provide an illustration of the covariance matrices, that supports such interpretation. 

\subsection{Numerical experiments on Algorithm~\ref{ex:cent_comp_LMS}}
\label{subsec:numerical_experiments}

In this section, we run~\Cref{ex:cent_comp_LMS} on both synthetic and real datasets  to illustrate the combined theoretical results of \Cref{sec:theoretical_analysis,sec:application_compressed_LSR}.
In \Cref{fig:sgd}, we compare the compression operators to the baseline of no-compression.
We plot the excess loss of the Polyak-Ruppert iterate $F(\overline{w}_k) - F(w)$, versus the index in log/log scale.  Each experiment is conducted 5 times, with a new dataset generated from a new seed. The standard deviation of $\log_{10}(F (\overline{w}_k) - F(w))$ is indicated by the shadow-area.

\textbf{Setting:} (a) \textit{Synthetic dataset generation:} The dataset is generated using \Cref{model:centralized} with $K=10^7$, $\sigma^2 = 1$, an optimal point $\ws$ set as a constant vector of ones and a geometric eigenvalues decay of $D_1 = \Diag{(1/i^4)_{i=1}^d}$ (resp. $D_2 = \Diag{(1/i)_{i=1}^d}$).
For $i\in \{1,2\}$, the covariance matrix is $\xCov_{\{i\}} = Q D_{\{i\}} Q^T$, where $Q$ is either orthogonal matrix, or $Q = \Id_d$ in the case of a diagonal features' matrix. 
(b)~\textit{Real datasets processing:} We  resize images of the \texttt{cifar-10} dataset to a $16 \times 16$ dimension, and then for both datasets, we apply standardization. To compute the optimal point (and so to compute the excess loss), we run SGD over $200$ passes on the whole dataset and consider the last Polyak-Ruppert average as the optimal point $w_*$.
(c)~\textit{\Cref{ex:cent_comp_LMS}:}
We take a constant step-size $\gamma = 1/ (2 (\omgC + 1) R^2)$ with $R^2$ the trace of the features' covariance,
and  $w_0 = 0$ as initial point. We set the batch-size $b=1$ and the compressor variance $\omgC = 10$ for synthetic datasets. 
For \texttt{cifar-10} and \texttt{quantum}, we run \Cref{ex:cent_comp_LMS} for $5\times10^6$ iterations (it corresponds to $100$ passes on the whole dataset) with a batch-size $b = 16$, and using a $s$-quantization (\Cref{def:s-quantization}). We set $s=16$ for \texttt{cifar-10} (factor 2 compression) and $s=8$ for \texttt{quantum} (factor 4 compression), the compressor variance is therefore $\omgC \approx 1$ for both datasets. These settings are summarized in \Cref{tab:settings_exp,tab:settings_exp_real_dataset} in \Cref{app:subsec:settings_xp}.
Additionally, to illustrate \Cref{cor:bm2013_with_nonlinear_operator_compression}, we plot on \Cref{fig:sgd_ortho_horizon} the final excess loss after running \Cref{ex:cent_comp_LMS} with an horizon-dependent step-size $\gamma = K^{-2/5}$, computed for seven values of $K\in\{10^i, i \in \llbracket 1, 7 \rrbracket\}$.

\textbf{Interpretation---$\xCov$  diagonal} (\Cref{fig:sgd_diag}). 
For sparsification, rand-$h$, and PP  (linear compressors), the rate of convergence is given by \Cref{thm:bm2013_with_linear_operator_compression}. As stated by \Cref{prop:covariance_formula_diagonal_case}, the covariance $\aniac$ is proportional to $\epsiXCov$ leading to a $O(1/K)$ rate. We indeed observe in \Cref{fig:sgd_diag} that excess loss is linear in a log/log scale.

For non-linear compression operators, the rate is given by \Cref{thm:bm2013_with_nonlinear_operator_compression}. 
On the one hand, $1$-quantization results in a slower eigenvalues' decay, leading to a larger $\Tr{\aniac \xCov^{-1}}$, thus a slower convergence than linear compressors. On the other hand, for sketching, covariance has a purely isotropic part scaling with $\Id_d$, which causes $\Tr{\aniac \xCov^{-1}}$ to strongly depend on the strong-convexity coefficient $\mu$ resulting in an extremely  large constant. Both behaviors  are observed  in \Cref{fig:sgd_diag}.

\textbf{Interpretation---$\xCov$ not diagonal} (\Cref{fig:sgd_ortho,fig:sgd_ortho_p1}). In the case of the high eigenvalues' decay of $\xCov_1$ ($\mu = 10^{-8}$), the only compressor that shows in \Cref{fig:sgd_ortho} a linear rate of convergence in the log/log scale is PP.
All others exhibit a saturation phenomenon after a certain number of iterations. This is again due to the unstructured part of the noise for all other compressors, as given by~\Cref{prop:covariance_formula}.
Besides, we also note an increase of the excess loss after some iterations that  is likely  caused by the accumulation of noise on axis onto which the curvature of $H$ is weak (but the isotropic noise is not). 
However, taking the optimal horizon-dependent step-size given by \Cref{cor:bm2013_with_nonlinear_operator_compression}, we recover on \Cref{fig:sgd_ortho_horizon} for all compressor $\C$ the sub-linear convergence rate of PP shown at \Cref{fig:sgd_ortho}, reducing by a factor $100$ the excess loss w.r.t. to the scenario where $\gamma = 1 / 2 (\omgC + 1) R^2)$. While using a small step-size is slightly worse for SGD, it reduces the second and third term of the variance in \Cref{thm:bm2013_with_nonlinear_operator_compression} that depends on $\mu$ for other compressors. 
And in the scenario of a slow eigenvalues' decay ($\mu = 10^{-2}$),  we observe   on \Cref{fig:sgd_ortho_p1} that all compressors reach the sub-linear rate (same slope -1 on the log/log plot), but with different  constants. This illustrates \Cref{thm:bm2013_with_linear_operator_compression,thm:bm2013_with_nonlinear_operator_compression} in the case of moderate coefficient $\mu$ where we expect the second and third parts of the variance term to be negligible.

\begin{figure}[t]
\resizebox{\linewidth}{!}{
\begin{tabular}{ccc}
\toprule
\multicolumn{2}{c}{Synthetic dataset} & Real datasets \\
\cmidrule(lr){1-2}\cmidrule(l){3-3}
 
    \centering     
    \begin{subfigure}{0.32\linewidth}
        \includegraphics[width=\linewidth]{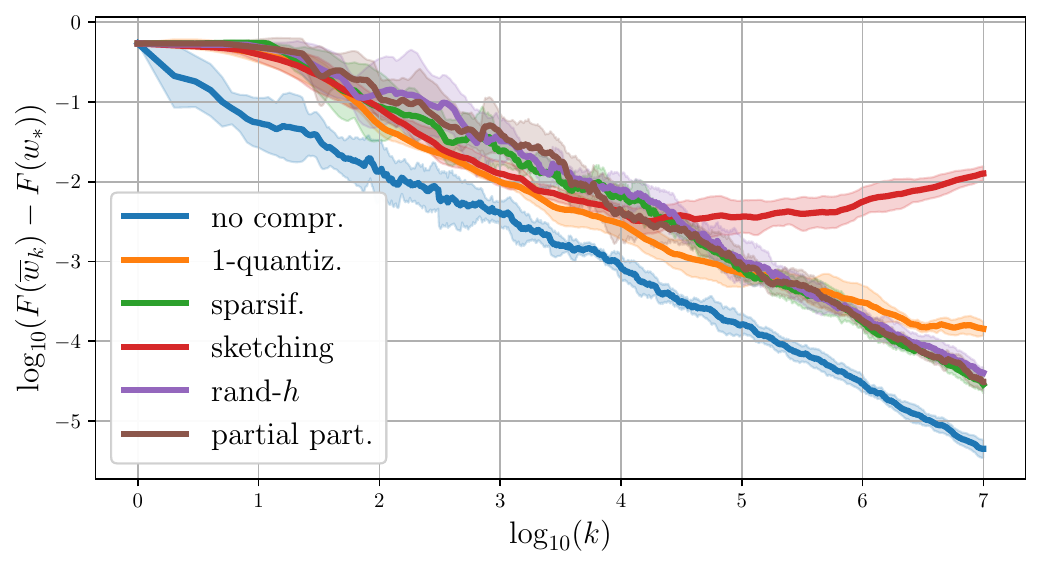}
        \caption{$\xCov_1$ diagonal. \label{fig:sgd_diag}}
    \end{subfigure} &
    \begin{subfigure}{0.32\linewidth}
        \includegraphics[width=\linewidth]{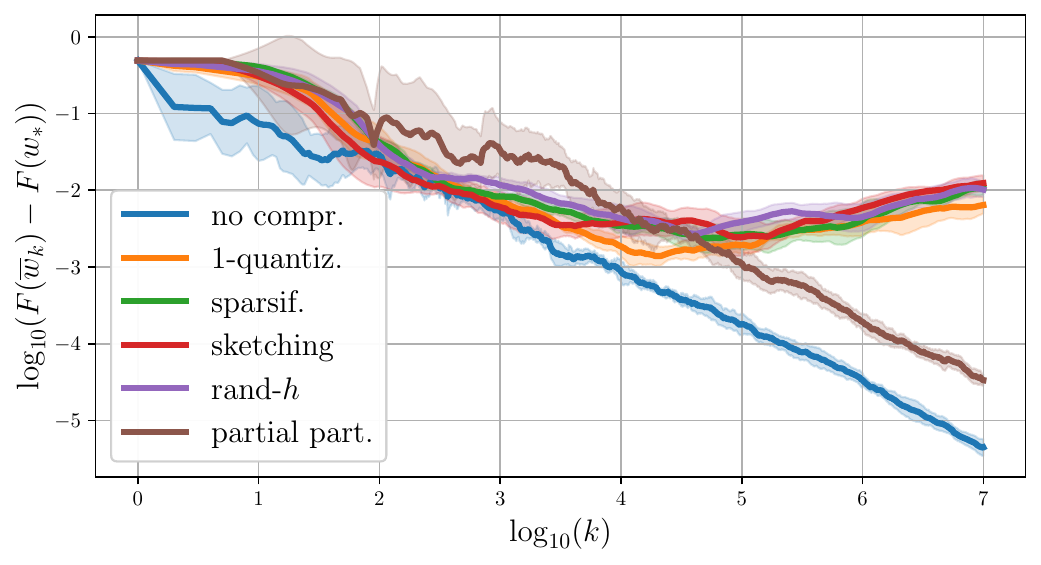}
        \caption{ $\xCov_1$ not diagonal. \label{fig:sgd_ortho}}
    \end{subfigure} &
    \begin{subfigure}{0.32\linewidth}
        \includegraphics[width=\linewidth]{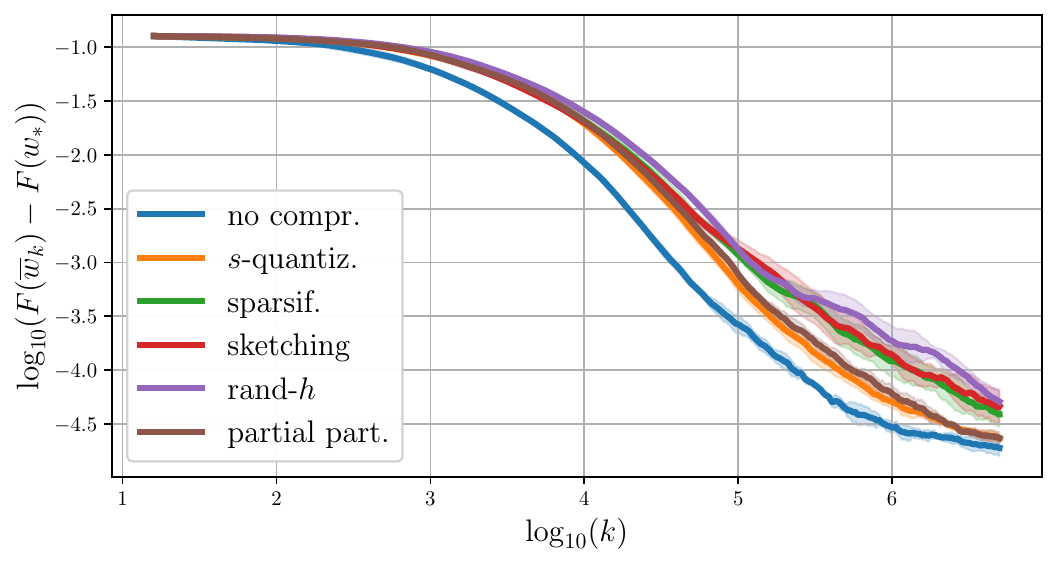}
        \caption{\texttt{quantum} \label{fig:sgd_quantum}}
    \end{subfigure} \\
    \begin{subfigure}{0.32\linewidth}
        \includegraphics[width=\linewidth]{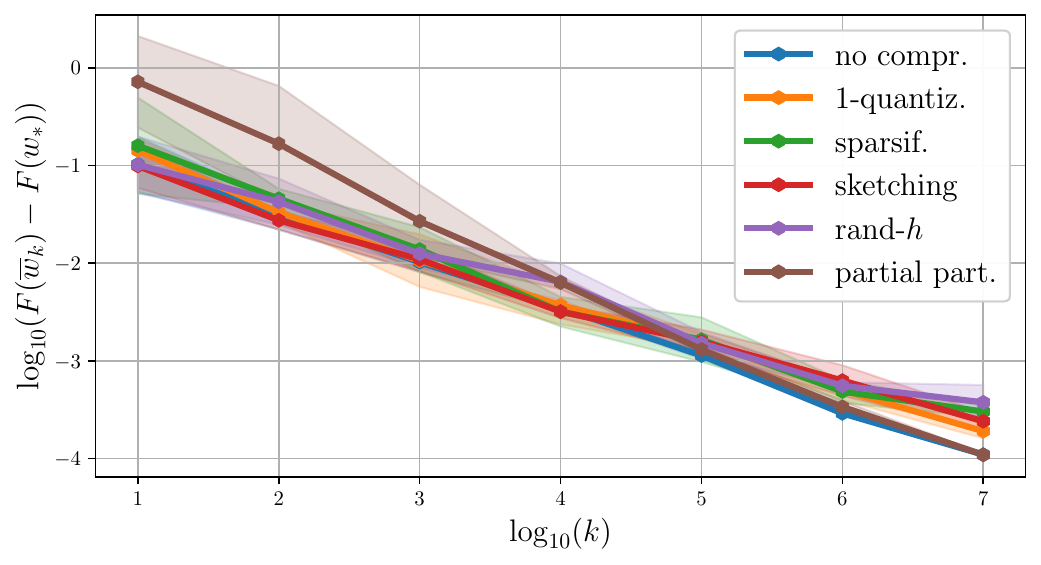}
        \caption{ $\xCov_1$ not diagonal, $\gamma = K^{-2/5}$. \label{fig:sgd_ortho_horizon}}
    \end{subfigure} &
    \begin{subfigure}{0.32\linewidth}
        \includegraphics[width=\linewidth]{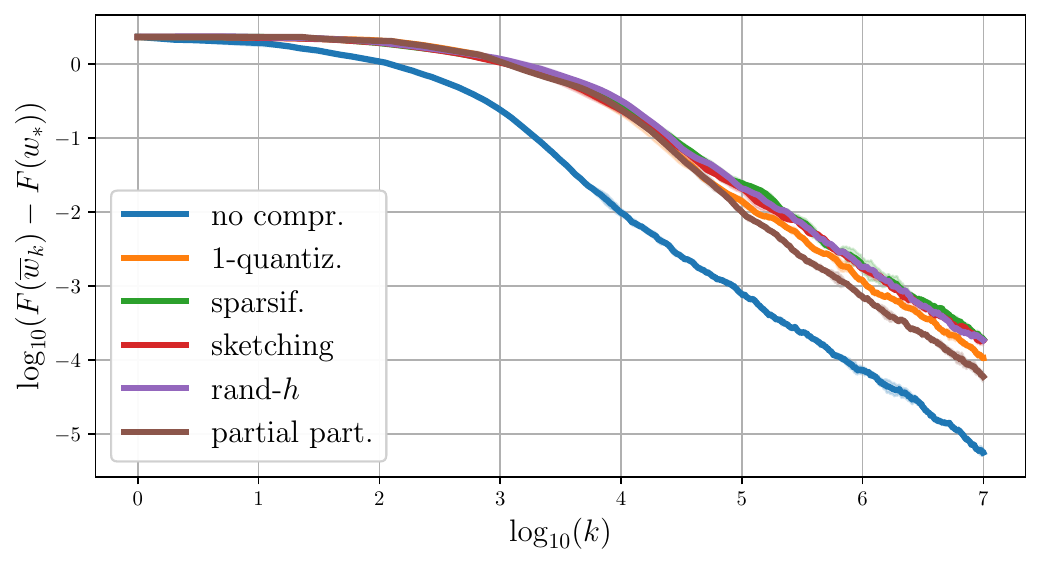}
        \caption{$\xCov_2$ not diagonal. \label{fig:sgd_ortho_p1}}
    \end{subfigure} &
    \begin{subfigure}{0.32\linewidth}
        \includegraphics[width=\linewidth]{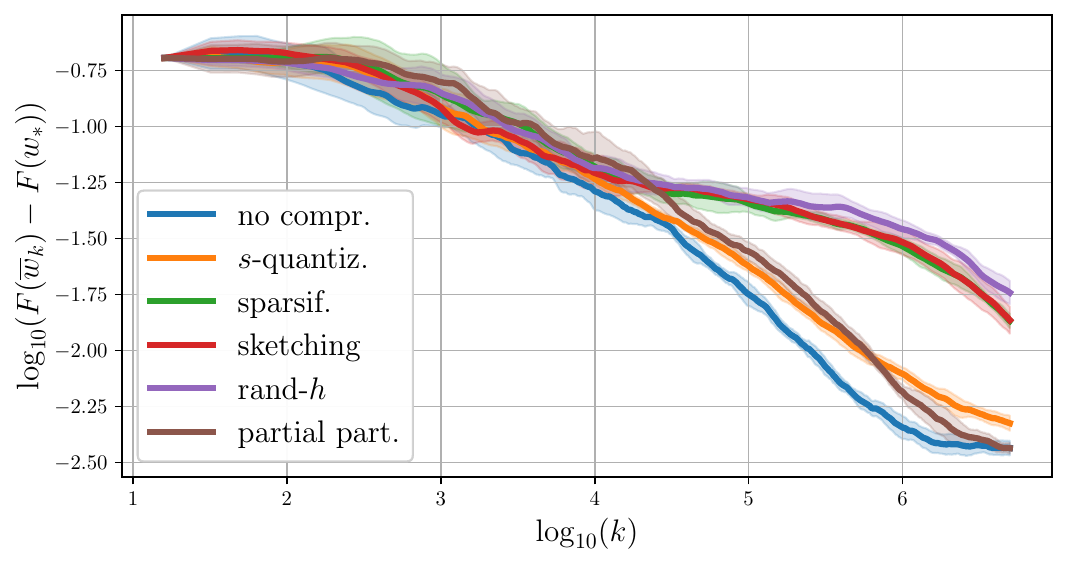}
    \caption{\texttt{cifar10} \label{fig:sgd_cifar10}}
    \end{subfigure} \\
    \bottomrule
\end{tabular}}
    \caption{Logarithm excess loss of the Polyak-Ruppert iterate for a single client ($N=1$). \label{fig:sgd}}
\end{figure}

\textbf{Interpretation - real datasets, $\xCov$ with constant diagonal} (\Cref{fig:sgd_quantum,fig:sgd_cifar10}). As we use $s$-quantization, this experience is going beyond \Cref{prop:covariance_formula,prop:particular_cases_quant} which only apply to $1$-quantization. In the case of covariance with constant diagonal, \Cref{prop:particular_cases_quant} states that $1$-quantization is better than projection-based compressors and comparable to partial participation. In practice, we observe that $s$-quantization performs competitively with PP and outperforms all other compressors.  
Besides, the asymptotic behavior is consistent with \Cref{fig:real_dataset} (top-row) for $\omega=1$, where the order $\Cpp \simeq \Cquant \ll \Cspars\simeq \Crandh\simeq \Csketch$ is observed.

\subsection{Conclusion}
\label{subsec:ccl}
In this section, we investigated how the compression scheme choice impacts the convergence rate, first by showing that quantization-based  and projection-based methods respectively satisfy  \Cref{thm:bm2013_with_nonlinear_operator_compression} and \Cref{thm:bm2013_with_linear_operator_compression}, resulting in different non-asymptotic behaviors. In the asymptotic regime, in both cases, the averaged excess loss scales as $\Tr{\xCov^{-1}\aniac}/K$. We then analyzed the impact of the most-used schemes on this limit rate. Overall, it appears that all compression schemes typically generate an \textit{unstructured-noise}, which covariance does not scale with $\xCov$, contrarily to the classical un-compressed \Cref{ex:LMS}. The one exception is PP, which corresponds (on a single worker) to performing fewer iterations. For other compression schemes, we show the impact of the covariance $\xCov$: depending on the correlation between features ($\xCov$ diagonal or not) and on the pre-processing (e.g., standardization for which $\xCov$ has diagonal constant), the ordering between compression scheme varies.
In many cases, this highlights the need for an additional regularisation when running \Cref{ex:cent_comp_LMS}: all compression schemes (but PP) result in a significant noise that accumulates along the low curvature directions. Our results can be extended to  the ridge (a.k.a., Tikhonov) regularized case~\citep[see][]{dieuleveut_harder_2017}, which creates an additional bias but changes the rate $\Tr{\xCov^{-1}\aniac}/K$ into $\Tr{(\xCov+\lambda \Id)^{-1} \aniac}/K$.  The theoretical optimal choice for $\lambda$ depending on $\xCov$ and the compression scheme could be obtained from our analysis but is left as future work.

We now turn to the distributed/federated case, which motivates the study of compression schemes for practical applications. 
\section{Application to Federated Learning}
\label{sec:federated_learning}

In this section, we consider \Cref{ex:dist_comp_LMS} under \Cref{model:fed}, which corresponds to heterogeneous Federated Learning on a network composed of $N$ clients. We hereafter consider two particular cases naturally raising from \Cref{model:fed}: covariate-shift and optimal-point-shift. Note this results can easily be extended to the case of a  heterogeneous level of noise by clients.

First, in \Cref{subsec:fl_sigma},  the \textit{covariate-shift} case, i.e., \Cref{model:fed} with $\ws^i =\ws$ for all~$i$ (thus the  distribution of $y^i$ conditional to $x^i$ does not change between workers), but on the other hand, the features' marginal distributions are different, in particular, $\xCov_i \neq \xCov_j$. Second,  in \Cref{subsec:fl_wstar}, the \textit{optimal-point-shift} case, i.e., for each client $i, j \in \OneToN$, their optimal points are different $\ws^i \neq \ws^j$, but $\xCov_i = \xCov_j$.
In the rest of the section, we  denote  $ \overline{\xCov} := \frac{1}{N} \sum_\iN \xCov_i $, $ \overline{R}^2 := \frac{1}{N} \sum_\iN R^2_i $, and we have $F(w_k) - F(\ws) = \frac{1}{2} \PdtScl{\etakm}{\overline{\xCov} \etakm}$.

\subsection{Heterogeneous covariance}
\label{subsec:fl_sigma}

In this section, we first show that  \Cref{thm:bm2013_with_linear_operator_compression,thm:bm2013_with_nonlinear_operator_compression} on \eqref{eq:LSA} from \Cref{sec:theoretical_analysis} can be applied to the Federated Learning case within the scenario of covariate-shift. \Cref{cor:fl_heter_cov} is proved in \Cref{app:subsec:validity_asu_random_fields_fl_sigma}. 

\begin{corollary}[\Cref{ex:dist_comp_LMS} with covariate-shift]
\label[corollary]{cor:fl_heter_cov}
Consider \Cref{ex:dist_comp_LMS} under \Cref{model:fed} with $\ws^i = \ws^j$ (and potentially $\xCov_i \neq \xCov_j$).
\begin{enumerate}[topsep=2pt,itemsep=0pt,leftmargin=0.5cm]
\item For a compressor $\C \in \{ \Cquant, \Cstabquant, \Crandh, \Cspars,$ $\Csketch, \Cpp \}$, 
\Cref{thm:bm2013_with_nonlinear_operator_compression} holds, with $\Fhess = \overline{\xCov}$, $\Ftrace^2 = \overline{R}^2$, $\boundAdd = (\omgC +1) \overline{R}^2 \sigma^2 / N$, $\boundMult = (\omgC + 1) \max_{i \in  \OneToN}(R_i^2) / N$, $\boundMultPrime = \omgCOne \sigma 
 \max_{i \in  \OneToN}(R_i^2)  / N$. 
\item Moreover for any linear compressor $\C \in \{\Crandh, \Cspars,$ $\Csketch, \Cpp \}$,  \Cref{thm:bm2013_with_linear_operator_compression} holds,  with  the same constants and $\ShaA = \sigma^2 \max_{i \in \OneToN} (\Sha_{\xCov_i}) / N$ and $\ShaM = \max_{i \in \OneToN} (R_i^2 \Sha_{\xCov_i}) /N$, with~$(\Sha_{\xCov_i})_\iN$ given in \Cref{cor:value_constants_in_thm}.
\end{enumerate}
\end{corollary}

The Hessian of the objective function is now $\overline{H}$, and \Cref{thm:bm2013_with_nonlinear_operator_compression,thm:bm2013_with_linear_operator_compression} still hold.  The proof consists in showing that with \Cref{lem:compressor}, \Cref{asu:main:second_moment_noise,asu:main:bound_add_noise,asu:main:bound_mult_noise_lin,asu:main:baniac_lin} on the resulting random field~$(\xi_k)_{k \in \N^*}$ are valid, with the constants given above.

In order to understand the impact of the compressor on the limit convergence rate, we establish a formula for $\aniac$ similar to \Cref{eq:aniac_to_C}.
In the setting of covariate-shift, we have for any clients $i,j\in\OneToN$, $\ws^i = \ws^j$, thus 
\begin{align*}
\xikstaru \overset{\text{def.}~\ref{def:add_mult_noise}}{=} \xi_k(0)  & \overset{\text{algo}~\ref{ex:dist_comp_LMS}}{=} \nabla F (\ws) -   \frac{1}{N} \sum_\iN \C_k^i(\g_k^i(\ws)) \\
&\overset{\text{eq.}~\ref{eq:def_oracle}}{=}   - \frac{1}{N} \sum_\iN \C_k^i((\PdtScl{x_k^i}{\ws} - y_k^i) x_k^i  ) \overset{\text{model}~\ref{model:fed} }{\underset{\text{ with }   \ws^i = \ws^j}{=}} \frac{1}{N} \sum_\iN \C_k^i(\varepsilon_k^i x_k^i)\,. 
\end{align*}
Next for all operators under consideration we have $\C^i_k(\varepsilon_k^i x_k^i) \overset{\text{a.s.}}{=} \varepsilon_k^i \C_k^i(x_k^i)$, thus, with $p_{\xCov_i}$ denoting the distribution of $ x_k^i $ with covariance $\xCov_i$, we have: 
\begin{align}
    \aniac &= \E\left[(\xikstaru)^{\otimes 2}\right] = \FullExpec{\bigpar{\frac{1}{N} \sum_\iN \C_k^i(\varepsilon_k^i x_k^i)}^{\kcarre}} \overset{\text{indep. of } (\C_k^i)_{i=1}^d}{=} \frac{1}{N^2} \sum_\iN \FullExpec{\C_k^i(\varepsilon_k^i x_k^i)^{\kcarre}} \notag \\ & = \frac{\sigma^2}{N^2} \sum_\iN \FullExpec{\C_k^i( x_k^i)^{\kcarre}}
    \overset{\mathrm{Def.~\ref{def:cov_of_compression}}}{=}\frac{\sigma^2}{N^2} \sum_\iN \compCov[i][k] \overset{\text{notation}}{=:} \frac{\sigma^2}{N} \overline{\mathfrak{C}((\C^i, p_{\epsiXCov_i})_\iN)} \,. \label{eq:aniac_to_C_dist}
\end{align} 
The operator $\overline{\mathfrak{C}((\C^i, p_{\epsiXCov_i})_\iN)}$ generalizes the notion of \emph{compressor's covariance} (\Cref{def:cov_of_compression}) to the case of multiple clients, and \Cref{eq:aniac_to_C_dist} corresponds to \Cref{eq:aniac_to_C}.

\begin{remark}[All clients use the same \textit{linear} compressor]
\label[remark]{rem:lin_FL_compcov}
If for all $i\in \OneToN$, $\C^i \overset{(d)}{=} \C$ and  $\C \in \{\Cpp, \Cspars,\Crandh, \Csketch\}$, leveraging \Cref{rem:composition_compressors}, we have
$$ \overline{\mathfrak{C}((\C^i, p_{\epsiXCov_i})_\iN)}  = \mathfrak{C}(\C, {\overline{H}})\,.$$ 
The analysis of \eqref{eq:LSA} on a single worker made in \Cref{sec:application_compressed_LSR} is still valid in this setting with now the Hessian of the problem being equal to the average of covariance $\overline{\xCov}$. \Cref{cor:fl_heter_cov} and \Cref{eq:aniac_to_C_dist} prove  that the case of covariate-shift is identical to the centralized setting with a variance  reduced by a factor $N$.
\end{remark}

\begin{remark}[Varying compressor/compression-level, or non-linear compression]
In most other cases, the computation of $ \frac{\sigma^2}{N} \overline{\mathfrak{C}((\C^i, p_{\epsiXCov_i})_\iN)} =\frac{\sigma^2}{N^2} \sum_\iN \compCov[i][i]  $ is possible  using the results of \Cref{subsec:impact_cov_on_additive_noise}
\end{remark}

Overall, in the covariate-shift case, most insights from the centralized case remain valid, especially, client sampling (i.e., PP) is the safest way to limit the impact of compression. Moreover, the trade-offs and ordering between compressors remain preserved, especially regimes in which quantization outperforms other competitors.

\subsection{Heterogeneous optimal point}
\label{subsec:fl_wstar}

Hereafter, we focus on the case of heterogeneous optimal points and consider  that all clients share the same covariance matrix, i.e. for any $i,j\in\OneToN$,  $\xCov_i = \xCov$, but potentially $\ws^i \neq \ws^j$. This can be seen as a case of \textit{concept-shift} \citep{kairouz_advances_2019}, and we also refer to the situation as \textit{optimal-point-shift}. This setting could eventually be combined with the covariate-shift case.  Similarly, \Cref{thm:bm2013_with_linear_operator_compression,thm:bm2013_with_nonlinear_operator_compression} on \eqref{eq:LSA} from \Cref{sec:theoretical_analysis} can be applied.

\begin{corollary}[\Cref{ex:dist_comp_LMS} with concept-shift]
\label[corollary]{cor:fl_heter_wstar}
Consider \Cref{ex:dist_comp_LMS} under \Cref{model:fed} with $\xCov_i = \xCov_j$ (and potentially $\ws^i \neq \ws^j$).
\begin{enumerate}[topsep=2pt,itemsep=0pt,leftmargin=0.5cm]
\item For a compressor $\C \in \{ \Cquant, \Cstabquant, \Crandh, \Cspars,$ $\Csketch, \Cpp \}$, 
\Cref{thm:bm2013_with_nonlinear_operator_compression} holds, with $\Fhess = \xCov$, $\Ftrace^2 = R^2$, $\boundAdd = \frac{R^2 (\omgC +1)}{N} (\kappa \Tr{H \Cov{ W_*}}  + \sigma^2 )$ with $W_* \sim \mathrm{Unif}(\{\ws^i, i\in \OneToN \})$, $\boundMult = (\omgC + 1)^2 / N$, and $\boundMultPrime = \omgCOne R^2 \sigma/ N$. 
    \item Moreover for any linear compressor $\C \in \{\Crandh, \Cspars,$ $\Csketch, \Cpp \}$,  \Cref{thm:bm2013_with_linear_operator_compression} holds,  with  the same constants and $\ShaA = \sigma^2 \Sha_{\xCov} / N$ and $\ShaM = R^2 \Sha_{\xCov} / N$, with $\Sha_{\xCov}$ given in \Cref{cor:value_constants_in_thm}.
\end{enumerate}
\end{corollary}

\Cref{cor:fl_heter_wstar} can be proved reusing computation made for \Cref{cor:fl_heter_cov} and using below \Cref{prop:impact_clients_hetero}. We next aim at computing the additive noise covariance. We note $\gwkstar^i = \g_k^i(\ws)$ the local stochastic gradient evaluated at optimal point $\ws$. We have, in \Cref{model:fed},  for any $w\in\R^d$, $F_i(w) := \E(\langle{x_k^i},{w - \ws^i}\rangle - x_k^i \varepsilon_k^i)^2 $,  thus $\nabla F(w) = \frac{1}{N} \sum_\iN \xCov (w - \ws^i) $, and $\ws = \sum_{i=1}^N \ws^i / N$. The setting of \Cref{def:class_of_algo} is verified with $ \Fhess = \xCov$, and for any $w\in\R^d$, that the random field $\xi_k$ can be computed as:
\begin{align*}
    \xi_k(w - \ws) &\overset{\text{Def. } \ref{def:class_of_algo} \& \mathrm{Alg.} \ref{ex:dist_comp_LMS}}{=}  \Fhess (w - \ws) - \frac{1}{N} \sum_\iN \C^i(\g_k^i(w)), \text{ thus }     \xikstaruk \overset{\text{Def. } \ref{def:add_mult_noise}}{=} - \frac{1}{N} \sum_\iN \C^i(\gwkstar^i),
\end{align*}
with $\gwkstar^i = (x_k^i \otimes x_k^i) (\ws - \ws^i) + x_k^i \varepsilon_k^i$. We  thus have, for  any $k\in \N$:
\begin{eqnarray}
 \aniac& = &  \E \left[(\xikstaruk )^{\kcarre} \right] \overset{\nabla F(\ws)=0}{=} \E\left[\left(\frac{1}{N} \sum_\iN \C^i(\gwkstar^i) - \nabla F_i(\ws) \right)^{\kcarre}\right] \nonumber \\
    &\overset{\forall i\neq j,~\C_k^i \perp \C_k^j}{\underset{\E \C^i_k(\gwkstar^i) = \nabla F_i(\ws)}{=}} &\frac{1}{N^2} \sum_\iN \E \left[\left(\C^i_k(\gwkstar^i) -  \nabla F_i(\ws)\right)^{\kcarre}\right] \nonumber \\
    &=& \frac{1}{N^2} \sum_\iN \left(\E [\C^i_k(\gwkstar^i)^{\kcarre}] - \nabla F_i(\ws)^{\kcarre}\right) \, \nonumber \\ 
    &=& \frac{\sigma^2}{N^2} \sum_\iN \compCov[i][][\Theta] - \ffrac{1}{N^2} H \sum_\iN (\ws -\ws^i)^\kcarre H  \preccurlyeq    \frac{\sigma^2}{N}\overline{\mathfrak{C}((\C^i, p_{\Theta_i})_\iN)} \,, \nonumber 
\end{eqnarray}
where $p_{\Theta_i}$ is the distribution of $\gwkstar^i$ (for any $k$). In the last inequality, we simply discarded the non-positive  term $- H \sum_\iN (\ws -\ws^i)^\kcarre H$. For linear compressors,  by \Cref{prop:covariance_formula}, $\aniac$ is a linear function of $\frac{1}{N}\sum_\iN \Theta_i$---the averaged second-order moment of the local gradients $(\gwkstar^i)_\iN$. In order to bound this quantity, following~\cite{dieuleveut_harder_2017}, we make the following assumption.

\begin{assumption} \label{ass:kurtosis}
The \textit{kurtosis}
for the projection of the covariates $x^i_1$ (or equivalently $x^i_k$ for any $k$)  is bounded on any direction $z \in \R^d$, i.e., there exists $\kappa>0$, such that:
$$
\forall i \in \OneToN, \ \forall z \in \R^d, \quad \mathbb{E}\left[\left\langle z, x^i_1\right\rangle^4 \right]\leq \kappa\langle z, H z\rangle^2
$$
\end{assumption}
For instance, it is verified for Gaussian vectors with $\kappa=3$.
By Cauchy-Schwarz inequality, it implies that $\mathbb{E}[\left\langle z, x^i_1\right\rangle^2 \ (x^i_1)^\kcarre ] \preccurlyeq \kappa \langle z, H z\rangle H$ for all $z \in \R^d$. We obtain the following proposition.

\begin{proposition}[Impact of client-heterogeneity]
\label[proposition]{prop:impact_clients_hetero}
Let $W_*$ be a  random variable uniformly distributed over~{$\{\ws^i, i\in \OneToN \}$}, thus such that, $\Cov{W_*}= \frac{1}{N} \sum_\iN(\ws - \ws^i)^\kcarre$, then:
$$ \frac{1}{N} \sum_\iN  \Theta_i \preccurlyeq  ( \kappa \Tr{H \Cov{ W_*}}  + \sigma^2 )\  H\,. $$ 
\end{proposition}

\begin{proof}
We have:
\begin{align*}
    \Theta_i &= \E[((x_k^i \otimes x_k^i) (\ws - \ws^i) + x_k^i \varepsilon_k^i)^\kcarre] \overset{(\varepsilon_k^i)\perp (x_k^i)}{=} \E[(x_k^i \otimes x_k^i) (\ws - \ws^i)^\kcarre  (x_k^i \otimes x_k^i) ]+ \sigma^2 H \nonumber
   \\ 
   &\overset{\text{Ass. }\ref{ass:kurtosis}}{\preccurlyeq} \kappa \PdtScl{\ws - \ws^i}{H(\ws - \ws^i)} H + \sigma^2 H = \kappa \Tr{H (\ws - \ws^i)^\kcarre} H + \sigma^2 H \,.
\end{align*}
\end{proof}
In words, we have the following two main observations.
\begin{remark}[\textbf{Structured noise before compression}]
\label[remark]{rem:structnoiseFL}
Before compression is possibly applied, the noise remains \textit{structured}, i.e., with covariance proportional to $H$, in the case of concept-shift. As a consequence, the rate for un-compressed \Cref{eq:LSA} will remain independent of the smallest eigenvalue of $H$. This remark extends to the case where $\Cpp$ is applied.
\end{remark}

\begin{remark}[\textbf{Heterogeneous vs homogeneous case.}] Compared to the homogeneous case, in which $\Theta_i=\sigma^2 H_i$ and $\aniac = \frac{\sigma^2}{N}\overline{\mathfrak{C}((\C^i, p_{H_i})_\iN)}$, the averaged second-order moment increases from $\sigma^2 \xCov$ to $(\kappa \Tr{H \Cov{ W_*}}  + \sigma^2 )\xCov$, showing the impact of the dispersion of the  optimal points $(\ws^i)_\iN$. This corresponds to the typical variance increase in the compressed heterogeneous SGD case~\citep{mishchenko_distributed_2019,philippenko_artemis_2020}.
\end{remark}

Concept-shift thus hinders the limit convergence rate.  To limit this effect, a solution is to introduce a control-variate term $(h^i_k)_{k\in \N^*, i\in \OneToN}$, that is subtracted to the gradient before compression and asymptotically approximate $\nabla F_i(\ws)$ for any $i\in \OneToN$ \citep[see][]{mishchenko_distributed_2019}. We explore this direction in \Cref{app:subsec:fl_wstar}.

\subsection{Numerical experiments}

We support the theoretical results from \Cref{subsec:fl_sigma,subsec:fl_wstar} by performing experiments in the FL framework that extend the ones from  \Cref{sec:application_compressed_LSR}.

On figures \Cref{fig:distributed_sgd}, we present the results of the excess loss of the Polyak-Ruppert iterate $F(\overline{w}_k) - F(w_*)$ versus the number of iterations in log/log scale. The experiments were run 5 times, each time with different datasets (dispersion is shown by shaded area). 

\textbf{Settings.} (a) \textit{Synthetic dataset generation:} The dataset is generated using \Cref{model:fed} with $N=10$, $K=10^6$ on each client, $\sigma^2 / {N} = 1$. For any clients $i$ in $\OneToN$, the covariance matrix is $\xCov_i = Q_i D_i Q_i^T$, where $Q_i$ is an orthogonal matrix.
For heterogeneous clients, the dataset generation is as follows.
 \textit{Covariate shift:} The rotation matrix $Q_i$ is sampled independently for each client and the diagonal matrix $D_i$  is $\Diag{(1/j^{\beta_i})_{j=1}^d}$ where $\beta_i \sim \mathrm{Unif}(\{3,4,5,6\})$. \textit{Concept-shift:} The optimal models of the clients $i \in \OneToN$ were drawn from a zero-centered normal distribution with a variance of $100\Id_d$, that is, $\ws^i \sim \mathcal{N}(0,  100 \Id_d)$. We also take for all client $i$ in $\OneToN$, $\xCov_i = Q D Q^T$, with $D = \Diag{(1/j})_{j=1}^d$.
(b)~\textit{Real-dataset and {covariate-shift}:} To simulate non-i.i.d.~clients, we split the dataset in heterogeneous
groups (with equal number of points) using a $K$-nearest neighbors clustering on the TSNE
representations \citep[defined by][]{maaten_visualizing_2008}. Thus,
the marginal feature distribution 
significantly varies between clients, providing a covariate-shift, while keeping the same distribution for the output conditional to the features on all clients.
(c)~\textit{\Cref{ex:dist_comp_LMS}:}
We take a constant step-size $\gamma = 1/ (2 (\omgC + 1) R^2)$ with $R^2 = \Tr{\xCov}$ and  $w_0 = 0$ as initial point. We set the batch-size $b=1$ for synthetic datasets and $b=16$ for real datasets, the compressor variance is $\omgC = 10$.
(d) \textit{\Cref{ex:dist_comp_LMS}} vs \textit{\Cref{ex:dist_comp_LMS_artemis}}: We take a bigger constant step-size $\gamma = (2 R^2)^{-1}$ in order to emphasize the difference between the case w./w.o.~control variate, we set $w_0 = 0$ as initial point and the compressor variance is $\omgC = 10$. We set the batch-size $b=32$ for \Cref{fig:dist_wstar_artemis} and $b=K$ for \Cref{fig:dist_wstar_artemis_full}. 

\begin{figure}
\label{fig:fl}
\resizebox{\linewidth}{!}{
\begin{tabular}{ccc}
\toprule
\multicolumn{2}{c}{Covariate-shift} & Concept-shift \\
\cmidrule(lr){1-2}\cmidrule(l){3-3}
 Synthetic dataset & Real datasets & Synthetic dataset \\
 \cmidrule(l){1-1}\cmidrule(l){2-2}\cmidrule(l){3-3}
 
\begin{subfigure}{0.32\linewidth}
\includegraphics[width=\linewidth]{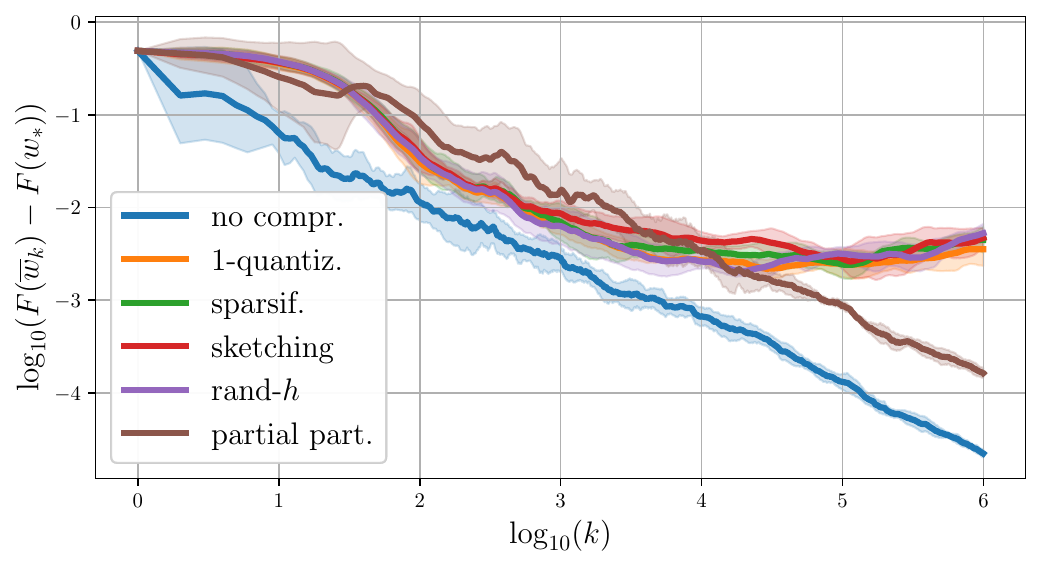}
\caption{No shift: $\forall i \in \OneToN, \xCov_i = \xCov$ \label{fig:dist_homog}}
\end{subfigure}
&\begin{subfigure}{0.32\linewidth}
\includegraphics[width=\linewidth]{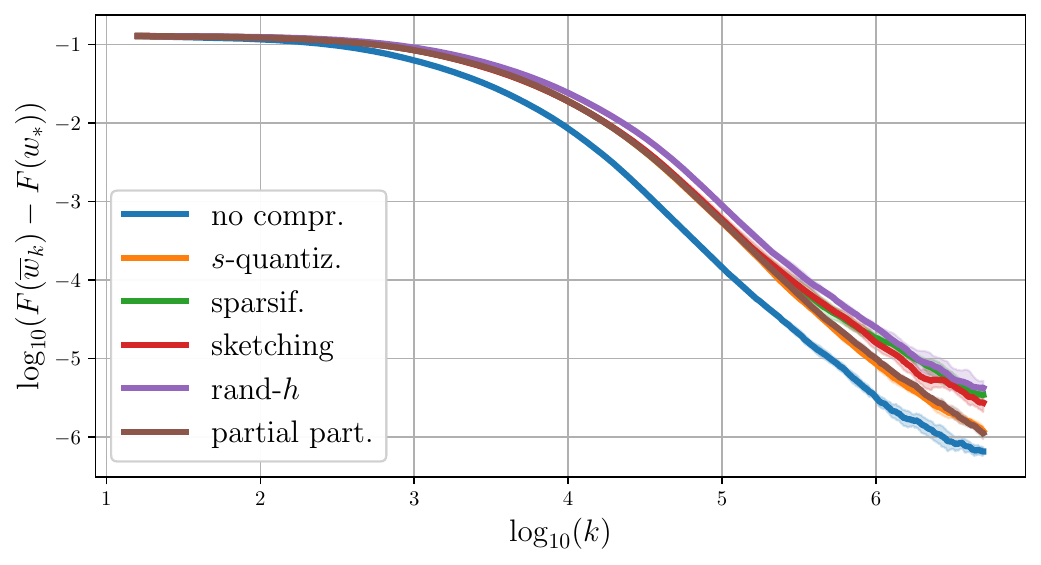}
\caption{\texttt{quantum}\label{fig:dist_quantum}}
\end{subfigure}
&
\begin{subfigure}{0.32\linewidth}
\includegraphics[width=\linewidth]{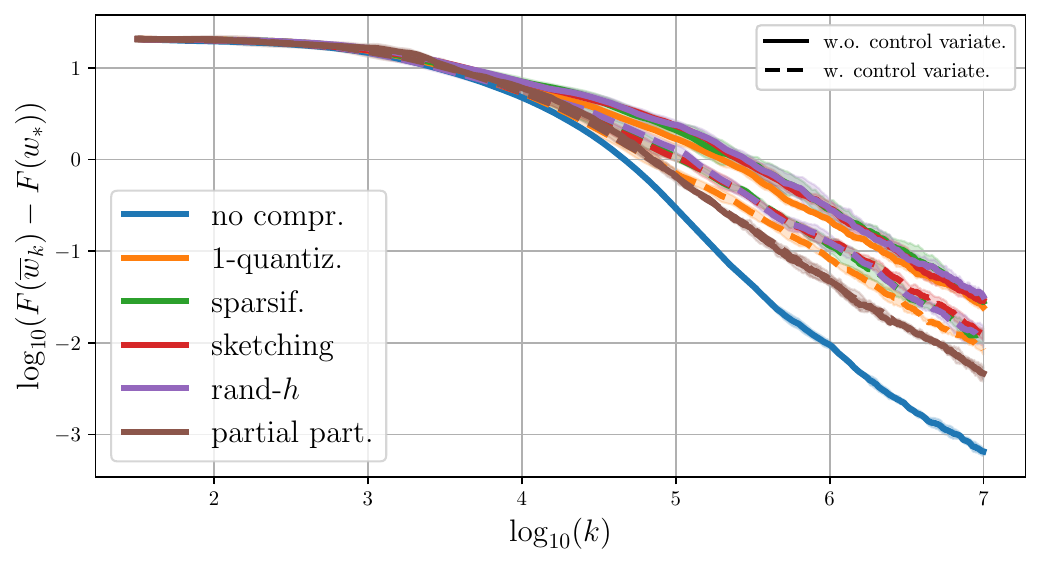}
\caption{Batch stochastic gradient 
\label{fig:dist_wstar_artemis}}
\end{subfigure}
\\
\begin{subfigure}{0.32\linewidth}
\includegraphics[width=\linewidth]{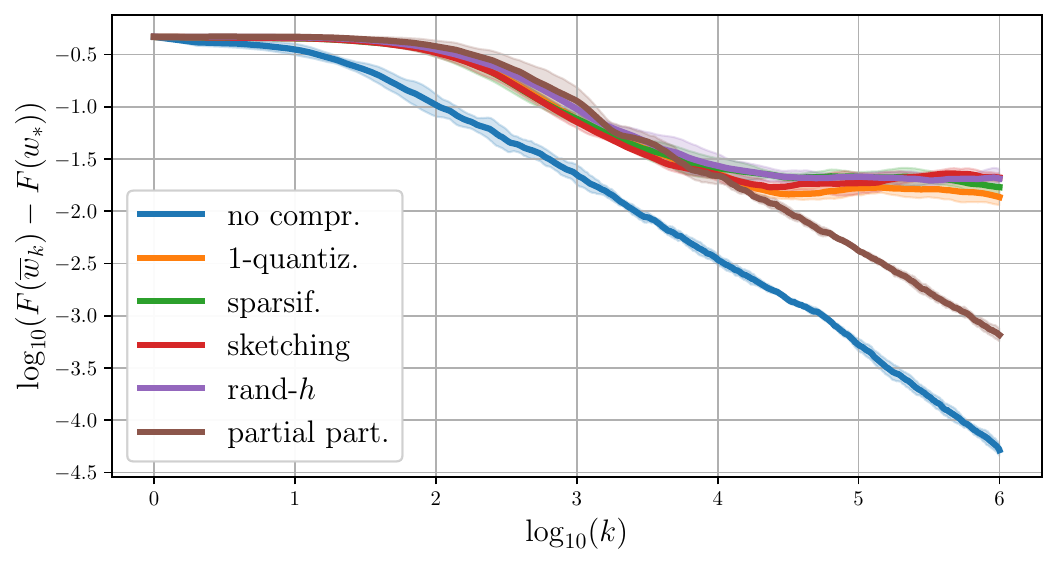}
\caption{$\forall i,j \in \OneToN, \xCov_i {\neq} \xCov_j$ \label{fig:dist_sigma}}
\end{subfigure}
&
\begin{subfigure}{0.32\linewidth}
\includegraphics[width=\linewidth]{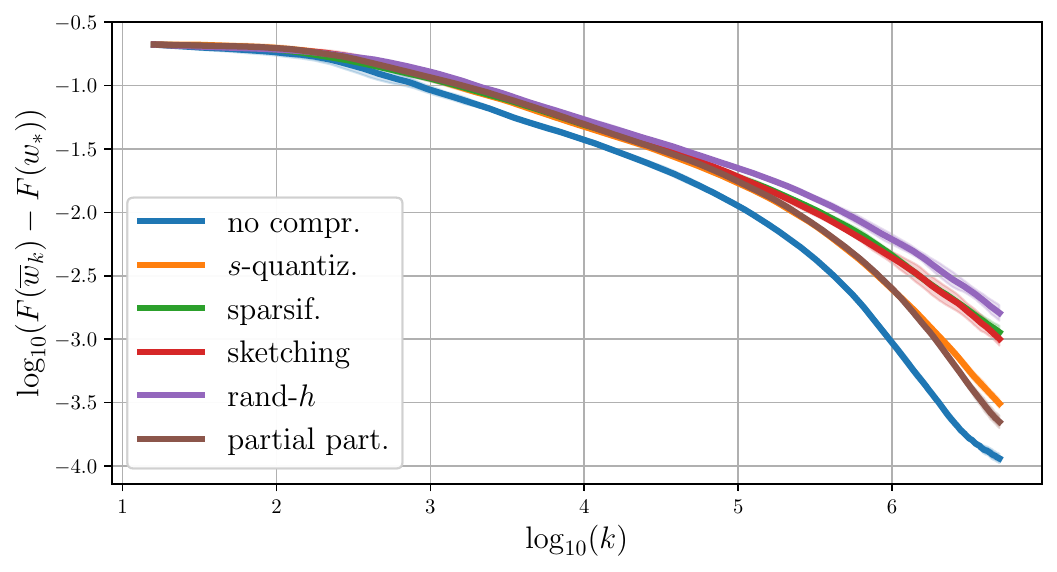}
\caption{\texttt{cifar-10} \label{fig:dist_cifar10}}
\end{subfigure}
&
\begin{subfigure}{0.32\linewidth}
\includegraphics[width=\linewidth]{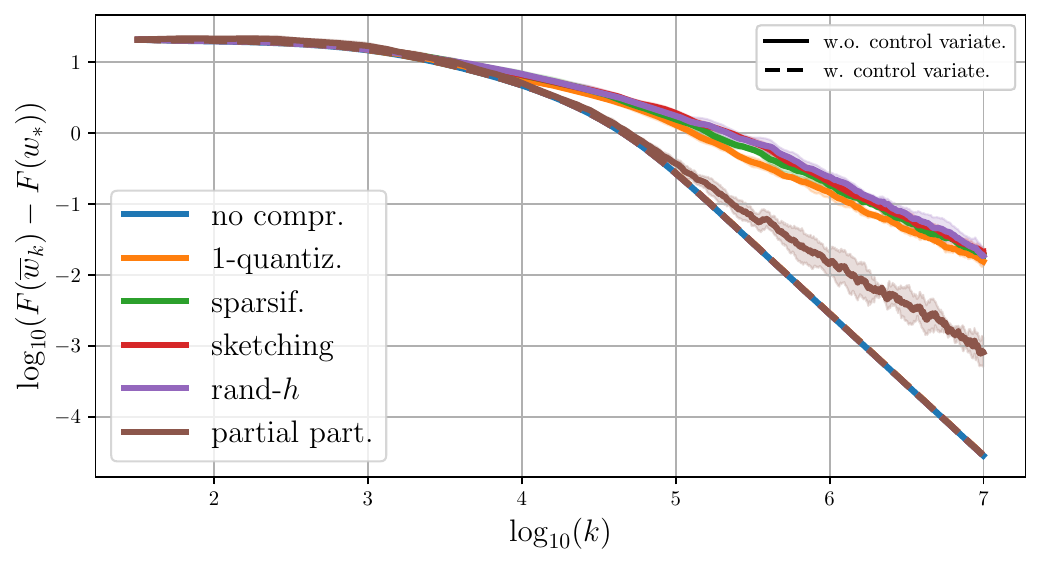}
\caption{True gradient $\g_k^i= \nabla F_i$
\label{fig:dist_wstar_artemis_full}}
\end{subfigure}
\\
\bottomrule
\end{tabular}}
\caption{Logarithm excess loss of the Polyak-Ruppert iterate iterations for $N=10$ clients.   }
\label{fig:distributed_sgd}
\end{figure}

\textbf{Interpretation---homogeneous case and covariate-shift case} (\Cref{fig:dist_homog,fig:dist_sigma,fig:dist_quantum,fig:dist_cifar10}). These experiments extend those presented in \Cref{subsec:numerical_experiments} in the case of a single client. The observations made in 
the centralized case (\Cref{fig:sgd}), especially on the impact of the compressor choice on the convergence and the ordering between limit convergence rates remain valid. This illustrates  \Cref{cor:fl_heter_cov} and  \Cref{rem:lin_FL_compcov}:~\Cref{thm:bm2013_with_linear_operator_compression,thm:bm2013_with_nonlinear_operator_compression} hold in the case of homogeneous client or in the case of heterogeneous covariance and the only compressor that ensures that the noise is structured is client sampling (partial participation). On the real datasets, quantization is also competitive.

\textbf{Interpretation---concept-shift case} (\Cref{fig:dist_wstar_artemis,fig:dist_wstar_artemis_full}). These experiments extend those presented on \Cref{fig:sgd_ortho_p1} (slow eigenvalues' decay with $\mu = 10^{-2}$) to the scenario of concept shift. First, we observe on \Cref{fig:dist_wstar_artemis} 
that for all compressors the convergence rate remains in $O(1/K)$, (though vanilla SGD converges faster during the first iterations). Second, we observe that control-variates improve convergence for compressors inducing un-structured noise ; this is predicted by theory, see \Cref{thm:asymptotic_normality_artemis}.
Third, on \Cref{fig:dist_wstar_artemis_full}, at each iteration $k\in\OneToK$, we use deterministic gradients $g_k^i  =\nabla F_i $ which leads to having a.s. $\xi_k^\mathrm{add} = 0$, and in the absence of compression, we obtain a  $O(1/K^2)$ convergence rate for $\overline w_K$ which corresponds in \Cref{thm:bm2013_with_nonlinear_operator_compression} to the case where the dependency on the initial condition is dominated by $\frac{\|\Fhess^{-1/2}\eta_0 \|^2}{\gamma^2 K^2}$. Overall, these experiments illustrate and support our theoretical insights.

\section{Conclusion and open directions}
\label{sec:conclusion}
\textbf{Conclusion.} In short, we investigate the impact of the choice of compression scheme on the convergence of the Polyak-Ruppert averaged iterate. By analysing the  case of  compressed least-squares regression, we shed light on the interplay between the Hessian of the optimization problem $\Fhess$, the features' distribution,  the additive noise's covariance $\aniac$, and the compression scheme. This shows fundamental differences between compression that deemed equivalent under the classical worst-case-variance assumption. We extend our analysis to the case of heterogeneous federated learning, a setting in which compression is widely used and its impact not fully understood.

More precisely, first, the analysis of the generic stochastic approximation algorithm \eqref{eq:LSA} provides (1) the fact that projection based compressions achieve a faster convergence rate than quantization based, and that yet,  their asymptotic rate is similar; (2) the analysis of quantization-based compression  requires introducing a new Hölder-type regularity assumption for the analysis of the stochastic approximation  scheme, and showing that such an assumption is satisfied for quantization. 

Second, the computation of the additive noise's covariance $\aniac$ reveals the impact of the compression scheme and the data distribution on the asymptotically dominant term. We obtain that (1) partial participation (i.e., client sampling in the federated case) is the only method that systematically ensures a convergence without a dependency on the strong-convexity constant; (2) other compressors may all induce an un-structured noise, with covariance scaling with $\Id$ or $\sqrt{H}$, that strongly hinders convergence by accumulating noise on low curvature directions; (3) the relative performance or various schemes changes depending on the pre-processing applied to the data, making quantization the best method (apart from PP) when standardization is applied, but one of the worst (with random Gaussian projection) when the features are independent and the eigenvalues of the covariance decay rapidly (4) in that particular last setting, all projection based methods (but Gaussian projection) behave similarly.

Third, we discuss how these results  apply to the federated case, that  corresponds to the initial motivation. We show that we encompass two particular heterogeneity situations and how  our analysis applies.
Overall, these results are a step towards a better understanding of the impact of a widely  used tool.

\textbf{Open directions.} This analysis can be extended to include various aspects that are beyond the scope of this work. First, one natural improvement for application in FL would be to consider also the scenario where each client runs several \textit{local iterations} \citep{mcmahan_communication-efficient_2017,karimireddy_scaffold_2020} before sending their updates, reducing further the cost of communication. 
Similar approach can be used, although the additive noise field would be more complicated, which potentially implies a different additive noise's covariance. 
Second, as mentioned in \Cref{subsec:ccl}, our analysis could also be extended to the case of stochastic approximation with ridge regularization \citep[e.g., following ][]{dieuleveut_harder_2017} which in practice is helpful to mitigate the impact of the lack of strong convexity. 
Third, an obvious direction is to extend  beyond quadratic functions and considering other objective functions, such as logistic regression or even shallow neural networks. Several results in the literature can be leveraged to tackle non quadratic but self-concordant losses~\citet{bach_self-concordant_2010,gadat_optimal_2023}.
Fourth, our analysis still only relies on second moments (variance and covariance) of the stochastic field. One major drawback of partial participation is to induce a significant increase on  higher order moments. Incorporating higher order bounds may also bring novel insights to the use of compression in FL.
Finally, all our analysis is made in finite dimension and our asymptotic focuses  on $K\to \infty$: further works should analyze the case of infinite dimension: within the reproducing kernel Hilbert space \citep{dieuleveut_nonparametric_2016} framework or within the overparametrized setting \citep{belkin_reconciling_2019}. 
\section{Acknowledgements}
{
We would like to thank Richard Vidal, Laeticia Kameni from Accenture Labs (Sophia Antipolis, France), and Eric Moulines from École Polytechnique for insightful discussions. We also thank Baptiste Goujaud for his great help in computing the covariance of sketching, and JM for her help on \Cref{fig:chart_summary_paper}. This research was supported by the \emph{SCAI: Statistics and Computation for AI} ANR Chair of research and teaching in artificial intelligence, by \textit{Hi!Paris}, and by \emph{Accenture Labs} (Sophia Antipolis, France).

Constantin Philippenko thanks his sister Ania, whose covariance has greatly helped to write the theorems.

} 
\bibliography{main}

\newpage	
\onecolumn
\appendix

\begin{center}
	{\Large{\bf Supplementary material}}
	\end{center}
	
\renewcommand{\theequation}{S\arabic{equation}}
	\renewcommand{\thefigure}{S\arabic{figure}}
	\renewcommand{\thetheorem}{S\arabic{theorem}}
	\renewcommand{\thelemma}{S\arabic{lemma}}
	\renewcommand{\theproposition}{S\arabic{proposition}}
	\renewcommand{\thecorollary}{S\arabic{corollary}}
 \renewcommand{\thedefinition}{S\arabic{definition}}
 \renewcommand{\theproperty}{S\arabic{property}}
 \renewcommand{\theremark}{S\arabic{remark}}
	\renewcommand{\thetable}{S\arabic{table}}
	
	In this appendix, we provide additional information to supplement our work.
	In \Cref{app:sec:technical_results}, we begin by detailing technical results, by introducing an auxiliary lemma and by proving \Cref{prop:asymptotic_normality} which gives a CLT for \eqref{eq:LSA}. 
	Secondly, in respectively \Cref{app:sec:nonlinear_bm} and \Cref{app:sec:linear_bm}, we give the proof of \Cref{thm:bm2013_with_nonlinear_operator_compression,thm:bm2013_with_linear_operator_compression}. 
Thirdly, in \Cref{app:subsec:validity_asu_random_fields}, we verify that the setting of single-client compressed LSR fulfills the setting presented in \Cref{sec:theoretical_analysis}. In \Cref{app:sec:covariance} we prove that \Cref{lem:compressor} hold and compute the compressors' covariance to establish \Cref{prop:covariance_formula,prop:covariance_formula_diagonal_case}. Finally, in \Cref{app:sec:bm_fl}, we provide demonstrations for the federated learning case, including verifying assumptions (covariate-shift scenario) on random fields in \Cref{app:subsec:validity_asu_random_fields_fl_sigma}, and proving a Central Limit \Cref{thm:asymptotic_normality_artemis} in \Cref{app:subsec:fl_wstar} for the concept-shift scenario. 

\paragraph{Additional notations.}
We use the Frobenius norm $\| A \|^2 := \Tr{A^\top A}$, which is the same notation as the vector Euclidean norm (no ambiguity in general), $J_r$ to denote the $d\times d$ diagonal matrix whose $r$ first diagonal elements are equal to one and all the other matrix's coefficients equal to zero,~$\mathcal{S}^{++}_d(\R)$ the cone of positive definite symmetric matrices, and $\mathcal{L}^p(\Omega, \mathcal{A}, \mathbb{P})$ the set of
random vectors defined on the probability space $(\Omega, \mathcal{A}, \mathbb{P})$ such that $\E[\|X\|^p] < \infty$. We define also the operator norm $\vertiii{A} := \sqrt{\max \eig{A^\top A}}$.

	\hypersetup{linkcolor = black}
	\setlength\cftparskip{2pt}
	\setlength\cftbeforesecskip{2pt}
	\setlength\cftaftertoctitleskip{3pt}
	\addtocontents{toc}{\protect\setcounter{tocdepth}{2}}
	\setcounter{tocdepth}{1}
	\tableofcontents
	
\addtocontents{toc}{
    \protect\thispagestyle{empty}} 
    \thispagestyle{empty} 
    
    \hypersetup{linkcolor=blue}

\setlength{\parindent}{0pt}

\section{Technical results}
\label{app:sec:technical_results}

\subsection{Settings of experiments}
\label{app:subsec:settings_xp}

In \Cref{tab:settings_exp,tab:settings_exp_real_dataset}, we summarize the settings of experiments presented in \Cref{subsec:numerical_experiments}. 

\begin{table}[h!]
\caption{Settings of experiments for a single client ($N=1$) on synthetic data (\Cref{fig:sgd_diag,fig:sgd_ortho}).}
\label{tab:settings_exp}
\centering
\begin{tabular}{lccccccccc}
\hline 
\addlinespace[1ex]
Parameter & $K$ & $d$ & $ \eig{\xCov}_{i}$ & $\ws$ & $\sigma^2$ & $\omgC$ & $\gamma^{-1}$ & $w_0$ &\#\textit{runs} \\
\addlinespace[1ex]
Values & $10^7$ & $100$ & $1 / i^4$ & $(1)_{i=1}^d$ & $1$ & $10$ & $2(\omgC + 1) R^2$ & 0 & 5 \\
\addlinespace[1ex]
\hline 
\end{tabular}
\end{table}

\begin{table}[h!]
\caption{Settings of experiments for a single client ($N=1$) on real data (\Cref{fig:sgd_cifar10,fig:sgd_quantum}).}
\label{tab:settings_exp_real_dataset}
\centering
\begin{tabular}{lccccccccc}
\hline 
\addlinespace[1ex]
Dataset & $d$ & standardization & $b$ & $\omgC$ & $\gamma^{-1}$ & $w_0$ &\#\textit{runs} & reference\\
\addlinespace[1ex]
\texttt{quantum} & $65$ & \multirow{2}{*}{\cmark} & \multirow{2}{*}{$16$} & \multirow{2}{*}{$1$} & \multirow{2}{*}{$2(\omgC + 1) R^2$} & \multirow{2}{*}{0} & \multirow{2}{*}{5} & \citepalias{caruana_kdd-cup_2004}\\
\texttt{cifar-10} & $256$ & &  &  & & & & \citepalias{krizhevsky_learning_2009}\\
\addlinespace[1ex]
\hline 
\end{tabular}
\end{table}

\subsection{Useful identities and inequalities}
\label{sect:useful_identies}

In this Subsection, we recall some classical inequalities and results.

\begin{inequality}
\label{app:lem:two_inequalities}
Let $N \in \N$ and $d \in \N$. For any sequence of vector $(a_i)_{i=1}^N \in \mathbb{R}^d$, we have the following inequalities:
\[
\lnrm \sum_{i=1}^N a_i \rnrm^2 \leq \left( \sum_{i=1}^N \lnrm a_i \rnrm \right)^2 \leq N \sum_{i=1}^N \lnrm a_i \rnrm^2 \,.
\]
\end{inequality}
The first part of the inequality corresponds to the triangular inequality, while the second part is Cauchy's inequality.

\begin{inequality}
\label{app:lem:ineq_subordinate_norm}
Let $x$ in $\R^d$ and $A$ in $\mathcal{M}_{d,d}(\R)$, then we have $\| A x \| \leq \vertiii{A} \| x\| $.
\end{inequality}

Below, we recall Minkowski's and Jensen's  inequalities. Additionally, we recall the Cauchy-Schwarz inequality for conditional expectations.

Let a probability space $(\Omega, \mathcal{A},\mathbb{P})$ with $\Omega$ a sample space, $\mathcal{A}$ a $\sigma$-algebra, and $\mathbb{P}$ a probability measure. 

\paragraph{Minkowski's inequality.}
Let $p>1$ and suppose that $X, Y$ are two random variables in $\mathcal{L}^p(\Omega, \mathcal{A},\mathbb{P})$ (i.e. their $p^{\mathrm{th}}$ moment is bounded), we have the following triangular inequality:
\begin{align}
\label{app:eq:minkowski}
\E[\| X + Y \|^p]^{1/p} \leq \E[\|X \|^p]^{1/p} + \E[\| Y \|^p]^{1/p}\,.
\end{align}

\paragraph{Jensen's inequality.}
Suppose that $X: \Omega \longrightarrow \R^d$ is a random variable, then for any convex function $f: \R^d \longrightarrow \R$ we have:
\begin{align}
\label{app:basic_ineq:jensen}
f\bigpar{\E(X)} \leq \E f(X) \,.
\end{align}

\paragraph{Cauchy-Schwarz's inequality for conditional expectations.}
Suppose that $X, Y$ are two random variables in $\mathcal{L}^2(\Omega, \mathcal{A},\mathbb{P})$ (i.e. their second moment is bounded), then for any $\sigma$-algebra $\mathcal{F} \subset  \mathcal{A}$ we have a.s.:
\begin{align}
\label{app:basic_ineq:cauchy_schwarz_cond}
\Expec{XY}{\mathcal{F}}^2 \leq \Expec{X^2}{\mathcal{F}} \Expec{Y^2}{\mathcal{F}}\,.
\end{align}
 
\paragraph{Convergence in $L^p$-norm.} Suppose that $(X_n)_{n \in \N}$ is a sequence of random variables in $\mathcal{L}^p(\Omega, \mathcal{A},\mathbb{P})$, and that $X$ is a random variable in $\mathcal{L}^p(\Omega, \mathcal{A},\mathbb{P})$. We say that $(X_n)_{n \in \N}$ converges in $L^p$-norm towards $X$ if $\E( \|X_n - X\|^p ) \xrightarrow[n \to +\infty]{} 0 $, it is denoted by: $
X_n\xrightarrow[n \to +\infty]{L^p} X \,.$

In \Cref{subsec:impact_cov_on_additive_noise}, we use ellipses to visual quadratic functions, therefore we provide in \Cref{def:covariance_ellipse} the mathematical definition.
\begin{definition}[Representing positive matrices through ellipsoids]
\label[definition]{def:covariance_ellipse}
Any symmetric positive definite matrix $M$ in $\mathcal{S}_d^{++}(\R)$ defines an ellipsoid
$\mathcal{E}_M =\{ x \in \R^d, x^\top M^{-1} x = 1 \}$
centered around zero. The eigenvectors of $M$ are the principal axes of the ellipsoid, and the squared root of the eigenvalues are the half-lengths of these axes. 
The ellipse corresponds to the sphere of radius 1 associated with the norm $N_{M^{-1}} = \sqrt{x^\top M^{-1} x}$.
\end{definition}

\subsection{An auxiliary inequality}

In this Section, we provide an auxiliary lemma that is specific to the framework considered in \Cref{sec:theoretical_analysis}. It will be used in the proof of \Cref{thm:bm2013_with_nonlinear_operator_compression} and corresponds to an adaptation of Lemma 1 from \cite{bach_non-strongly-convex_2013}. 

\fbox{
\begin{minipage}{0.97\textwidth}
\begin{lemma}[Auxiliary inequality on $ \sum_{k=1}^K \fullexpec{\| \Fhess^{1/2} \eta_k \|^2}/K$]
\label[lemma]{app:lem:from_bm2013}
Under \Cref{asu:main:bound_add_noise,asu:main:bound_mult_noise}, for any $K$ in $\N^*$ and any step-size $\gamma \in \R^+$ s.t. $\gamma ( \Ftrace^2 + 2 \boundMult) \leq 1$, the sequence $(w_k)_{k\in\N^*}$ produced by a setting such as in \Cref{def:class_of_algo}, verifies the following bound:
\begin{align*}
\frac{1}{K} \sum_{k=0}^{K-1}   \fullexpec{ \sqrdnrm{\Fhess^{1/2} (w_k - \ws)}} \leq \ffrac{\SqrdNrm{\eta_0}}{2\gamma K (1 - \gamma ( \Ftrace^2 + 2\boundMult))} + \ffrac{5\boundAdd \gamma }{1 -\gamma ( \Ftrace^2 + 2\boundMult)} \,.
\end{align*}
\end{lemma}
\end{minipage}
}

\begin{proof}
Let $k$ in $\N^*$, we start writing that by \Cref{def:class_of_algo}, we have $w_k = \wkm - \gamma \nabla F(\wkm) + \gamma \xik$. Thus taking the squared norm and developing it, gives:
\begin{align}
\label{app:eq:first_line_aux_lemma}
    \SqrdNrm{\eta_k} &= \SqrdNrm{\etakm} - 2 \gamma \PdtScl{\etakm}{\nabla F(\wkm) - \xik} + \gamma^2 \SqrdNrm{\nabla F(\wkm) - \xik}\,.
\end{align}

We need to bound the last term. By \Cref{def:add_mult_noise}, we have $\xik = \xikmultFN[\etakm] + \xikstaru$, hence using \Cref{app:lem:two_inequalities}, we have:
\begin{align*}
   \SqrdNrm{\nabla F(\wkm) - \xik} &\leq 2\sqrdnrm{\nabla F(\wkm) - \xikmultFN[\etakm]} +
   2\sqrdnrm{\xikstaru} \,,
\end{align*}
taking expectation w.r.t the $\sigma$-algebra $\Flast_{k-1}$, developping $\SqrdNrm{\nabla F(\wkm) - \xikmultFN[\etakm]}$ and because $\Expec{\xikmultFN[\etakm]}{\Flast_{k-1}} = 0$ (the random fields $(\xi_k)_{k\in \N^*}$ are zero-centered, see \Cref{def:class_of_algo}), we have:
\begin{align*}
   &\Expec{\SqrdNrm{\nabla F(\wkm) - \xik}}{\Flast_{k-1}} \\
   &\qquad\leq 2\Expec{\sqrdnrm{\nabla F(\wkm)}}{\Flast_{k-1}} + 2\Expec{\sqrdnrm{\xikmultFN[\etakm]}}{\Flast_{k-1}} +2\Expec{\sqrdnrm{\xikstaru}}{\Flast_{k-1}} \,.
\end{align*}

Now, we use \Cref{def:class_of_algo} and \Cref{asu:main:bound_add_noise,asu:main:bound_mult_noise}: it leads to:
\begin{align*}
   &\Expec{\SqrdNrm{\nabla F(\wkm) - \xik}}{\Flast_{k-1}} \\
   &\qqquad\leq 2  \Ftrace^2 \sqrdnrm{\Fhess^{1/2} \etakm} 
   + 4 \boundMult \sqrdnrm{\Fhess^{1/2} \etakm} + 8 \boundAdd  + 2 \boundAdd \\
   &\qqquad\leq 2 ( \Ftrace^2 + 2\boundMult) \sqrdnrm{\Fhess^{1/2} \etakm} + 10\boundAdd \,.
\end{align*}

Because the sequence of random field $(\xi_k)_{k\in \N^*}$ is zero-centered (\Cref{def:class_of_algo}), we have:
\[
\Expec{\PdtScl{\etakm}{\nabla F(\wkm) - \xik}}{\Flast_{k-1}} = \PdtScl{\etakm}{\Fhess\etakm} = \sqrdnrm{\Fhess^{1/2} \etakm} \,,
\]
hence back to \Cref{app:eq:first_line_aux_lemma}, we obtain:
\begin{align}
\label{app:eq:to_bound_eta_K}
\begin{split}
    \Expec{\SqrdNrm{\eta_k}}{\Flast_{k-1}} &\leq \SqrdNrm{\etakm} - 2 \gamma (1 - \gamma ( \Ftrace^2 + 2\boundMult)) \sqrdnrm{\Fhess^{1/2} \etakm} +  10 \boundAdd \gamma^2  \,.
\end{split}
\end{align}

It follows that if $\gamma ( \Ftrace^2 + 2 \boundMult) \leq 1$, summing the previous bound and taking full expectation gives:
\begin{align*}
\frac{1}{K} \sum_{k=1}^{K}   \FullExpec{\sqrdnrm{\Fhess^{1/2} \etakm}} \leq \ffrac{\SqrdNrm{\eta_0} - \E\SqrdNrm{\eta_K}}{2\gamma K (1 - \gamma ( \Ftrace^2 + 2\boundMult))} + \ffrac{5\boundAdd \gamma }{1 - \gamma ( \Ftrace^2 + 2\boundMult)} \,,
\end{align*}
which allows concluding.
\end{proof}

\subsection{Asymptotic results: central limit theorem for \texorpdfstring{\eqref{eq:LSA}}{LSA}}
\label{app:subsec:clt}

The demonstration of \Cref{prop:asymptotic_normality} uses the following theorem from \citet{polyak_acceleration_1992} guaranteeing the asymptotic normality of the Polyak-Ruppert iterate.

\begin{theorem}\textbf{From \citet[see Theorem 1]{polyak_acceleration_1992}.}
\label{app:thm:polyak_juditsky_92}

For $k$ in $\N^*$, we denote $\eta_k = w_k - \ws$ and we define $w_k = w_{k-1} - \gamma_k \nabla F(w_{k-1}) + \gamma_k \xi(\eta_{k-1})$.
If we assume that:
\begin{itemize}
    \item $\gamma_{k} \xrightarrow[k \rightarrow+\infty]{} 0$ and $\gamma_{k}^{-1}(\gamma_{k}-\gamma_{k+1})=\underset{k \to +\infty}{o}(\gamma_{k})$,
    \item $F$ is strongly convex and $\left\|\nabla^{2} F \right\|_{\infty}<\infty$,
    \item the convergence in probability of the conditional covariance to a matrix $\Sigma$ holds, i.e., we have a.s.~$ \expec{\xi(\eta_{k-1}) \xi(\eta_{k-1})^\top}{\Flast_{k-1}} \xrightarrow[k \rightarrow+\infty]{\mathbb{P}} \Sigma \,.$
\end{itemize}
Then for any $K$ in $\N^*$, we have the asymptotic normality of $(\sqrt K \eta_{K-1})_{K\in\N^*}$:
$$
\sqrt{K} \etabarkm \xrightarrow[K \rightarrow+\infty]{\mathcal{L}} \mathcal{N}(0,  \Sigma^{*}) \text { with } \Sigma^{\star}=\left\{\nabla^{2} F\left(\ws\right)\right\}^{-1} \Sigma \left\{\nabla^{2} F\left(\ws\right)\right\}^{-1} .
$$
\end{theorem}

Below we present our CLT that gives the asymptotic normality of $(\sqrt K \eta_{K-1})_{K\in\N^*}$ in the case of strongly-convex case and decreasing step size.

\fbox{
\begin{minipage}{0.97\textwidth}
\begin{proposition}[CLT for \eqref{eq:LSA}---strongly convex-case, deacreasing step-size]
\label[proposition]{app:prop:asymptotic_normality}
Under \Cref{asu:main:bound_add_noise,asu:main:second_moment_noise}, consider a sequence $(w_k)_{k\in\N^*}$ produced in the setting of \Cref{def:class_of_algo} using a step-size $(\gamma_k)_{k\in \N^*}$ s.t. $\gamma_k =k^{-\alpha}$, $\alpha\in (0,1)$. Then  $(\eta_K)_{K\geq0}$ converges in $L^2$-norm to 0, i.e.~$\eta_K \xrightarrow[K \to +\infty]{L^2} 0$.

Furthermore, $(\sqrt{K} \etabarkm)_{K\geq0}$ is asymptotically normal with mean zero and covariance such that:
\[
\sqrt{K}\etabarkm \xrightarrow[K \to +\infty]{\mathcal{L}} \mathcal{N}(0, \Fhess^{-1} \aniac \Fhess^{-1}).
\]
\end{proposition}
\end{minipage}
}

\begin{proof}

\textbf{First, we have that in the case of decreasing step size s.t. for any $k$ in $\N$, $\gamma_k = k^{-\alpha}$, we have: $\eta_K \xrightarrow[K \to +\infty]{L^2} 0$.} This is a classical computation for SGD with bounded variance (\Cref{asu:main:bound_add_noise,asu:main:bound_mult_noise}.).
Detailed computations are  for instance given in lectures notes of \citet[][pages 164-167 and 182]{bach_lecturesnotes_2022}, and \cite{kushner_stochastic_2003}. To apply Theorem 1 from \citet[][recalled in \Cref{app:thm:polyak_juditsky_92}]{polyak_acceleration_1992}, which gives the desired result, it suffices to prove the convergence in probability of the covariance of the noise $\xik$ towards $\aniac$, as $k\to \infty$.

\textbf{In the following, we show that  $\underset{k \to +\infty}{\lim} \Expec{ \xik \xik^\top}{\Flast_{k-1}} \stackrel{\mathbb{P}}{=}\aniac  $.}
We start writing:
\begin{align*}
    \xik \xik^\top &= (\xikstaruk - \xikmultFN[\etakm]) (\xikstaruk - \xikmultFN[\etakm])^\top \\
    &= (\xikstaruk)^{\kcarre} - \xikstaruk \xikmultFN[\etakm]^\top - \xikmultFN[\etakm] (\xikstaruk)^\top + \xikmultFN[\etakm]^{\kcarre} \,.
\end{align*} 

(i) First, from \Cref{def:add_mult_noise}, it flows that $\Expec{\xikstaruk \otimes \xikstaruk}{\Flast_{k-1}} = \aniac$.

(ii) Second, we show that $\expec{\xikmultFN[\etakm]^{\kcarre}}{\Flast_{k-1}}$ converges to $0$ in probability: it is sufficient to show that: $\expec{\| \xikmultFN[\etakm]^{\kcarre} \|_F}{\Flast_{k-1}} \xrightarrow[k \to +\infty]{} 0 \,.$
To do so, we use the fact that 
$ \| \xikmultFN[\etakm]^{\kcarre} \|_F  = \|\xikmultFN[\etakm]\|_2^2,  \,$ 
then with \Cref{asu:main:bound_mult_noise_holder}:
$\expec{\sqrdnrm{\xi_k^\mathrm{mult}(w - \ws)}}{\Flast_{k-1}} \leq \boundMultPrime \|H^{1/2} \etakm \| + \boundMult \|H^{1/2} \etakm \|^2$. And we have the result as we showed that $\etakm \xrightarrow[k \to +\infty]{L^2} 0$.

(iii) Third, it remains to show that $\expec{\xikmultFN[\etakm] (\xikstaruk)^\top}{\Flast_{k-1}} \xrightarrow[k \to +\infty]{L^1} 0$. We use the Cauchy-Schwarz inequality's \ref{app:basic_ineq:cauchy_schwarz_cond} for conditional expectation:
\begin{align*}
    \Expec{ \|\xikmultFN[\etakm] (\xikstaruk)^\top \|_F}{\Flast_{k-1}}^2 &= \Expec{ \|\xikmultFN[\etakm]\|_2\| (\xikstaruk)^\top \|_2}{\Flast_{k-1}}^2 \\
    &\leq \Expec{ \|\xikmultFN[\etakm]\|^2_2}{\Flast_{k-1}} \Expec{ \|\xikstaruk\|_2}{\Flast_{k-1}} \,.
\end{align*}

The sequence of random vectors $(\xikstaruk)_{k \in \N^*}$ is i.i.d., and moreover  we have shown previously that  
$\expec{\sqrdnrm{\xikmultFN[\etakm]}}{\Flast_{k-1}}$ tends to 0, hence  $\expec{\xikmultFN[\etakm] (\xikstaruk)^\top}{\Flast_{k-1}}$ converges to $0$ in probability.
Consequently, we can state that
$\expec{\xik^{\kcarre}}{\Flast_{k-1}}\xrightarrow[k \to +\infty]{\mathbb{P}} \aniac  \,.$

\end{proof}

\section{Generalization of Bach and Moulines (2013).}
\label{app:sec:nonlinear_bm}

In this section, we give the demonstration of \Cref{thm:bm2013_with_nonlinear_operator_compression} which extends Theorem 1 from \cite{bach_non-strongly-convex_2013}; the demonstration is close to the original one.

\subsection{Proof principle}
\label{app:subsec:proof_principale_gen}
For $k$ in $\N^*$, the proof relies (1) on decomposing $\E[\|\Fhess^{1/2} \etabarkm\|^2]$ in two terms using the Minkowski inequality \ref{app:eq:minkowski} to make appear a recursion $(\eta_k^0)_{k\in\N^*}$ without multiplicative noise, and another $(\alpha_k)_{k\in\N^*}$ without additive noise, (2) on an expansion of $\eta_k^0$ and $\overline{\eta}_k^0$ as polynomials in $\gamma$, and (3) on using the Hölder-type \Cref{asu:main:bound_mult_noise_holder} to bound $\alpha_k$. We define the sequence $(\eta_k^0)_{k\in\N^*}$ such that it involves only an additive noise:
\begin{align}
\label{app:eq:def_eta0}   
\eta_k^0 = (\Id_d - \gamma \Fhess)\etakm^0 + \gamma \xikstaruk\,.
\end{align}

Then, we decompose $\fullexpec{\|\Fhess^{1/2} \etabarkm\|^2}$ in the following way using Minkowski inequality \ref{app:eq:minkowski}:
\begin{align}
\label{app:eq:minkowski_decomposition_in_intro}
    \FullExpec{\|\Fhess^{1/2} \etabarkm\|^2} \leq \bigpar{\FullExpec{\|\Fhess^{1/2} \etabarkm^0\|^2}^{1/2} + \FullExpec{\|\Fhess^{1/2}(\etabarkm - \etabarkm^0)\|^2}^{1/2} }^2\,.
\end{align}
The goal is then to establish a bound for the two above quantities.

\textbf{1. Bounding $\E [\|\Fhess^{1/2} \etabarkm^0\|^2]$.}

The bound on $\E [\|\Fhess^{1/2} \etabarkm^0\|^2]$ is given in \Cref{app:lem:bound_on_eta_k_0}.
For $k$ in $\N^*$, the proof relies on an expansion of $\eta_k^0$ and $\overline{\eta}_k^0$ as polynomials in $\gamma$. The recursion defining the sequence $(\eta_k^0)_{k\in\N^*}$ is $\eta_k^0 = (\Id_d - \gamma \Fhess) \etakm^0 + \gamma \xi_k^{\mathrm{add}}$.
If we denote $M_i^k = \bigpar{\Id_d - \gamma \Fhess}^{k-i}$ and $M_i^{i-1} = \Id_d$, we have:
$$
\eta_k^0 = M_1^k \eta_0^0 + \gamma \sum_{i=1}^k M_{i+1}^k \xi_k^{\mathrm{add}}\,.
$$
For $K$ in $\N^*$, it leads to $\overline{\eta}_{K-1}^0 = \frac{1}{K} \sum_{k=0}^{K-1} M_1^k \eta_0^0 + \frac{\gamma}{K} \sum_{k=1}^{K-1} \bigpar{ \sum_{i=k}^K M_{k+1}^i} \xi_k^{\mathrm{add}}$, and with Minkowski inequality \ref{app:eq:minkowski} to:
\begin{align}
\label{app:eq:polynomial_exp_in_gamma}
    \FullExpec{\sqrdnrm{\Fhess^{1/2} \overline{\eta}_{K-1}^0}}^{1/2}\leq \FullExpec{\SqrdNrm{\frac{\Fhess^{1/2}}{K} \sum_{k=0}^{K-1} M_1^k \eta_0^0}}^{1/2} + \FullExpec{\SqrdNrm{  \frac{\gamma \Fhess^{1/2}}{K} \sum_{k=1}^{K-1}\sum_{i=k}^K M_{k+1}^i\xi_k^{\mathrm{add}}}}^{1/2}.
\end{align}

The left term depends only on initial conditions $\eta_0^0$ ($=\eta_0$) and the right term depends only on the additive noise. This is why, in the proof, we expend $\eta_{k-1}^0$ and $\overline{\eta}_{k-1}^0$ separately for the noise process (i.e., when assuming $\eta_0  = 0$) and for the noise-free process
that depends only on the initial conditions (i.e. when assuming that the additive noise $(\xi_k^{\mathrm{add}})_{k\in\N^*}$ is uniformly equal to zero). In the end, the two bounds computed separately may be added.

\textbf{2. Bounding $\E [ \|\Fhess^{1/2}(\etabarkm - \etabarkm^0)\|^2]$.}

The bound on $\E [ \|\Fhess^{1/2}(\etabarkm - \etabarkm^0)\|^2]$ is given in \Cref{app:lem:tight_bound_on_eta_k_minus_eta_k_0}.
For $k$ in $\N^*$, the demonstration is based on an exact expression of $\alpha_k = \eta_k - \eta_k^0$ and $\overline{\alpha}_k$ computed by unrolling the recursion from $\alpha_k$ to $\alpha_0$. Because $\alpha_0 = 0$ and because there is no additive noise involved in $\alpha_k$, we obtain for $K$ in $\N^*$, an expression of $\overline{\alpha}_{K-1}$ that depends only on the multiplicative noise at iteration $k$ in $\{1, \cdots, K\}$:
\begin{align*}
    \overline{\alpha}_{K-1} &= \frac{\gamma}{K} \sum_{k=1}^{K-1} (\Id_d - (\Id_d - \gamma \Fhess)^{K-k}) (\gamma \Fhess)^{-1} \xi_k^{\mathrm{mult}}(\etakm) \,.
\end{align*}
We then show (\Cref{app:eq:ligne1}) that bounding $\E [ \|\Fhess^{1/2}(\etabarkm - \etabarkm^0)\|^2]$ leads to bound the following sum $\frac{1}{K^2} \sum^{K-1}_{k=1} \E [\sqrdnrm{\Fhess^{-1/2}\xi_k^{\mathrm{mult}}(\etakm)} \mid \Flast_{k-1}]$, and this bound is established using the Hölder-type \Cref{asu:main:bound_mult_noise_holder}; which concludes this part of the proof.

\begin{figure}
    \centering
    \includegraphics[width=0.9\linewidth]{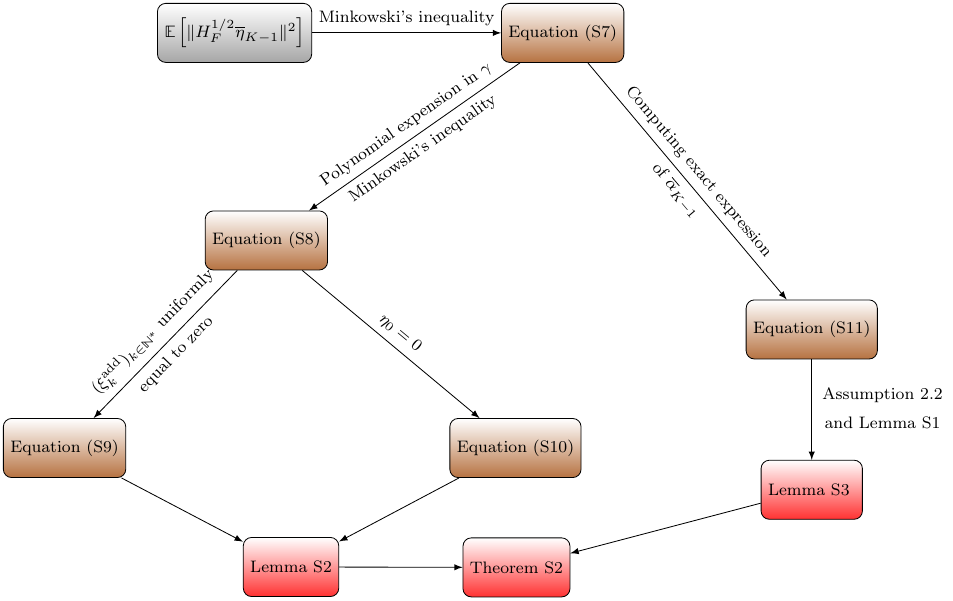}
    \vspace{-0.25cm}
    \caption{Proof principle of \Cref{app:thm:bm2013_with_nonlinear_operator_compression}}
    \label{fig:grap_of_proofs_nonlin}
    \vspace{-0.3cm}
\end{figure}

\subsection{Two bounds}
\label{app:subsec:noise_process_generalized}

In this subsection, we give two lemmas that provide a bound on $\E[\|\Fhess^{1/2} \etabarkm^0\|^2]$ and $\E[\|\Fhess^{1/2}(\etabarkm - \etabarkm^0)\|^2]$. 

These bounds are required due to the decomposition of $\E[\|\Fhess^{1/2} \etabarkm\|^2]$ done in \Cref{app:eq:minkowski_decomposition_in_intro}.
\begin{itemize}[topsep=2pt,itemsep=1pt,leftmargin=0.5cm,noitemsep]
    \item The bound on $\E[\|\Fhess^{1/2}\overline{\eta}_k^0\|^2]$ is given in \Cref{app:lem:bound_on_eta_k_0}. It is established by decomposing the noise process and the noise-free process. The bound on the noise process comes from Lemma 2 \citep{bach_non-strongly-convex_2013} and involves the additive noise's covariance $\aniac$.
    \item The bound on $\E[\|\Fhess^{1/2} (\overline{\eta}_K - \overline{\eta}_K^0)\|^2]$ is established in \Cref{app:lem:tight_bound_on_eta_k_minus_eta_k_0}.
\end{itemize} 

Note that in order to demonstrate \Cref{app:lem:tight_bound_on_eta_k_minus_eta_k_0}, we need to bound $ \sum_{k=1}^K\| \Fhess^{1/2} \eta_k \|^2/K$. This is done in \Cref{app:lem:from_bm2013} which is an adaptation of Lemma 1 from \cite{bach_non-strongly-convex_2013} to random mechanisms. This auxiliary lemma holds for any kind of multiplicative noise---linear or non-linear. 

Below lemma provides a bound on $\E[\|\Fhess^{1/2}\overline{\eta}_k^0\|^2]$. 

\fbox{
\begin{minipage}{0.97\textwidth}
\begin{lemma}[Bound on $\fullexpec{\|\Fhess^{1/2}\overline{\eta}_k^0\|^2}$]
\label[lemma]{app:lem:bound_on_eta_k_0}
Under the setting considered in \Cref{def:class_of_algo}, under \Cref{asu:main:bound_add_noise}, for any $K$ in $\N^*$ and any step-size $\gamma \in \R^+$ s.t. $\gamma  \Ftrace^2 \leq 1$, 
the sequence $(\eta_k^0)_{k\in\N^*}$ defined in \Cref{app:eq:def_eta0} verifies the following bound:
\begin{align*}
    \FullExpec{\|\Fhess^{1/2} \etabarkm^0 \|^2}^{1/2} \leq &\frac{1}{\sqrt{K}} \bigpar{ \ffrac{\| \Fhess^{-1/2} \eta_0 \|}{\gamma \sqrt{K}} \wedge \frac{\|\eta_0 \|}{\sqrt{\gamma}} + \sqrt{\Tr{\aniac \Fhess^{-1} }} } \,.
\end{align*}

\end{lemma}
\end{minipage}
}

\begin{proof}

The proof relies on the proof presented by \cite{bach_non-strongly-convex_2013} and is done separately for the noise process and for the noise-free process that depends only on the initial condition. The bounds may then be added (see the discussion in \Cref{app:subsec:proof_principale_gen}).

\textbf{Noise-free process.}

As in section A.3 from \cite{bach_non-strongly-convex_2013}, we assume in this section that the random fields $(\xi_k^\mathrm{add})_{k\in \N^*}$ is uniformly equal to zero and that $\gamma  \Ftrace^2 \leq 1$. We thus have for any $k$ in $\N^*$ that~$\eta_k^0 = (\Id_d - \gamma \Fhess) \etakm^0 $.

\textit{First inequality.}
By recursion, we have $\eta_k^0 = (\Id_d - \gamma \Fhess)^k \eta_0^0 $, averaging over $K$ in $\N^*$ and computing the resulting geometric sum, we have:
\begin{align*}
    \overline{\eta}_{K-1}^0 = \ffrac{1}{K} \sum_{k=0}^{K-1} (\Id_d - \gamma \Fhess)^{k} \eta_0^0 = \frac{1}{K} (\Id_d - (\Id_d - \gamma \Fhess)^{K-1}) (\gamma \Fhess)^{-1} \eta_0^0 \preccurlyeq \ffrac{1}{\gamma K} \Fhess^{-1} \eta_0^0.
\end{align*}

And because $\eta_0^0 = \eta_0$, it gives $\FullExpec{\PdtScl{\etabarkm^0}{\Fhess \etabarkm^0}}  \leq \frac{\|\Fhess^{1/2}\eta_0\|^2}{\gamma^2 K^2}$.

\textit{Second inequality.} From the expression of $\eta_k^0$ flows:
\[
\fullexpec{ \sqrdnrm{\eta_k^0}} = \fullexpec{ \sqrdnrm{\etakm^0}} - 2 \gamma \PdtScl{\etakm^0}{\Fhess \etakm^0} + \gamma^2\PdtScl{\etakm^0}{\Fhess^2 \etakm^0} \,.
\]
Considering that $\Fhess \preccurlyeq \Tr{\Fhess} \Id_d \preccurlyeq  \Ftrace^2 \Id_d$ (\Cref{def:class_of_algo}) and that $\gamma  \Ftrace^2 \leq 1$, 
because $\eta_0^0 = \eta_0$, by convexity we have: $\fullexpec{\sqrdnrm{\Fhess^{1/2} \etabarkm^0}}  \leq \frac{1}{K} \sum_{k=1}^{K} \fullexpec{\sqrdnrm{\Fhess^{1/2} \etakm^0}} \leq \frac{\SqrdNrm{\eta_0}}{\gamma K}\,.$

\textit{Putting things together.}

In the end, we take the minimum of the two above bounds and obtain that:
\begin{align}
\label{app:eq:noise_unif_equal_to_zero}
    \fullexpec{\sqrdnrm{\Fhess^{1/2} \etabarkm^0}} \leq \ffrac{\|\Fhess^{-1/2} \eta_0\|^2}{\gamma^2 K^2} \wedge \ffrac{\SqrdNrm{\eta_0}}{\gamma K} \,.
\end{align}

\textbf{Noise process.}

We assume in this part that $\eta_0^0 = \eta_0 = 0$. We apply Lemma 2 from \citet{bach_non-strongly-convex_2013} to $\etakm^0$. This sequence of iterates has an i.i.d. noise process $(\xikstaruk)_{k\in\N^*}$ which is such that~$\FullExpec{\xikstaruk \otimes \xikstaruk} =~\aniac$ (existence guaranteed by \Cref{asu:main:bound_add_noise}). Therefore we have:
\begin{align}
\label{app:eq:no_initial_bias}
    \fullexpec{\sqrdnrm{\Fhess^{1/2} \etabarkm^0}} \leq \ffrac{ \Tr{\aniac \Fhess^{-1} }}{K} \,.
\end{align}

\textbf{Putting things together.}
We now take results derived from the part without noise and the part with noise, and we get from Minkowski inequality:
\begin{align*}
    \FullExpec{\|\Fhess^{1/2} \etabarkm^0 \|^2}^{1/2} \leq &\frac{1}{\sqrt{K}} \bigpar{ \ffrac{\| \Fhess^{-1/2} \eta_0 \|}{\gamma \sqrt{K}} \wedge \frac{\|\eta_0 \|}{\sqrt{\gamma}} + \sqrt{\Tr{\aniac \Fhess^{-1} }} } \,.
\end{align*}
\end{proof}

Below lemma provides a bound on $\fullexpec{ \|\Fhess^{1/2} (\overline{\eta}_K - \overline{\eta}_K^0)\|^2}$. 

\fbox{
\begin{minipage}{0.97\textwidth}
\begin{lemma}[Bound on $\fullexpec{\|\Fhess^{1/2} (\overline{\eta}_K - \overline{\eta}_K^0)\|^2}$]
\label[lemma]{app:lem:tight_bound_on_eta_k_minus_eta_k_0}
Under the setting considered in \Cref{def:class_of_algo} with $\mu> 0$, under \Cref{asu:main:bound_add_noise} , under \Cref{asu:main:bound_mult_noise,asu:main:bound_mult_noise_holder}, for any $K$ in $\N^*$ and any step-size $\gamma \in \R^+$ s.t. $\gamma ( \Ftrace^2 + 2 \boundMult) < 1$, 
the sequence  $(\overline{\eta}_k - \overline{\eta}_k^0)_{k\in\N^*}$ verifies the following bound:
\begin{align*}
\FullExpec{\|\Fhess^{1/2} (\overline{\eta}_K - \overline{\eta}_K^0)\|^2}^{1/2}] &\leq \frac{1}{\sqrt{K}} \Bigg( \sqrt{\boundMultPrime \mu^{-1}} \bigpar{\ffrac{5\boundAdd \gamma}{1 - \gamma( \Ftrace^2 + 2 \boundMult)}}^{1/4} \\
&\qqquad+ \sqrt{\boundMult\mu^{-1}} \bigpar{\ffrac{15\boundAdd \gamma }{1 - \gamma( \Ftrace^2 + 2 \boundMult)}}^{1/2} \Bigg) \,.
\end{align*}

\end{lemma}
\end{minipage}
}

\begin{remark}
    To demonstrate \Cref{app:lem:tight_bound_on_eta_k_minus_eta_k_0}, we use the Hölder-type \Cref{asu:main:bound_mult_noise_holder}. This is why we obtain a term with a square root in the bound.
\end{remark}

\begin{proof}

Let $k$ in $\N^*$, we denote $\alpha_k = \eta_k - \eta_k^0$, with $\eta_k = (\Id_d - \gamma \Fhess)\etakm + \gamma \xik$ and $\eta_k^0 = (\Id_d - \gamma \Fhess)\etakm^0 + \gamma \xikstaruk$. First, we write the exact expression of $\alpha_{k-1}$:
\begin{align*}
    \alpha_k &= (\Id_d - \gamma \Fhess) \alpha_{k-1} + \gamma (\xik - \xikstaruk) \\
    &= (\Id_d - \gamma \Fhess)^k \alpha_0 + \gamma \sum_{i=1}^k (\Id_d - \gamma \Fhess)^{k-i} (\xi_i(\eta_{i-1}) - \xi_i^{\mathrm{add}}) \,,
\end{align*}    
and because $\eta_0^0 = \eta_0$, it follows that $\alpha_0 = \eta_0 - \eta_0^0 = 0$. Averaging over $K$ in $\N^*$, we have the exact expression of $\overline{\alpha}_{K-1}$:
\begin{align*}
    \overline{\alpha}_{K-1} &= \frac{\gamma}{K} \sum_{k=0}^{K-1} \sum_{i=1}^k (\Id_d - \gamma \Fhess)^{k-i} (\xi_i(\eta_{i-1}) - \xi_i^{\mathrm{add}})) \\
    &= \frac{\gamma}{K} \sum_{i=1}^{K-1} \bigpar{\sum_{k=i}^{K-1} (\Id_d - \gamma \Fhess)^{k-i}} (\xi_i(\eta_{i-1}) - \xi_i^{\mathrm{add}}))   \,.
\end{align*}

Computing the geometric sum results in:
\begin{align*}
    \overline{\alpha}_{K-1} &= \frac{\gamma}{K} \sum_{k=1}^{K-1} (\Id_d - (\Id_d - \gamma \Fhess)^{K-k}) (\gamma \Fhess)^{-1} (\xik  - \xikstaruk) \,.
\end{align*}

And because for any $k$ in $\N$, $0 \preccurlyeq (\Id_d - \gamma  \Fhess)^k \preccurlyeq \Id_d $, we obtain: 
$$\overline{\alpha}_{K-1}\preccurlyeq \frac{1}{K} \sum_{k=1}^{K-1}  \Fhess^{-1} (\xik  - \xikstaruk)\,,$$
hence $\sqrdnrm{\Fhess^{1/2} \overline{\alpha}_{K-1}} = \sqrdnrm{\frac{1}{K} \sum_{k=1}^{K-1}  \Fhess^{-1/2} (\xik  - \xikstaruk)}$. We take full expectation, because for any $k$ in $\N^*$, by \Cref{def:class_of_algo,def:add_mult_noise}, $\xikmultFN[\etakm] = \xik  - \xikstaruk$ is $\Flast_k$-measurable and $\Expec{\xikmultFN[\etakm]}{\Flast_{k-1}} = 0 $, we can unroll the sum and we have in the end that the variance of the sum is the sum of variances:
\begin{align}
\label{app:eq:ligne1}
    \FullExpec{\SqrdNrm{\Fhess^{1/2} \overline{\alpha}_{K-1}}} \leq \frac{1}{K^2} \sum^{K-1}_{k=1} \Expec{\SqrdNrm{\Fhess^{-1/2}\xikmultFN[\etakm]}}{\Flast_{k-1}}\,.
\end{align}
Computing $\expec{\sqrdnrm{\Fhess^{-1/2}\xikmultFN[\etakm]}}{\Flast_{k-1}}$ for $k$ in $\N$, we first have:
\begin{align*}
\sqrdnrm{\Fhess^{-1/2}\xikmultFN[\etakm]} \leq \vertiii{\Fhess^{-1/2}}^2\sqrdnrm{\xikmultFN[\etakm]}\,,
\end{align*}
where we used \Cref{app:lem:ineq_subordinate_norm}. Because $\Fhess$ is a symmetric semi-positive matrix, we have $\vertiii{\Fhess^{-1/2}}^2 = 1/\mu$, hence:
$\sqrdnrm{\Fhess^{-1/2}\xikmultFN[\etakm]} \leq \mu^{-1}\sqrdnrm{\xikmultFN[\etakm]}$.
Taking expectation conditionally to the $\sigma$-algebra $\Flast_{k-1}$ and invoking \Cref{asu:main:bound_mult_noise_holder} gives:
\begin{align}
\begin{split}
\label{app:eq:ligne2}
    \expec{\sqrdnrm{\Fhess^{-1/2}\xikmultFN[\etakm]}}{\Flast_{k-1}} \leq \mu^{-1} (\boundMultPrime \|\Fhess^{1/2} \etakm \| + 3 \boundMult \|\Fhess^{1/2} \etakm \|^2 )\,.
\end{split}
\end{align}

Combining equations~\ref{app:eq:ligne1} and~\ref{app:eq:ligne2}, we obtain:
\begin{align*}
    \fullexpec{\sqrdnrm{\Fhess^{1/2} \overline{\alpha}_{K-1}}} &\leq \frac{\boundMultPrime}{\mu K^2} \sum^{K-1}_{k=1}  \fullexpec{\|\Fhess^{1/2} \etakm \|} + \frac{3\boundMult}{\mu K^2} \sum^{K-1}_{k=1} \fullexpec{ \|\Fhess^{1/2} \etakm \|^2} \,.
\end{align*}

Now using Jensen's inequality for concave function allows us to write:
\begin{align*}
    \frac{1}{K} \sum_{k=1}^K \fullexpec{\|\Fhess^{1/2} (w - \ws) \|} \leq \frac{1}{K} \sum_{k=1}^K \sqrt{\fullexpec{ \|\Fhess^{1/2} (w - \ws) \|^2}} \leq \sqrt{\frac{1}{K}\sum_{k=1}^K \fullexpec{ \|\Fhess^{1/2} (w - \ws) \|^2}} \,,
\end{align*}
thus we have:
\begin{align*}
    \fullexpec{\sqrdnrm{\Fhess^{1/2} \overline{\alpha}_{K-1}}} &\leq \frac{\boundMultPrime}{\mu K} \sqrt{ \frac{1}{K} \sum^{K-1}_{k=1} \fullexpec{ \|\Fhess^{1/2} \etakm \|^2}} + \frac{3\boundMult}{\mu K^2} \sum^{K-1}_{k=1} \fullexpec{\|\Fhess^{1/2} \etakm \|^2} \,.
\end{align*}

Using \Cref{app:lem:from_bm2013} (with $\eta_0 = 0$), we finally obtain:
\begin{align*}
    \fullexpec{\sqrdnrm{\Fhess^{1/2} \overline{\alpha}_{K-1}}} &\leq \frac{1}{K} \bigpar{\boundMultPrime \mu^{-1} \sqrt{\ffrac{5\boundAdd \gamma}{1 - \gamma( \Ftrace^2 + 2 \boundMult)}} +   \ffrac{15\boundAdd \gamma \boundMult \mu^{-1}}{1 - \gamma( \Ftrace^2 + 2 \boundMult)} } \,.
\end{align*}

In the end, we take the square root (and use that for any $a,b$ in $\R_+$, $\sqrt{a + b} \leq \sqrt{a} + \sqrt{b}$) which allows concluding:
\begin{align*}
\FullExpec{\|\Fhess^{1/2} (\overline{\eta}_K - \overline{\eta}_K^0)\|^2}^{1/2} &\leq \frac{1}{\sqrt{K}} \Bigg( \sqrt{\boundMultPrime \mu^{-1}} \bigpar{\ffrac{5\boundAdd \gamma}{1 - \gamma( \Ftrace^2 + 2 \boundMult)}}^{1/4} \\
&\qqquad+ \sqrt{\boundMult\mu^{-1}} \bigpar{\ffrac{15\boundAdd \gamma }{1 - \gamma( \Ftrace^2 + 2 \boundMult)}}^{1/2} \Bigg) \,.
\end{align*}
\end{proof}

\subsection{Final theorem}
\label{app:subsec:final_thm_bm2013_with_nonlinear_comp}

In this section, we gather the pieces of proof required to demonstrate \Cref{thm:bm2013_with_nonlinear_operator_compression}.  

\fbox{
\begin{minipage}{0.97\textwidth}
\begin{theorem}[Non-linear multiplicative noise]
\label{app:thm:bm2013_with_nonlinear_operator_compression}
Under \Cref{asu:main:second_moment_noise,asu:main:bound_add_noise}, considering any constant step-size $\gamma$ such that $ \gamma( \Ftrace^2 + 2 \boundMult) \leq 1/2$, then for any $K$ in $\N^*$, the sequence $(w_k)_{k\in\N^*}$ produced by a setting such as in \Cref{def:class_of_algo} verifies the following bound:
\begin{align*}
    \E[ F(\overline{w}_{K-1}) - F(\ws) ] &\leq \frac{1}{2K} \Bigg(\ffrac{\| \Fhess^{-1/2} \eta_0 \|}{\gamma \sqrt{K}} \wedge \frac{\|\eta_0 \|}{\sqrt{\gamma}} + \sqrt{\Tr{\aniac \Fhess^{-1} }}  \\ 
    &\qqquad+ \bigpar{10\boundAdd \gamma}^{1/4}\sqrt{\boundMultPrime \mu^{-1}}  
     + \bigpar{30\boundAdd \gamma}^{1/2} \sqrt{\boundMult\mu^{-1}} \Bigg)^2 \,.
\end{align*}
\end{theorem}
\end{minipage}
}

\begin{proof}

As explained in the discussion in \Cref{app:subsec:proof_principale_gen} (\Cref{app:eq:minkowski_decomposition_in_intro}), we define the sequence $(\eta_k^0)_{k\in\N^*}$ which involves only an additive noise $\eta_k^0 = (\Id_d - \gamma \Fhess)\etakm^0 + \gamma \xikstaruk$. Then, we decompose $\fullexpec{\|\Fhess^{1/2} \etabarkm\|}$ using Minkowski's inequality \ref{app:eq:minkowski}:
\begin{align}
\label{app:eq:minkowski_decomposition}
        \FullExpec{\|\Fhess^{1/2} \etabarkm\|^2} \leq \bigpar{\FullExpec{\|\Fhess^{1/2} \etabarkm^0\|^2}^{1/2} + \FullExpec{\|\Fhess^{1/2}(\etabarkm - \etabarkm^0)\|^2}^{1/2} }^2\,.
\end{align}

\textit{First term.}

To bound the first term, we use \Cref{app:lem:bound_on_eta_k_0} which gives:
\begin{align*}
    \FullExpec{\|\Fhess^{1/2} \etabarkm^0 \|^2}^{1/2} \leq &\frac{1}{\sqrt{K}} \bigpar{ \ffrac{\| \Fhess^{-1/2} \eta_0 \|}{\gamma \sqrt{K}} \wedge \frac{\|\eta_0 \|}{\sqrt{\gamma}} + \sqrt{\Tr{\aniac \Fhess^{-1} }} } \,.
\end{align*}

\textit{Second term.}

From \Cref{app:lem:tight_bound_on_eta_k_minus_eta_k_0}, we have:
\begin{align*}
 \FullExpec{ \|\Fhess^{1/2} (\overline{\eta}_K - \overline{\eta}_K^0)\|^2}^{1/2} &\leq \frac{1}{\sqrt{K}} \Bigg( \sqrt{\boundMultPrime \mu^{-1}} \bigpar{\ffrac{5\boundAdd \gamma}{1 - \gamma( \Ftrace^2 + 2 \boundMult)}}^{1/4} \\
&\qqquad+ \sqrt{\boundMult\mu^{-1}} \bigpar{\ffrac{15\boundAdd \gamma }{1 - \gamma( \Ftrace^2 + 2 \boundMult)}}^{1/2} \Bigg) \,.
\end{align*}

\textbf{Final bound.}
Hence, back to \Cref{app:eq:minkowski_decomposition}, we get:
\begin{align*}
    \FullExpec{\|\Fhess^{1/2} \etabarkm \|^2}^{1/2}
    \leq &\frac{1}{\sqrt{K}} \Bigg( \ffrac{\| \Fhess^{-1/2} \eta_0 \|}{\gamma \sqrt{K}} \wedge \frac{\|\eta_0 \|}{\sqrt{\gamma}} + \sqrt{\Tr{\aniac \Fhess^{-1} }}  \\
    &\qquad+\sqrt{\boundMultPrime \mu^{-1}} \bigpar{\ffrac{5\boundAdd \gamma}{1 - \gamma( \Ftrace^2 + 2 \boundMult)}}^{1/4} \\
    &\qquad+ \sqrt{\boundMult\mu^{-1}} \bigpar{\ffrac{15\boundAdd \gamma }{1 - \gamma( \Ftrace^2 + 2 \boundMult)}}^{1/2} \Bigg) \,,
\end{align*}
and considering $ \gamma( \Ftrace^2 + 2 \boundMult) \leq 1/2$, it concludes the proof because $\E[ F(\overline{w}_{K-1}) - F(\ws) ] = \E[\|\Fhess^{1/2} \etabarkm \|^2] / 2$:
\begin{align*}
    \E[ F(\overline{w}_{K-1}) - F(\ws) ] &\leq \frac{1}{2K} \Bigg(\ffrac{\| \Fhess^{-1/2} \eta_0 \|}{\gamma \sqrt{K}} \wedge \frac{\|\eta_0 \|}{\sqrt{\gamma}} + \sqrt{\Tr{\aniac \Fhess^{-1} }}  +\bigpar{10\boundAdd \gamma}^{1/4} \sqrt{\boundMultPrime \mu^{-1}}  \\
    &\qqquad+ \bigpar{30\boundAdd \gamma}^{1/2}\sqrt{\boundMult\mu^{-1}}   \Bigg)^2 \,.
\end{align*}

\end{proof}

\section{Generalisation of Bach and Moulines (2013) for linear multiplicative noise.}
\label{app:sec:linear_bm}

In this Section, we give the demonstration of \Cref{thm:bm2013_with_linear_operator_compression} which extends Theorem 1 from \cite{bach_non-strongly-convex_2013} to the case of linear multiplicative noise. The demonstration follows the same steps as the one given by \citet{bach_non-strongly-convex_2013}. The minor differences lie in the generality of the form of the multiplicative noise in our approach. \citet{bach_non-strongly-convex_2013} only analyse LMS algorithm, while we here consider \eqref{eq:LSA} with assumptions on the linear multiplicative noise process. Moreover, our theorem decomposes into 3 terms instead of 2.

\subsection{Proof principle}
\label{app:subsec:proof_principale_lin}

For $k$ in $\N^*$, the proof relies on an expansion of $\eta_k$ and $\overline{\eta}_k$ as polynomials in $\gamma$. Because we consider a linear multiplicative noise, there exists a matrix $\Xi_k$ in $\R^{d\times d}$ s.t. for any $z$ in $\R^d$, $\xi_k^{\mathrm{mult}}(z) = \Xi_k z$ (\Cref{asu:main:bound_mult_noise_lin}); hence the recursion defined in \Cref{def:class_of_algo} can be rewritten as:
\[
\eta_k = \etakm - \gamma \nabla F(\etakm) + \gamma \xi_k^{\mathrm{mult}}(\etakm) + \gamma \xi_k^{\mathrm{add}} = (\Id_d - \gamma\Fhess + \gamma \Xi_k)\etakm + \gamma \xi_k^{\mathrm{add}} \,.
\]
We denote $M_i^k = \bigpar{\Id_d - \gamma \Fhess + \gamma \Xi_k} \cdots \bigpar{\Id_d - \gamma \Fhess + \gamma \Xi_i}$ and $M_i^{i-1} =\Id_d$, then we have that $\eta_k = M_1^k \eta_0 + \gamma \sum_{i=1}^k M_{i+1}^k \xi_k^{\mathrm{add}}.$

For $K$ in $\N^*$, it leads to $\overline{\eta}_{K-1} = \frac{1}{K} \sum_{k=0}^{K-1} M_1^k \eta_0 + \frac{\gamma}{K} \sum_{k=1}^{K-1} \bigpar{ \sum_{i=k}^K M_{k+1}^i} \xi_k^{\mathrm{add}}$, and with Minkowski's inequality \ref{app:eq:minkowski} to:
\begin{align}
\label{app:eq:polynomial_expension_linear}
    \sqrt{\FullExpec{\left\| \Fhess^{1/2} \overline{\eta}_{K-1} \right\|^2}} \leq \FullExpec{\left\| \frac{\Fhess^{1/2}}{K} \sum_{k=0}^{K-1} M_1^k \eta_0\right\|^2}^{1/2} + \FullExpec{\left\|  \frac{\gamma \Fhess^{1/2}}{K} \sum_{k=1}^{K-1}    \sum_{i=k}^K M_{k+1}^i \xi_k^{\mathrm{add}}\right\|^2}^{1/2}.
\end{align}

The left term depends only on initial conditions and the right term depends only on the noise process. 

This is why, in the proof, we expend $\eta_{k-1}$ and $\overline{\eta}_{k-1}$ separately for the noise process (i.e., when assuming $\eta_0 = 0$) and for the noise-free process
that depends only on the initial conditions (i.e. when assuming that the additive noise $(\xi_k^{\mathrm{add}})_{k\in\N^*}$ is uniformly equal to zero). In the end, the two bounds computed separately may be added.

To study the noise process, inspiring from \citet{bach_non-strongly-convex_2013}, we define the following sequence:
\begin{equation}
\label{app:eq:double_sequence_eta}
\left\{\begin{aligned}
& \eta_k^0 =  (\Id_d - \gamma \Fhess) \etakm^0 + \gamma \xikstaruk \\
& \eta_k^r = (\Id_d - \gamma \Fhess) \etakm^r + \gamma \xikmultFN[\etakm^{r-1}] \qquad \text{with} \qquad \forall r \geq 0\,, \eta_0^r = 0 \,. 
\end{aligned}\right.
\end{equation}
Then, we decompose $\fullexpec{\|\Fhess^{1/2} \etabarkm\|^2}$ in the following way using Minkowski's inequality \ref{app:eq:minkowski}:
\begin{align*}
    \sqrt{\FullExpec{\|\Fhess^{1/2} \etabarkm\|^2}} \leq\E[\|\Fhess^{1/2} \sum_{i=0}^r \etabarkm^i\|^2]^{1/2} + \E[\|\Fhess^{1/2}(\etabarkm - \sum_{i=0}^r \etabarkm^i)\|^2]^{1/2}.
\end{align*}
The goal is then to establish a bound for the two above quantities.

\begin{figure}
    \centering
    \includegraphics[width=0.9\linewidth]{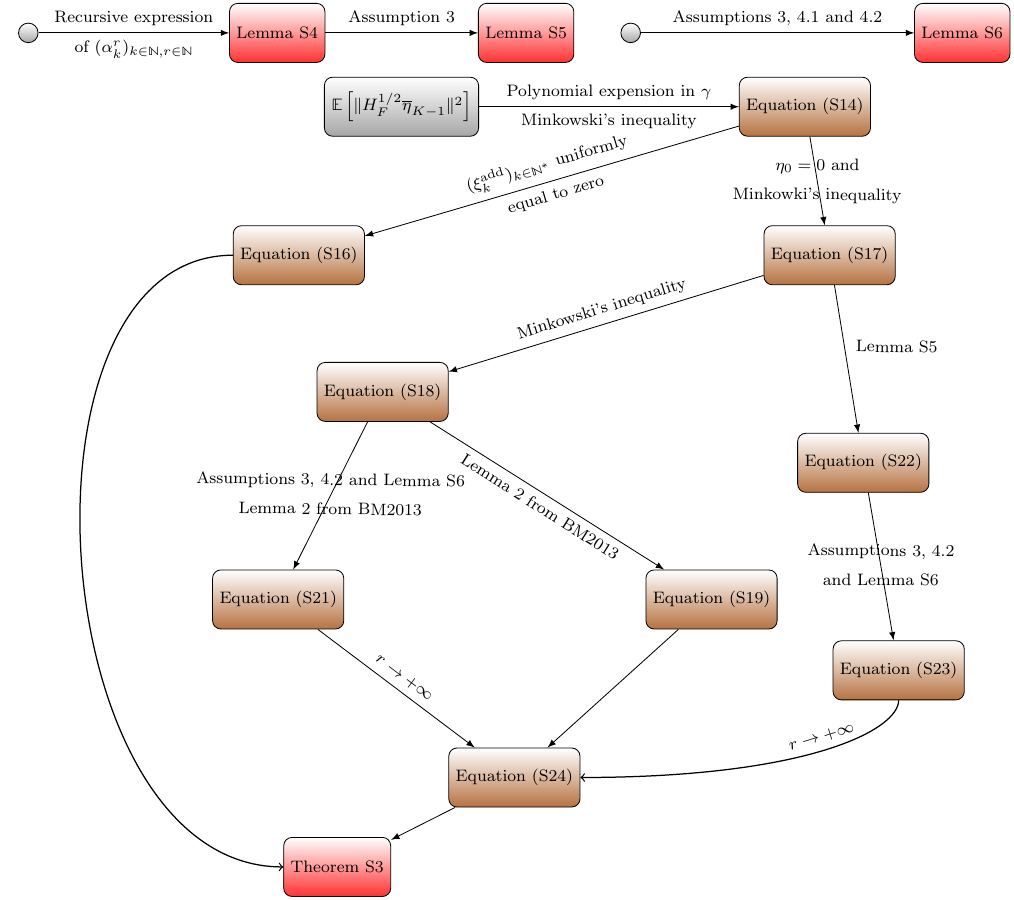}
    \vspace{-0.25cm} 
    \caption{Proof principle of \Cref{app:thm:bm2013_with_linear_operator_compression}. }
    \label{fig:grap_of_proofs_lin}
    \vspace{-0.3cm}
\end{figure}

\subsection{Lemmas for the noise process}
\label{app:subsec:noise_process}

In this Subsection, we provide lemmas for the noise process, and thus we suppose that $\eta_0 = 0$. The noise-free process is later considered in \Cref{app:subsec:final_thm_bm2013_with_linear_comp} and puts together with the results of the coming Subsection. The sketch of the proof relies on establishing two bounds. 
\begin{itemize}[topsep=2pt,itemsep=1pt,leftmargin=0.5cm,noitemsep]
    \item For $r,k$ in $\N \times \N^*$, noting $\alpha_k^r= \eta_k - \sum_{i=0}^{r} \eta_k^i$, the first one is a bound on $\E[\|\Fhess^{1/2} \overline{\alpha}_{K-1}^r\|^2]$ that tends to zero when $r$ tends to $+\infty$.
    \item The second one is on $\sum_{i=0}^{r} \E[\|\Fhess^{1/2} \etabarkm^i\|^2]$ and is established using Lemma 2 from \citep{bach_non-strongly-convex_2013}. It will correspond to the final variance term and it involves the additive noise's covariance $\aniac$. 
\end{itemize}

In the following, we provide \Cref{app:lem:bound_cov_eta_from_bm2013,app:lem:bounding_rec,app:lem:recurrence_from_bm2013}. Let $r,k $ in $\N \times \N^*$. 
\begin{itemize}[topsep=2pt,itemsep=1pt,leftmargin=0.5cm,noitemsep]
    \item \Cref{app:lem:recurrence_from_bm2013} builds a recursive expression of $\alpha_{k}^r = \eta_k - \sum_{i=0}^{r} \eta_k^i$.
    \item \Cref{app:lem:bounding_rec} provides a bound on~$\E[\|\Fhess^{1/2}\overline{\alpha}_{K-1}^r\|^2]$ which involves $\E \|\xikmultFN[\etakm^r]\|^2$.
    \item \Cref{app:lem:bound_cov_eta_from_bm2013} bounds the covariance of $\etakm^r$, this result will be necessary when computing
    the expectation of $\xikmultFN[\etakm^{r}]^{\kcarre}$.
\end{itemize}

Below, we provide the lemma that builds a recursive expression of $\eta_k - \sum_{i=0}^{r} \eta_k^i$, with $k,r $ in $\N^*$.

\fbox{
\begin{minipage}{0.97\textwidth}
\begin{lemma}[A recursion on $\eta_k - \sum_{i=0}^{r} \eta_k^i$]
\label[lemma]{app:lem:recurrence_from_bm2013}
Under the setting given in \Cref{def:class_of_algo}, considering that $\xikmultFN[\cdot]$ is linear (\Cref{asu:main:bound_mult_noise_lin}), for any $k$ in $\N^*$ and any step-size $\gamma>0$, considering $(\eta_k^r)_{r\in \N}$ as given by \Cref{app:eq:double_sequence_eta}, denoting for $r$ in $\N$, $\alpha_k^r = \eta_k - \sum_{i=0}^{r} \eta_k^i$, we have the following recursive expression for the sequence of iterate $(\alpha_k^r)_{r\in \N}$:
\begin{align*}
\forall r \geq 0, \alpha_k^r = (\Id_d - \gamma \Fhess  )\alpha_{k-1}^r + \xikmultFN[\alpha_{k-1}^r] +  \gamma \xikmultFN[\etakm^r]  \,.
\end{align*}
\end{lemma}
\end{minipage}
}

\begin{proof}
Let $k$ in $\N^*$, the proof is done by recursion. For $r=0$, by \Cref{def:class_of_algo,def:add_mult_noise}, we have $\eta_k = \etakm - \gamma \nabla F(\wkm) + \gamma \xi_k(\etakm) = (\Id_d -\gamma \Fhess) \etakm + \gamma \xikstaruk + \gamma \xikmultFN[\etakm]$, which gives:
\begin{align*}
    \alpha_k^0 = \eta_k - \eta_k^0 &= \bigg\{ (\Id_d -\gamma \Fhess) \etakm + \gamma \xikstaruk + \gamma \xikmultFN[\etakm]  \bigg\} - \bigg\{ (\Id_d -\gamma \Fhess) \etakm^0 + \gamma \xikstaruk \bigg\} \\
    &= (\Id_d - \gamma \Fhess) (\etakm - \etakm^0) + \gamma \xikmultFN[\etakm] \\
    &= (\Id_d - \gamma \Fhess) (\etakm - \etakm^0) + \gamma \xikmultFN[\etakm - \etakm^0] + \gamma \xikmultFN[\etakm^0]  \,,
\end{align*}
which is possible because $\xi_k^{\mathrm{mult}}$ is linear (\Cref{asu:main:bound_mult_noise_lin}).
To go from $r$ to $r+1$, we have $
    \alpha_k^{r+1} = \eta_k - \sum_{i=0}^{r+1} \eta_k^i = \eta_k - \sum_{i=0}^{r} \eta_k^i - \eta_k^{r+1}$.
Then by definition of $\eta_k^{r+1}$ and using the hypothesis:
\begin{align*}
    \alpha_k^{r+1} &= (\Id_d - \gamma \Fhess) \bigpar{ \etakm - \sum_{i=0}^r \etakm^i } + \xikmultFN[\etakm - \sum_{i=0}^r \etakm^i] + \gamma \xikmultFN[\etakm^r] \\
    &\qquad -  (\Id_d -\gamma \Fhess) \etakm^{r+1} - \gamma \xikmultFN[\etakm^{r}]  \\
    &= (\Id_d - \gamma \Fhess) \bigpar{ \etakm - \sum_{i=0}^{r+1} \etakm^i } + \xikmultFN[\etakm - \sum_{i=0}^{r+1} \etakm^i] + \gamma \xikmultFN[\etakm^{r+1}]\,,
\end{align*}
again by linearity. This concludes the proof.
\end{proof}

The next lemma is the adaptation to our settings of Lemma 1 from \cite{bach_non-strongly-convex_2013}. We give a bound on~$\fullexpec{\|\Fhess^{1/2}\overline{\alpha}_{K-1}^r\|^2}$ with a quantity that tends to $0$. This result will be used in the final demonstration of \Cref{app:thm:bm2013_with_linear_operator_compression}.

\fbox{
\begin{minipage}{0.97\textwidth}
\begin{lemma}[Bound on $\eta_K - \sum_{i=0}^{r} \eta_K^i$]
\label[lemma]{app:lem:bounding_rec}
Under the setting given in \Cref{def:class_of_algo}, considering that $\xi_k^{\mathrm{mult}}$ is linear (\Cref{asu:main:bound_mult_noise_lin}), for any $r,K$ in $\N \times \N^*$ and any step-size $\gamma$ s.t. $\gamma ( \Ftrace^2 + \boundMult) \leq 1$, the recursion $\alpha_K^r= \eta_K - \sum_{i=0}^{r} \eta_K^i$ verifies the following bound:
\begin{align*}
\forall r \geq 0, (1- \gamma ( \Ftrace^2 + \boundMult))\E \PdtScl{\overline{\alpha}_{K-1}^r}{\Fhess \overline{\alpha}_{K-1}^r} \leq \frac{\gamma}{K} \sum_{k=1}^K \E \|\xikmultFN[\etakm^r]\|^2  \,.
\end{align*}
\end{lemma}
\end{minipage}
}

\begin{proof}
Let $r,k$ in $\N \times \N^*$, we denote $\alpha_k^r= \eta_k - \sum_{i=0}^{r} \eta_k^i$, then we have shown in \Cref{app:lem:recurrence_from_bm2013} that: \begin{align*}
    \alpha_k^r= (\Id_d - \gamma \Fhess) \alpha_{k-1}^r+ \xikmultFN[\alpha_{k-1}^r] + \gamma \xikmultFN[\etakm^r] \,.
\end{align*}
Taking the squared norm and developing it:
\begin{align*}
    \SqrdNrm{\alpha_k^r} &= \SqrdNrm{\alpha_{k-1}^r} + 2 \gamma \PdtScl{\alpha_{k-1}^r}{\xikmultFN[\alpha_{k-1}^r] + \xikmultFN[\etakm^r] - \Fhess \alpha_{k-1}^r}  \\
    &\qquad+ \gamma^2 \sqrdnrm{\xikmultFN[\alpha_{k-1}^r] + \xikmultFN[\etakm^r] - \Fhess \alpha_{k-1}^r} \,,
\end{align*}

and developing the last term with \Cref{app:lem:two_inequalities} leads to:
\begin{align*}
   \SqrdNrm{\alpha_k^r} &\leq \SqrdNrm{\alpha_{k-1}^r} + 2 \gamma \PdtScl{\alpha_{k-1}^r}{\xikmultFN[\alpha_{k-1}^r] + \xikmultFN[\etakm^r] - \Fhess \alpha_{k-1}^r} \\
    &\qquad + 2 \gamma^2 \bigg\{\sqrdnrm{\xikmultFN[\etakm^r]} + \sqrdnrm{\Fhess \alpha_{k-1}^r- \xikmultFN[\alpha_{k-1}^r]} \bigg\} \,.
\end{align*}

Because $\alpha_{k-1}^r$ is $\Flast_{k-1}$-measurable and $\expec{\xikmultFN[\alpha_{k-1}^r]}{\Flast_{k-1}} = 0$ (expectation of $\xikmultFN[\cdot]$ is zero, see \Cref{def:add_mult_noise,def:class_of_algo}), taking expectation w.r.t. the $\sigma$-algebra $\Flast_{k-1}$, using \Cref{asu:main:bound_mult_noise_lin} and again \Cref{def:class_of_algo} gives:
\begin{align*}
    \expec{\sqrdnrm{\Fhess \alpha_{k-1}^r- \xikmultFN[\alpha_{k-1}^r][u_k]}}{\Flast_{k-1}} 
    &= \expec{\sqrdnrm{\Fhess \alpha_{k-1}^r}}{\Flast_{k-1}} \\
    &\qquad+ \expec{\sqrdnrm{ \xikmultFN[\alpha_{k-1}^r][u_k]}}{\Flast_{k-1}} \\
    &\leq ( \Ftrace^2 + \boundMult) \sqrdnrm{\Fhess^{1/2} \alpha_{k-1}^r} \,.
\end{align*}
Hence:
\begin{align*}
    \expec{\sqrdnrm{\alpha_k^r}}{\Flast_{k-1}} &\leq \sqrdnrm{\alpha_{k-1}^r} - 2 \gamma (1 - \gamma ( \Ftrace^2 + \boundMult))  \PdtScl{\alpha_{k-1}^r}{\Fhess \alpha_{k-1}^r}  \\
    &\qquad+ 2 \gamma^2 \expec{\sqrdnrm{\xikmultFN[\etakm^r]}}{\Flast_{k-1}}\,,
\end{align*}
which gives when taking full expectation and averaging over $K$ in $\N^*$:
\begin{align*}
    (1- \gamma ( \Ftrace^2 + \boundMult)) \frac{1}{K} \sum_{k=1}^K \E \PdtScl{\alpha_{k-1}^r}{\Fhess \alpha_{k-1}^r} &\leq \frac{1}{2 \gamma}(\SqrdNrm{\alpha_0^r} - \SqrdNrm{\alpha_{k-1}^r}) \\
    &\qquad+ \frac{\gamma}{K} \sum_{k=1}^K \fullexpec{\sqrdnrm{\xikmultFN[\etakm^r]} }\,,
\end{align*}
and by convexity $\left\langle\overline{\alpha}_{K-1}^r, H \overline{\alpha}_{K-1}^r\right\rangle \leqslant \frac{1}{K} \sum_{k=1}^{K}\left\langle\alpha_{k-1}^r, \Fhess \alpha_{k-1}^r\right\rangle$, which allows to conclude as $\alpha_0^r = 0$.
\end{proof}

In below lemma, we bound $\FullExpec{\etakm^r \otimes \etakm^r}$ for $r, k$ in $\N \times \N^*$. It is required because we will use Lemma 2 from \cite{bach_non-strongly-convex_2013} and apply it to the sequence $(\etakm^r)_{k \in \N^*, r\in \N}$. The noise process of this sequence is equal to $\xikmultFN[\etakm^{r-1}]$; and computing the expectation of its covariance involves knowing $\FullExpec{\etakm^r \otimes \etakm^r}$.

\fbox{
\begin{minipage}{0.97\textwidth}
\begin{lemma}[Bounding the covariance of $\etakm^r$]
\label[lemma]{app:lem:bound_cov_eta_from_bm2013} 
Under the setting in \Cref{def:class_of_algo}, under \Cref{asu:main:bound_mult_noise_lin,asu:main:baniac_lin,asu:main:bound_add_noise}, i.e. considering that $\xikmultFN[\cdot]$ is linear, for any $K$ in $\N^*$, any step-size $\gamma>0$,
and for any $r \geq 0$, we have the following bound on the covariance of $\etakm^r$:
\begin{align*}
\FullExpec{\etakm^r \otimes \etakm^r} \preccurlyeq \gamma^{r+1} \ShaA \ShaM^{r} \Id_d \,.
\end{align*}
\end{lemma}
\end{minipage}
}

\begin{proof}

\textbf{Let $r > 0$, we first prove by recursion that we have:}
\[
\forall k > 0\,,\eta_k^{r+1} = \gamma \sum_{i=1}^k (\Id_d - \gamma \Fhess) ^{k-i} \xikmultiplicatif_i(\eta_{i-1}^r)\,.
\]

For $k=0$, we indeed have $\eta^{r+1}_0 = 0$.
To go from $k$ to $k+1$:
\begin{align*}
    \eta_{k+1}^{r+1} &= (\Id_d -\gamma \Fhess) \eta^{r+1}_{k} + \gamma \xikmultiplicatif_{k+1}(\eta_k^{r}) \qquad \text{by definition,}\\
    &= \gamma \sum_{i=1}^k (\Id_d - \gamma \Fhess) ^{k-i} \xikmultiplicatif_i(\eta_{i-1}^r) + \gamma (\Id_d - \gamma \Fhess)^{(k+1) - (k+1)} \xi_{k+1}^{\mathrm{mult}}(\eta_k^{r}) \,,
\end{align*}
by hypothesis, which allows concluding.

\textbf{We now prove by recursion the main result of the lemma.}

\textit{Initialization.}
For $r=0$, by definition, we have $\eta_k^0 = (\Id_d - \gamma \Fhess) \etakm^0 + \gamma \xikstaruk$, unrolling the sum gives $
\eta_k^0 = (\Id_d - \gamma \Fhess)^{k} \eta_{0}^0 + \gamma \sum_{i=1}^k (\Id_d - \gamma \Fhess)^{k - i}\xi_i^{\mathrm{add}}$. Because we consider $\eta_0^0 =0$ and given that the sequence of noise $(\xi_i^{\mathrm{add}})_{i \in \llbracket 1, k \rrbracket}$ is independent at each iterations, we have:
\begin{align*}
    \FullExpec{\eta_k^0 \otimes \eta_k^0} &= \gamma^2 \sum_{i=1}^k (\Id_d - \gamma \Fhess)^{k - i} \FullExpec{\xi_i^{\mathrm{add}}\otimes \xi_i^{\mathrm{add}}} (\Id_d - \gamma \Fhess)^{k - i} \,.
\end{align*}
Because the sequence of additive noise $(\xi_i^{\mathrm{add}})_{i\in \N^*}$is i.i.d., for any $i$ in $\{1, \cdots, k\}$, we have that $\FullExpec{\xi_i^{\mathrm{add}}\otimes \xi_i^{\mathrm{add}}} = \aniac \preccurlyeq \ShaA \Fhess$ (\Cref{asu:main:bound_cov_add_noise_lin}), hence: 
\begin{align*}
    \FullExpec{\eta_k^0 \otimes \eta_k^0} &\preccurlyeq \gamma^2 \sum_{i=1}^k (\Id_d - \gamma \Fhess)^{k - i} \ShaA \Fhess (\Id_d - \gamma \Fhess)^{k - i} \,. 
\end{align*}
These matrices commute:
\begin{align*}
    \FullExpec{\eta_k^0 \otimes \eta_k^0} &\preccurlyeq \gamma^2 \ShaA \sum_{i=1}^k (\Id_d - \gamma \Fhess)^{2k - 2i}  \Fhess \,,~\text{and because it is a geometric sum:}\\
    &\preccurlyeq \gamma^2  \ShaA  \bigpar{\Id_d - (\Id_d - \gamma \Fhess)^{2k - 2}}\bigpar{\Id_d - (\Id_d - \gamma \Fhess)^2}^{-1} \Fhess  \\
    &\preccurlyeq \gamma^2 \ShaA \bigpar{\Id_d - (\Id_d - \gamma \Fhess)^{2k - 2}}\bigpar{2 \gamma \Fhess - \gamma^{2} \Fhess^2}^{-1} \Fhess  \\
    &\preccurlyeq \gamma \ShaA  \Fhess^{-1}  \Fhess\qquad  \text{because $\gamma \Fhess \preccurlyeq \Id_d$,} \\
    &\preccurlyeq \gamma \ShaA \Id_d \,.
\end{align*}

\textbf{Recursion.} Let $r\geq 0$, to go from $r$ to $r+1$, we start writing:
\begin{align*}
    \eta_k^{r+1} \otimes \eta_k^{r+1} &=   \gamma^2 \sum_{i=1}^{k} (\Id_d - \gamma \Fhess)^{k-1-i} \xikmultiplicatif_i(\eta_{i-1}^{r}) \otimes \xikmultiplicatif_i(\eta_{i-1}^{r})  (\Id_d - \gamma \Fhess)^{k-1-i} \,.
\end{align*}

Now we use linearity of the multiplicative noise (\Cref{asu:main:bound_mult_noise_lin}), thus there exists a matrix $\Xi_k$ in $\R^{d\times d}$ s.t. for any $z$ in~$\R^d$, we have $\xi_k^\mathrm{mult}(z) = \Xi_k z$, and it leads to:
\begin{align*}
    \eta_k^{r+1} \otimes \eta_k^{r+1} &= \gamma^2 \sum_{i=1}^{k} (\Id_d - \gamma \Fhess)^{k-i} \Xi_i(\eta_{i-1}^{r} \otimes \eta_{i-1}^{r}) \Xi_i^\top  (\Id_d - \gamma \Fhess)^{k-i} \,.
\end{align*}

Taking full expectation, we have:
\begin{align*}
    \FullExpec{\eta_k^{r+1} \otimes \eta_k^{r+1}} &= \gamma^2 \sum_{i=1}^{k} (\Id_d - \gamma \Fhess)^{k-i} \FullExpec{\Expec{\Xi_i (\eta_{i-1}^{r} \otimes \eta_{i-1}^{r}) \Xi_i^\top }{\sigma(\Xi_i)}} (\Id_d - \gamma \Fhess)^{k-i} \\
    &= \gamma^2 \sum_{i=1}^{k} (\Id_d - \gamma \Fhess)^{k-i} \FullExpec{\Xi_i \expec{\eta_{i-1}^{r} \otimes \eta_{i-1}^{r}}{\sigma(\Xi_i)} \Xi_i^\top } (\Id_d - \gamma \Fhess)^{k-i} \,,
\end{align*}
and because for any $i$ in $\{1, \cdots, k\}$, $\eta_{i-1}^{r}$ is independent of $\Xi_i $, we have $\Expec{\eta_{i-1}^{r} \otimes \eta_{i-1}^{r}}{\sigma(\Xi_i)} =  \FullExpec{\eta_{i-1}^{r} \otimes \eta_{i-1}^{r}} \preccurlyeq \gamma^{r+1} \ShaA \ShaM^{r} \Id_d $, where we use the hypothesis for $r$.
We have in the end:
\begin{align*}
    \FullExpec{\eta_k^{r+1} \otimes \eta_k^{r+1}}
    &\preccurlyeq  \gamma^{r + 3} \ShaA \ShaM^{r} \sum_{i=1}^{k} (\Id_d - \gamma \Fhess)^{k-i} \FullExpec{\Xi_i \Xi_i^\top} (\Id_d - \gamma \Fhess)^{k-i} \,.
\end{align*}

Furthermore, by \Cref{asu:main:bound_cov_mult_noise_lin} we have $\FullExpec{\Xi_i \Xi_i^\top} \preccurlyeq \ShaM \Fhess$, thus: 
\begin{align*}
    \FullExpec{\eta_k^{r+1} \otimes \eta_k^{r+1}} &\preccurlyeq \gamma^{r + 3} \ShaA \ShaM^{r+1} \sum_{i=1}^{k} (\Id_d - \gamma \Fhess)^{2k-2-2i} \Fhess \\
    &\preccurlyeq \gamma^{r + 3} \ShaA \ShaM^{r+1} \gamma^{-1} \Fhess^{-1} \Fhess \,,
\end{align*}
because $\sum_{i=1}^{k} (\Id_d - \gamma \Fhess)^{2k-2-2i} = \bigpar{\Id_d - (\Id_d - \gamma \Fhess)^{2k}} \bigpar{2  \gamma \Fhess - \gamma^2 \Fhess^2}^{-1} \preccurlyeq \gamma^{-1} \Fhess^{-1}$. 
In the end, we have $\fullexpec{\eta_k^{r+1} \otimes \eta_k^{r+1}} \preccurlyeq \gamma^{r+2} \ShaA \ShaM^{r+1} \Id_d$,
which concludes the proof.

\end{proof}

\subsection{Final theorem}
\label{app:subsec:final_thm_bm2013_with_linear_comp}

In this section, we gather the pieces of proof required to demonstrate \Cref{thm:bm2013_with_linear_operator_compression}. As done in \Cref{app:sec:nonlinear_bm}, we consider separately the noise process and the noise-free process, then put them together to obtain the final result.

\fbox{
\begin{minipage}{0.97\textwidth}
\begin{theorem}[Linear multiplicative noise, convex case]
\label{app:thm:bm2013_with_linear_operator_compression}
Under \Cref{asu:main:bound_add_noise}, under \Cref{asu:main:bound_mult_noise_lin,asu:main:baniac_lin} i.e. with a linear multiplicative noise, considering any constant step-size $\gamma$ such that $\gamma ( \Ftrace^2 + \boundMult) \leq 1$ and $4 \gamma \ShaM \Ftrace^2 \leq 1$, then for any $K$ in $\N^*$, the sequence $(w_k)_{k\in\N^*}$ produced by a setting such as in \Cref{def:class_of_algo}, verifies the following bound:
\begin{align*}
    \E[ F(\overline{w}_{K-1}) - F(\ws) ]\leq \frac{1}{2K} \bigpar{\frac{\|\eta_0 \|}{\sqrt{\gamma}} + \sqrt{\Tr{\aniac \Fhess^{-1}}} + \ffrac{ \bigpar{ \gamma d  \ShaA \ShaM}^{1/2}}{1-\sqrt{\gamma \ShaM}}}^2 \,.
\end{align*}
\end{theorem}
\end{minipage}
}

\begin{proof}
Let $K$ in $\N^*$, the proof relies on the proof presented by \cite{bach_non-strongly-convex_2013} and is done separately for the noise process and for the noise-free process that depends only on the initial condition. The bounds may then be added (see the discussion in \Cref{app:subsec:proof_principale_lin}).

\textbf{Noise-free process.}
As in section A.3 from \cite{bach_non-strongly-convex_2013}, we assume here that the additive noise $(\xi_k^{\mathrm{add}})_{k\in \N^*}$ is uniformly equal to zero and that $\gamma ( \Ftrace^2 + \boundMult) \leq 1$. Using \Cref{def:add_mult_noise,def:class_of_algo}, we thus have for any $k$ in $\N^*$ that $\eta_k = \etakm - \gamma \Fhess \etakm + \gamma \xi_k^{\mathrm{mult}}(\etakm)$, it flows:
\begin{align*}
\fullexpec{\sqrdnrm{\eta_k}} &= \fullexpec{\sqrdnrm{\etakm}} - 2 \gamma\E[ \PdtScl{\etakm}{\Fhess \etakm}] + \gamma^2\E [\|\Fhess \etakm - \xi_k^{\mathrm{mult}}(\etakm)\|^2 ]\\
&= \fullexpec{\sqrdnrm{\etakm} }- 2 \gamma\E[ \PdtScl{\etakm}{\Fhess \etakm}] + \gamma^2\fullexpec{\sqrdnrm{\Fhess \etakm } }+ \gamma^2\E [\|\xi_k^{\mathrm{mult}}(\etakm)\|^2 ]\,.
\end{align*}
Considering that $\Fhess \preccurlyeq \Tr{\Fhess}\Id_d \preccurlyeq \Ftrace^2\Id_d$ and using \Cref{asu:main:bound_mult_noise_lin}, we obtain:
\begin{align*}
\fullexpec{\sqrdnrm{\eta_k}} &\leq \fullexpec{\sqrdnrm{\etakm}} - 2 \gamma\E[\|\Fhess^{1/2} \etakm \|^2] +\gamma^2( \Ftrace^2 + \boundMult) \E[\|\Fhess^{1/2} \etakm \|^2] \,.
\end{align*}
Because the step-size $\gamma$ is s.t. $\gamma ( \Ftrace^2 + \boundMult) \leq 1$, 
we recover that in the absence of noise, we have:
\begin{align}
\label{app:eq:bound_without_noise_linear}
\fullexpec{\sqrdnrm{\Fhess^{1/2} \etabarkm}} \leq \frac{\SqrdNrm{\eta_0}}{\gamma K} \,.
\end{align}

\textbf{Noise process.}
Now, all the following results comes from \Cref{app:subsec:noise_process} where we assume that $\eta_0 = w_0 - \ws = 0$, we start using Minkowski's inequality \ref{app:eq:minkowski}:
\begin{align}
\label{app:eq:part_noise_minkowski_linear_thm}
    \FullExpec{\|\Fhess^{1/2} \etabarkm\|^2}^{1/2} \leq \FullExpec{\|\Fhess^{1/2} \sum_{i=0}^r \etabarkm^i\|^2}^{1/2} + \FullExpec{\|\Fhess^{1/2}(\etabarkm - \sum_{i=0}^r \etabarkm^i)\|^2}^{1/2}\,.
\end{align}

\textit{First term.}

Let $r \in \N$, again using Minkowski's inequality \ref{app:eq:minkowski}, we have 
\begin{align}
\label{app:eq:split_ania_sha_shaprime}
\E [\|\Fhess^{1/2} \sum_{i=0}^r \etabarkm^i\|^2]^{1/2} &\leq \sum_{i=0}^r  \E [ \|\Fhess^{1/2} \etabarkm^i\|^2 ]^{1/2} \nonumber \\
&= \E [ \|\Fhess^{1/2} \etabarkm^0\|^2 ]^{1/2} + \sum_{i=1}^r  \E [ \|\Fhess^{1/2} \etabarkm^i\|^2 ]^{1/2}\,.
\end{align}

By \Cref{app:eq:double_sequence_eta}, we have $\eta_k^0 =  (\Id_d - \gamma \Fhess) \etakm^0 + \gamma \xikstaruk$, hence to bound the first term, we have to apply Lemma 2 from \citet{bach_non-strongly-convex_2013} to the sequence $(\etakm^0)_{k\in\N^*}$ and we obtain 
\begin{align}
    \E [ \|\Fhess^{1/2} \etabarkm^0\|^2 ] \leq \Tr{\aniac \Fhess^{-1}} / K \,.
\end{align}

Let $i$ in $\{1, \cdots, r\}$, to bound the second term, we have to apply Lemma 2 from \citet{bach_non-strongly-convex_2013} to the sequence $(\etakm^i)_{k \in\N^*}$.
To do so, we bound the covariance of the noise which is here equal to  $\xikmultFN[\etakm^{i-1}]$ (by definition of $\etakm^i$, see \Cref{app:eq:double_sequence_eta}).

Because the multiplicative noise is linear, using \Cref{asu:main:bound_mult_noise_lin}, there exists a matrix $\Xi_k$ in $\R^{d\times d}$ s.t. $\xikmultFN[\etakm^{i-1}] = \Xi_k\etakm^{i-1}$.
It follows that taking the expectation w.r.t to the $\sigma$-algebra $\sigma(\Xi_k)$, and because $\eta_{k-1}^{i-1}$ is independent of it, using \Cref{app:lem:bound_cov_eta_from_bm2013}, we have: 
\[\Expec{\eta_{k-1}^{i-1} \otimes \eta_{k-1}^{i-1}}{\sigma(\Xi_k)} =  \FullExpec{\eta_{k-1}^{i-1} \otimes \eta_{k-1}^{i-1}} \preccurlyeq \gamma^{i} \ShaA \ShaM^{i-1} \Id_d \,.
\]

Thus, the noise $\xikmultFN[\etakm^{i-1}]$ is such that:
\begin{align*}
    \expec{\xikmultFN[\etakm^{i-1}] \otimes \xikmultFN[\etakm^{i-1}]}{\sigma(\Xi_k)} &= \Xi_k \FullExpec{\etakm^{i-1} \otimes \etakm^{i-1}} \Xi_k^\top \preccurlyeq \gamma^{i} \ShaA \ShaM^{i-1} \Xi_k \Xi_k^\top \,.
\end{align*}
Taking full expectation, we furthermore consider \Cref{asu:main:bound_cov_mult_noise_lin} which gives that:
$\FullExpec{\Xi_i \Xi_i^\top} \preccurlyeq  \ShaM \Fhess$, 
hence:
\begin{align}
\label{app:eq:bound_mult_noise_cov_i}
    \FullExpec{\xikmultFN[\etakm^{i-1}] \otimes  \xikmultFN[\etakm^{i-1}]} &\leq \gamma^i \ShaA \ShaM^{i} \Fhess\,.
\end{align}
Using Lemma 2 from \citet{bach_non-strongly-convex_2013} results to:
\begin{align}
\label{app:eq:lemma2_from_bm2013}
\sum_{i=1}^r  \E [ \|\Fhess^{1/2} \etabarkm^i\|^2 ]^{1/2} \leq \sum_{i=1}^r  \gamma^i \ShaA \ShaM^i  \Tr{\Fhess \Fhess^{-1} }/ K \,.
\end{align}

In the end, we obtain from \Cref{app:eq:split_ania_sha_shaprime}:
\begin{align*}
\E [\|\Fhess^{1/2} \sum_{i=0}^r \etabarkm^i\|^2]^{1/2} &\leq \ffrac{\sqrt{\Tr{\aniac \Fhess^{-1}}}}{\sqrt{K}} + \ffrac{\sqrt{d \ShaA }}{\sqrt{K}} \sum_{i=1}^r \gamma^{i/2} \ShaM^{i/2} \\
&\leq \ffrac{\sqrt{\Tr{\aniac \Fhess^{-1} } }}{\sqrt{K}} + \frac{\sqrt{\gamma d \ShaA \ShaM} \bigpar{1-(\gamma \ShaM)^{r/2}}}{\sqrt{K} \bigpar{1-\sqrt{\gamma \ShaM}}}\,.
\end{align*}

\textit{Second term.}

If $\gamma ( \Ftrace^2 + \boundMult) \leq 1$, \Cref{app:lem:bounding_rec} gives:
\begin{align}
\label{app:eq:from_LemmaS5_linear}
\E \PdtScl{\etabarkm - \sum_{i=0}^r \etabarkm^i}{H (\etabarkm - \sum_{i=0}^r \etabarkm^i)} \leq \frac{\gamma}{(1- \gamma ( \Ftrace^2 + \boundMult)) K} \sum_{k=1}^K \FullExpec{ \sqrdnrm{\xikmultFN[\etakm^r]}} \,.
\end{align}
Furthermore, $\SqrdNrm{\xikmultFN[\etakm^r]} =\Tr{\xikmultFN[\etakm^r]^{\kcarre}}$, by reusing what has been written in the previous paragraph (\Cref{app:eq:bound_mult_noise_cov_i}), we obtain:
\begin{align*}
    \sqrdnrm{\xikmultFN[\etakm^r]} &\leq \gamma^{r+1} \ShaA \ShaM^{r+1} \Tr{\Fhess} \\
    &\leq \gamma^{r+1} \ShaA \ShaM^{r+1} \Ftrace^{2} \qquad \text{(\Cref{def:class_of_algo}).} 
\end{align*}
It follows that we have:
\begin{align}
\label{app:eq:bounds_with_sha_linear}
\E \PdtScl{\etabarkm - \sum_{i=0}^r \etabarkm^i}{H (\etabarkm - \sum_{i=0}^r \etabarkm^i)} \leq \frac{\gamma^{r+2} \ShaA \ShaM^{r+1} \Ftrace^{2}}{(1- \gamma ( \Ftrace^2 + \boundMult))} \,. 
\end{align}

\textit{Putting things together.}
In the end, from the Minkowski decomposition done in \Cref{app:eq:part_noise_minkowski_linear_thm}, we combine the two terms and it leads to:
\begin{align*}
    \FullExpec{\PdtScl{\etabarkm}{\Fhess \etabarkm}}^{1/2} &\leq \bigpar{\frac{\gamma^{r+2} \ShaA \ShaM^{r+1} \Ftrace^{2}}{(1- \gamma ( \Ftrace^2 + \boundMult))}  }^{1/2} + \ffrac{ \sqrt{\Tr{\aniac \Fhess^{-1} }}}{\sqrt{K}}  \\
    &\qquad + \frac{\sqrt{\gamma d \ShaA \ShaM} \bigpar{1-(\gamma \ShaM)^{r/2}}}{\sqrt{K} \bigpar{1-\sqrt{\gamma \ShaM}}}\,.
\end{align*}
This implies that for any $\gamma \ShaM \leq 1$, we obtain, by letting $r$ tend to $+\infty$:
\begin{align}
\label{app:eq:r_tends_to_infty}
    \FullExpec{\PdtScl{\etabarkm}{\Fhess \etabarkm}}^{1/2} &\leq \ffrac{1}{\sqrt{K}} \bigpar{ \sqrt{\Tr{\aniac \Fhess^{-1} }} + \ffrac{ \bigpar{ \gamma d \ShaA \ShaM}^{1/2}}{1-\sqrt{\gamma \ShaM}} } \,.
\end{align}

\textbf{Final bound.}
We now take results derived from the part without noise, and the part with noise, to get:
\begin{align*}
    \E [\PdtScl{\etabarkm}{\Fhess \etabarkm}]^{1/2} \leq \frac{1}{\sqrt{K}} \bigpar{\frac{\|\eta_0 \|}{\sqrt{\gamma}} + \sqrt{\Tr{\aniac \Fhess^{-1}}} + \ffrac{ \bigpar{ \gamma d \ShaA \ShaM}^{1/2}}{1-\sqrt{\gamma \ShaM}}}\,, 
\end{align*}
which leads to the desired result considering that $4 \gamma \ShaM \leq 1$. 
\end{proof}

\section{Validity of the assumptions made on the random fields}
\label{app:subsec:validity_asu_random_fields}

In this section, we verify that all the assumptions on the random fields done in \Cref{subsec:def_ania_asu_field} are fulfilled in the setting of compressed least-squares regression analyzed in \Cref{sec:application_compressed_LSR}.
To do so, we first need to define the filtrations considered in this section.

For $k$ in $\N^*$, we note $u_k$ the noise that controls the randomization $\C_k(\cdot)$ at round $k$.
In \Cref{sec:theoretical_analysis}, we have denoted by $\Flast_k$ the $ \sigma$-algebra generated by $(x_1, \varepsilon_1, u_1, \cdots, x_k, \varepsilon_k)$. In particular, $w_k$ and $\overline{w}_k$ are $\Flast_k$-measurable.
We also consider the following filtrations.
\begin{definition}
\label[definition]{app:def:filtrations}
We note $(\Fx_k)_{k \in \N}$ the filtration associated with the features noise, $(\Fy_k)_{k \in \N}$ the filtration associated with the output noise, and $(\Fsto_k)_{k \in \N}$ the filtration associated with the stochastic gradient noise, which is the union of the two previous filtrations. Thus, we define $\Flast_{0} = \{ \varnothing \}$ and for $k \in \N^*$:
\begin{align*}
    \Fx_k &= \sigma \bigpar{\Flast_{k-1} \cup \{x_k\} } \\
    \Fy_k &= \sigma \bigpar{\Flast_{k-1} \cup \{\varepsilon_k\} } \\
    \Fsto_k &= \sigma \bigpar{\Flast_{k-1} \cup \{ x_k, \varepsilon_k\}} \\
    \Flast_k &= \sigma \bigpar{\Flast_{k-1} \cup \{x_k, \varepsilon_k, u_k\} } \,.
\end{align*}

Note that there are two filtrations $\Fx$ and $\Fy$ for the two independent noises that are both involved to compute the stochastic gradient. This will help us to compute the scalar product of these two quantities.  

\end{definition}

We start by providing a bound on the distance between two compressions, this lemma will be used to prove \Cref{prop:mult_noise}.

\fbox{
\begin{minipage}{0.97\textwidth}
\begin{lemma}
\label[lemma]{app:lem:dependent_compr}
For any compressor $\C$ in $\mathbb{C}$ verifying \Cref{lem:compressor}, for all $x, y$ in $\R^d$, we have:
$$\fullexpec{\SqrdNrm{\C(x) - \C(y)}} \leq 2(\omgC + 1) \SqrdNrm{x}  + 2 (\omgC + 1)\SqrdNrm{y} \,.$$ 
\end{lemma}
\end{minipage}
}

\begin{proof}
Let a compressor $\C$ in $\mathbb{C}$ and $x, y$ in $\R^d$, using \Cref{app:lem:two_inequalities}, we have that:
\begin{align*}
\SqrdNrm{\C(x) - \C(y)} &\leq 2\SqrdNrm{\C(x)} + 2\SqrdNrm{\C(y)} \,.
\end{align*}

Taking expectation and using \Cref{lem:compressor} allows to conclude:
\begin{align*}
\FullExpec{\SqrdNrm{\C(x) - \C(y)}} &\leq 2 (\omgC +1)  \SqrdNrm{x} + 2 (\omgC +1) \SqrdNrm{y} \,.
\end{align*}
\end{proof}

Now we prove that all the assumptions done in \Cref{sec:theoretical_analysis} are correct.

\begin{property}[Validity of the setting presented in \Cref{def:class_of_algo}]
\label{prop:zero_centered_noise}
Consider the \Cref{ex:cent_comp_LMS} in the context of \Cref{model:centralized}, we have that the setting presented in \Cref{def:class_of_algo} is verified.
\end{property}

\begin{proof}
From~\Cref{ex:cent_comp_LMS}, we have for any $k$ in $\N^*$ and any $w$ in $\R^d$ $\xi_k(w - \ws) = \nabla F(w) - \C_k(\g_k(w))$.
Because $(\g_k)_{k\in\N^*}$ and $(\C_k)_{k\in\N^*}$ are by definition two sequences of i.i.d. random fields (\Cref{ex:cent_comp_LMS}), it follows that their composition is also i.i.d., \emph{hence $(\xi_k)_{k\in\N^*}$ is a sequence of i.i.d. random fields.}

Taking expectation w.r.t. the $\sigma$-algebra $\Fsto_k$, we have $\Expec{\C_k(\g_k(w))}{\Fsto_k} = \g_k(w)$ (\Cref{lem:compressor}), next with the $\sigma$-algebra $\Flast_{k-1}$, we have $\Expec{\g_k(w)}{\Flast_{k-1}} = \nabla F(w)$ (Equation \ref{eq:def_oracle}).
\emph{Hence, the random fields are zero-centered.}

From \Cref{model:centralized}, we have for any $k$ in $\N^*$ and any $w$ in $\R^d$ that:
\begin{align*}
    F(w) &= \frac{1}{2} \FullExpec{(\PdtScl{x_k}{w} - y_k)^2} = \frac{1}{2} \FullExpec{(w - w_*)^\top (x_k \otimes x_k) (w - w_*) - 2 \varepsilon_k \PdtScl{x_k}{w - w_*} + \varepsilon_k^2} \\
    &=  \frac{1}{2} ((w - \ws)^\top \xCov (w - \ws) + \sigma^2)\,,
\end{align*}
hence $F$ is quadratic with Hessian equal to $\xCov$ whose trace is equal to $R^2$.
\end{proof}

\begin{property}[Validity of \Cref{asu:main:bound_add_noise}]
\label{prop:add_noise}
Considering \Cref{ex:cent_comp_LMS} under the setting of \Cref{model:centralized} with \Cref{lem:compressor}, for any iteration $k$ in $\N^*$, the second moment of the additive noise $\xi_k^{\mathrm{add}}$ can be bounded by~$(\omgC +1)  R^2 \sigma^2 $, i.e., \Cref{asu:main:bound_add_noise} is verified.
\end{property}

\begin{proof}
Let $k$ in $\N^*$. Because we consider \Cref{ex:cent_comp_LMS}, with \Cref{def:class_of_algo,def:add_mult_noise}, we first have~$\xi_k^{\mathrm{add}} = - \C_k(\gwkstar)$, then with \Cref{lem:compressor} we obtain $\expec{\SqrdNrm{\C_k(\gwkstar)}}{\Fsto_k} \leq (\omgC + 1) \SqrdNrm{\gwkstar}$.
Next, we first have from \Cref{model:centralized} and \Cref{eq:def_oracle} that $\gwkstar = \varepsilon_k x_k$, secondly because $\bigpar{(\varepsilon_k)_{k\in \OneToK}}$ is independent from $\bigpar{(x_k)_{k\in\OneToK}}$ (\Cref{model:centralized}), we have that $\E[\|\varepsilon_k x_k\|^2] \leq \sigma^2  R^2$, hence it results to:
\begin{align*}
    \expec{\|\xi_k^{\mathrm{add}}\|^2}{\Flast_{k-1}} = \E[\|\xi_k^{\mathrm{add}}\|^2] \leq (\omgC + 1) \sigma^2  R^2\,.
\end{align*}
\end{proof}
\begin{property}[Validity of \Cref{asu:main:bound_mult_noise}]
\label{prop:mult_noise}
Considering \Cref{ex:cent_comp_LMS}, under the setting of \Cref{model:centralized} with \Cref{lem:compressor}, for any iteration $k$ in $\N^*$, the second moment of the multiplicative noise $\xi_k^{\mathrm{mult}}(w)$ can be bounded for any $w$ in $\R^d$ by $2(\omgC +1) R^2 \SqrdNrm{\xCov^{1/2} (w - \ws)} + 4 (\omgC + 1)  \sigma^2 R^2  $, i.e., \Cref{asu:main:bound_mult_noise} is verified.
\end{property}

\begin{proof}
Let $k$ in $\N^*$, we note $\eta = w - \ws$. First, because we consider \Cref{ex:cent_comp_LMS}, with \Cref{def:class_of_algo,def:add_mult_noise}, we have $\xi_k(\eta) = \nabla F(w) - \C_k(\g_k(w))$ and $\xikstaru = - \C_k(\gwkstar)$, hence:
$$\xi^\mathrm{mult}_k(\eta) = \xi_k(\eta) - \xikstaru = \nabla F(w) - \C_k(\g_k(w)) + \C_k(\gwkstar)\,,$$ 
thus developing the squared-norm of $\xi^\mathrm{mult}_k(\eta)$ gives:
\begin{align*}
    \| \xi^\mathrm{mult}_k(\eta) \|^2 &= \SqrdNrm{\nabla F(w)} + 2 \PdtScl{\nabla F(w)}{\C_k(\gwkstar) - \C_k(\g_k(w))} + \SqrdNrm{\C_k(\gwkstar) - \C_k(\g_k(w))} \,.
\end{align*}

On the first side we have $\Expec{\Expec{\C_k(\gwkstar) - \C_k(\g_k(w))}{\Fsto_k}}{\Flast_{k-1}}= - \nabla F(w_{k-1})$.
On the second side, we use \Cref{app:lem:dependent_compr}; this allows us to write:
\begin{align*}
    \Expec{\SqrdNrm{\C_k(\gwkstar) - \C_k(\g_k(w))}}{\Fsto_k} &\leq 2(\omgC + 1) \SqrdNrm{\g_k(w)} + 2(\omgC + 1) \SqrdNrm{\gwkstar} \,.
\end{align*}
Note that this bound is far from being optimal when $\g_k(w)=\gwkstar$ or if $\C$ is the identity.
Next, we decompose as follows:
\begin{align*}
    \Expec{\SqrdNrm{\C_k(\gwkstar) - \C_k(\g_k(w))}}{\Fsto_k} &\leq 2(\omgC + 1) \SqrdNrm{\g_k(w) - \gwkstar} \\
    &\quad+ 4(\omgC + 1) \PdtScl{\g_k(w) - \gwkstar}{\gwkstar} + 4(\omgC + 1) \SqrdNrm{\gwkstar}\,.
\end{align*}
Taking expectation w.r.t. the $\sigma$-algebra $\Fx_k$, recalling that $\g_k(w) - \gwkstar$ is $\Fx_k$-measurable (\Cref{app:def:filtrations}) and considering \Cref{model:centralized} allows to write:
\begin{align*}
\Expec{\SqrdNrm{\C_k(\gwkstar) - \C_k(\g_k(w)}}{\Fx_k} &\leq 2(\omgC + 1) \SqrdNrm{\gwkstar - \g_k(w)} + 4(\omgC+1) \sigma^2  R^2 \\
&\leq 2 (\omgC + 1) \SqrdNrm{(x_k \otimes x_k) \etakm} + 4(\omgC + 1) \sigma^2 R^2 \,,
\end{align*}
and now taking expectation w.r.t the $\sigma$-algebra $\Flast_{k-1}$ concludes the proof:
\begin{align*}
\expec{\SqrdNrm{\C_k(\gwkstar) - \C_k(\g_k(w))}}{\Flast_{k-1}} \leq 2(\omgC + 1)  R^2 \sqrdnrm{\xCov^{1/2} (w_k - \ws)} + 4(\omgC + 1) \sigma^2  R^2\,.
\end{align*}
\end{proof}

\begin{property}[Validity of \Cref{asu:main:bound_mult_noise_holder}]
\label{prop:mult_noise_holder}
Considering \Cref{ex:cent_comp_LMS}, under the setting of \Cref{model:centralized} with \Cref{lem:compressor}, for any iteration $k$ in $\N^*$, the second moment of the multiplicative noise $\xi_k^{\mathrm{mult}}(w)$ can be bounded for any $w$ in $\R^d$ by $\omgCOne R^2 \sigma\|\xCov^{1/2} (w - \ws) \| + 3 \omgCTwo  R^2 \|\xCov^{1/2} (w - \ws) \|^2$, i.e. \Cref{asu:main:bound_mult_noise_holder} is verified.
\end{property}

\begin{proof}
Let $k$ in $\N^*$, we note $\eta = w - \ws$. Because we consider \Cref{ex:cent_comp_LMS}, with \Cref{def:class_of_algo,def:add_mult_noise}, we have the following decomposition:
\begin{align*}
    \xikmultFN[\eta] &= \SqrdNrm{\nabla F(w)} + 2 \PdtScl{\nabla F(w)}{\C_k(\gwkstar) - \C_k(\g_k(w))} + \SqrdNrm{\C_k(\gwkstar) - \C_k(\g_k(w))} \,.
\end{align*}

We take expectation w.r.t. the $\sigma$-algebra $\Fsto_k$ and use \Cref{item:holder_compressor} of \Cref{lem:compressor}:
\begin{align*}
    \Expec{\xikmultFN[\eta]}{\Fsto_{k}} &\leq \SqrdNrm{\nabla F(w)} + 2 \PdtScl{\nabla F(w)}{\gwkstar - \g_k(w)} \\
    &\qquad+ \omgCOne \min(\|\gwkstar \|,  \| \g_k(w) \|) \| \gwkstar - \g_k(w) \| + 3 \omgCTwo \| \gwkstar - \g_k(w) \|^2 \,.
\end{align*}

Then, we have $\min(\|\gwkstar \|,  \| \g_k(w) \|) \| \gwkstar - \g_k(w) \| \leq \|\gwkstar \| \| \gwkstar - \g_k(w) \|$, taking expectation conditionally to the $\sigma$-algebra $\Flast_{k-1}$, applying the Cauchy-Schwarz's \Cref{app:basic_ineq:cauchy_schwarz_cond} and considering \Cref{model:centralized}, we have:
\begin{align*}
\expec{\|\gwkstar \| \| \gwkstar - \g_k(w) \|}{\Flast_{k-1}}^2 &\leq \expec{\|\gwkstar \|^2 }{\Flast_{k-1}} \expec{ \| \gwkstar - \g_k(w) \|^2}{\Flast_{k-1}} \\
&\leq \sigma^2 R^4 \|\xCov^{1/2} (w - \ws) \|^2\,.
\end{align*}
Therefore, we can conclude:
\begin{align*}
    \Expec{\xikmultFN[\eta]}{\Flast_{k-1}} &\leq - \SqrdNrm{\nabla F(w)} + \sigma R^2 \omgCOne \|\xCov^{1/2} (w - \ws) \| + 3\omgCTwo  R^2 \|\xCov^{1/2} (w - \ws) \|^2 \,.
\end{align*}
\end{proof}

\begin{property}[Validity of \Cref{asu:main:bound_mult_noise_lin}]
\label{prop:mult_noise_lin}
Considering \Cref{ex:cent_comp_LMS}, under the setting of \Cref{model:centralized} with \Cref{lem:compressor}, if the compressor $\C$ is linear, then for any iteration $k$ in $\N^*$, the multiplicative noise $\xi_k^{\mathrm{mult}}$ is linear, thus there exist a matrix $\Xi_k$ in $\R^{d\times d}$ such that for any $w$ in $\R^d$, $\xi_k^{\mathrm{mult}}(w) = \Xi_k w$. Furthermore, the second moment of the multiplicative noise can be bounded for any $w$ in $\R^d$ by $(\omgC +1)  R^2 \SqrdNrm{\xCov^{1/2} (w - \ws)}$, hence \Cref{asu:main:bound_mult_noise_lin} is verified.
\end{property}

\begin{proof}
Let $k$ in $\N^*$, we note $\eta = w - \ws$. First, because we consider \Cref{ex:cent_comp_LMS}, with \Cref{def:class_of_algo,def:add_mult_noise}, we have $\xi_k(\eta) = \nabla F(w) - \C_k(\g_k(w))$ and $\xikstaru = - \C_k(\gwkstar)$, hence:
\begin{align*}
\xikmultFN[\eta] = \xi_k(\eta) - \xikstaru = \nabla F(w) - \C_k(\g_k(w)) + \C_k(\gwkstar)\,.
\end{align*}
Because the random mechanism $\C_k$ is linear, there exists a random matrix $\Pi_k$ in $\R^{d \times d}$  such that for any $z$ in $\R^d$, we have $\C_k(z) = \Pi_k z$, it follows that:
\begin{align*}
\xikmultFN[\eta] = \nabla F(w) + \C_k(\gwkstar - \g_k(w)) = (\xCov - \Pi_k (x_k \otimes x_k)) \eta\,.
\end{align*}

Hence, the first part of \Cref{asu:main:bound_mult_noise_lin} is verified with $\Xi_k = \xCov - \Pi_k (x_k \otimes x_k)$. Now, we compute the second moment of the multiplicative noise. We start by developing its squared norm:
\begin{align*}
    \sqrdnrm{\xikmultFN[\eta]} &= \SqrdNrm{\nabla F(w)} + 2 \PdtScl{\nabla F(w)}{\C_k(\gwkstar - \g_k(w))} + \SqrdNrm{\C_k(\gwkstar - \g_k(w))} \,.
\end{align*}

Taking expectation conditionally to the $\sigma$-algebra $\Fsto_k$, and using \Cref{lem:compressor} gives:
\begin{align*}
    \Expec{\|\xikmultFN[\eta])\|^2}{\Fsto_k} &= \SqrdNrm{\nabla F(w)} + 2 \PdtScl{\nabla F(w)}{\gwkstar - \g_k(w)} + (\omgC + 1)\SqrdNrm{\gwkstar - \g_k(w)} \,.
\end{align*}

Finally, with $\sigma$-algebra $\Flast_{k-1}$ and considering \Cref{model:centralized} we have:
\begin{align*}
    \Expec{\|\xikmultFN[\eta]\|^2}{\Flast_{k-1}} &= - \SqrdNrm{\nabla F(w)} + (\omgC + 1) R^2 \sqrdnrm{\xCov^{1/2} (w - \ws)} \,,
\end{align*}
which allows to conclude.
\end{proof}

\begin{property}[Validity of \Cref{asu:main:baniac_lin}]
\label{prop:baniac_lin}
Considering \Cref{ex:cent_comp_LMS} under the setting of \Cref{model:centralized} with \Cref{rem:as_bounded_features,lem:compressor}, if the compressor $\C$ is linear, then 
for any $k$ in $\N^*$, there exists a constant $\Sha_H>0 $ s.t. $\aniac \preccurlyeq \sigma^2 \Sha_\xCov \Fhess$ and $\FullExpec{\Xi_k \Xi_k^\top} \preccurlyeq  R^2 \Sha_\xCov \xCov$; \Cref{asu:main:baniac_lin} is thus verified.
\end{property}

\begin{proof}

Let $k$ in $\N^*$, we note $\eta = w - \ws$.
We first need to compute $\Sha_\xCov$ in $\R^d$ for each compressor $\C$ in $\{ \Cquant, \Cstabquant, \Crand, \Cspars, \Csketch$, $ \Cpp \}$, it comes from \Cref{prop:covariance_formula} which results having a constant $\Sha_\xCov$~s.t.: 
\begin{align}
\label{app:eq:bound_compressor_cov}
    \compCov[][][] = \E_{E \sim p_\xCov}[\C(E)^{\kcarre}] \preccurlyeq \Sha_\xCov \xCov \,.
\end{align} Indeed, $\Diag{\xCov}$ car be bounded by $\Tr{\xCov} \Id_d$, and then $\Id_d$ by $\mu^{-1}\xCov$. This constant $\Sha_\xCov$ can be computed from \Cref{prop:covariance_formula} for any compressor:
\begin{center}
    \begin{tabular}{lllll}
    \toprule
Compressor & $\Crandh$ & $\Cspars$ & $\Cpp$ & $\Csketch$ \\
\midrule
    $\Sha_\xCov$ & $\frac{h-1}{p(d-1)} + (1- \frac{h-1}{d-1}) \frac{\tau}{p}$&  $1 + \frac{(1- p) \tau}{p}$ &$\frac{1}{p}$ & $\frac{\alpha - \beta}{p} + \frac{\beta \tau}{p}$ \\ 
    $\Sha_\xCov$ (if $H$ diagonal) & $\frac{1}{p}$&  $\frac{1}{p}$  &$\frac{1}{p}$& $\frac{\alpha - \beta}{p} + \frac{\beta \tau}{p}$\\
    \bottomrule
\end{tabular}
\end{center}
Where $p = h/d$, $\tau = \Tr{\xCov} / \mu$, and for sketching $\alpha = \frac{\subdim +2}{d+2}$ and $\beta = \frac{d - \subdim}{(d - 1)(d+2)}$.

We now show that the two inequalities given in \Cref{asu:main:baniac_lin} are valid.

\textit{First inequality.} 

By \Cref{def:ania}, we have $\aniac = \FullExpec{\xikstaruk \otimes \xikstaruk} = \FullExpec{\C_k(\varepsilon_k x_k)^{\kcarre}}$, because $\bigpar{(\varepsilon_k)_{k\in \OneToK}}$ is independent from $\bigpar{(x_k)_{k\in\OneToK}}$ (\Cref{model:centralized}) and using compressor linearity and \Cref{app:eq:bound_compressor_cov}, it gives:~$
\aniac = \sigma^2 \FullExpec{\C_k(x_k)^{\kcarre}} = \sigma^2 \compCov[][][] \preccurlyeq \sigma^2 \Sha_\xCov \xCov \,.
$

\textit{Second inequality.}

Using \Cref{prop:baniac_lin}, because the compressor $\C$ is linear, there exists two matrices $\Pi_k, \Xi_k$ in $\R^{d\times d}$ s.t. for any $z$ in~$\R^d$, we have $\C_k(z) = \Pi_k z$ and $\xi_k^\mathrm{mult}(z) = \Xi_k z$, which gives that $\Xi_k = \xCov - \Pi_k (x_k \otimes x_k)$.
It follows that:
\begin{align*}
    \Xi_k \Xi_k^\top =\xCov \xCov^\top - \xCov \Pi_k (x_k \otimes x_k) - \Pi_k (x_k \otimes x_k) \xCov + \Pi_k (x_k \otimes x_k)  (x_k \otimes x_k) \Pi_k^\top \,.
\end{align*}

Given that the compression is unbiased (\Cref{lem:compressor}) we have $\Expec{\Pi_k}{\Fsto_k} = \Id_d$, hence:
\begin{align*}
    \Expec{\Xi_k \Xi_k^\top}{\Fsto_k} &= \xCov \xCov^\top - \xCov (x_k \otimes x_k) - (x_k \otimes x_k) \xCov + \Expec{\Pi_k (x_k \otimes x_k)  (x_k \otimes x_k) \Pi_k^\top}{\Fsto_k} \,,
\end{align*}
and now taking expectation w.r.t the $\sigma$-algebra $\Flast_{k-1}$:
\begin{align*}
    \Expec{\Xi_k \Xi_k^\top}{\Flast_{k-1}} &= - \xCov \xCov^\top + \Expec{\Pi_k (x_k \otimes x_k)  (x_k \otimes x_k) \Pi_k^\top}{\Flast_{k-1}} \,.
\end{align*}

In the end, we have that $\Expec{\Xi_k \Xi_k^\top}{\Flast_{k-1}} \preccurlyeq \Expec{\Pi_k (x_k \otimes x_k)  (x_k \otimes x_k) \Pi_k^\top}{\Flast_{k-1}}$,
and if we consider that the second moment of the features $(x_k)_{k\in \N^*}$ is almost surely bounded (\Cref{rem:as_bounded_features}), we obtain:
\begin{align}
\label{app:eq:using_as_bounded_features}
   \Expec{\Xi_k \Xi_k^\top}{\Flast_{k-1}} &\preccurlyeq  R^2 \Expec{\Pi_k (x_k \otimes x_k) \Pi_k^\top}{\Flast_{k-1}} 
   \preccurlyeq R^2 \Expec{\C_k(x_k)^{\kcarre}}{\Flast_{k-1}} \,.
\end{align}
Thus, using \Cref{app:eq:bound_compressor_cov}, we can state that $\Expec{\Xi_k \Xi_k^\top}{\Flast_{k-1}} \preccurlyeq R^2 \Sha_H \xCov$,
which concludes the second part of the verification of \Cref{asu:main:baniac_lin}.

\end{proof} 
\section{Compression operators}
\label{app:sec:covariance}

In this Section, we provide additional details about compression operators. First, we prove in \Cref{app:subsec:var_cov_compression} that \Cref{lem:compressor} hold and compute the compressor's covariance given in \Cref{prop:covariance_formula}. The specific computations for sketching are given separately in \Cref{app:subsec:sketching_cov} because they are more complex. Third, it allows to prove \Cref{prop:particular_cases,prop:particular_cases_quant} in \Cref{app:subsec:proof_particular_case}.  And finally, in \Cref{app:subsec:plot_cov_matrix}, we plot the covariance matrix induced by quantization and sparsification for \texttt{quantum} and \texttt{cifar-10}.

\subsection{Computation of the variance and covariance of the compression operators }
\label{app:subsec:var_cov_compression}

In this Subsection, we first prove \Cref{lem:compressor}. \Cref{item:urvb} is frequently established in the literature and corresponds to the worst-case assumption, see the introduction for references. On the other hand, \Cref{item:holder_compressor} is the Hölder-type bound, which is not used in the literature up to our knowledge. Next, we compute the compressors' covariances that have been given in \Cref{prop:covariance_formula}.

\begin{lemma}
\label[lemma]{app:lem:compressor}
For any compressor $\C \in \{ \Cquant, \Cstabquant, \Crandh, \Cspars, \Csketch, \Cpp \}$, there exists constants $\omgC, \omgCOne \in \R^*_+$, such that the random operator $\C$ satisfies the following properties for all $z, z'\in\R^d$.
\begin{enumerate}[label={\textbf{L.\arabic*:}},ref={L.\arabic*},noitemsep]
\item \label{app:item:urvb} $\E [\C(z)] = z$ and $\E [ \| \C(z) - z\|^2] \leq \omgC \|z\|^2$ (unbiasedness and variance relatively bounded),
\item \label{app:item:holder_compressor} $\E [ \| \C(z) - \C(z')\|^2] \leq  \omgCOne  \min(\|z\|, \|z'\|) \|z - z'\|+ 3 \omgCTwo \|z- z' \|^2 \text{(Hölder-type bound),}$
\end{enumerate}
with $\omgC = \sqrt{d}$ and $\omgCOne
 = 12 \sqrt{d}$ (resp.  $\omgC = (1-p)/p$ and $\omgCOne
 = 0$) for $\Cquant$ and $\Cstabquant$ (resp. $\Crandh, \Cspars$, $\Csketch, \Cpp$). 
\end{lemma}

\begin{proof}

\paragraph{Value of $\omgC$ (\Cref{item:urvb} of \Cref{lem:compressor}).}

For projection-based compressors, the proof is straightforward, for quantization-based, the proof can be found in \citet{alistarh_qsgd_2017}.

\paragraph{Value of $\omgCOne$ (\Cref{item:holder_compressor} of \Cref{lem:compressor}).}

For linear compressors, it is straightforward to obtain $\omgCOne = 0$.

For quantization, we take $x, y$ in $\R^d$, we note $(u_i)_{i=1}^d$ the vector controlling the randomness of compression, and we write $\Cquant(x) - \Cquant(y) = A + B + C$, with:
\begin{enumerate}[topsep=2pt,itemsep=1pt,leftmargin=0.5cm,noitemsep]
    \item $A := \|x \| \sign(x) \Bern(\frac{|x|}{\|x\|}) - \|x \| \sign(x) \Bern(\frac{|x|}{\|y\|})$
    \item $B := \|x \| \sign(x) \Bern(\frac{|x|}{\|y\|}) - \|x \| \sign(y) \Bern(\frac{|y|}{\|y\|})$
    \item $C := \|x \| \sign(y) \Bern(\frac{|y|}{\|y\|})-\|y \| \sign(y) \Bern(\frac{|y|}{\|y\|})$.
\end{enumerate}

We note $\| \cdot \|$ the 2-norm and $\| \cdot \|_1$ the 1-norm. By symmetry, we suppose that $\sqrdnrm{y} \geq \sqrdnrm{x}$.

\textit{First term.} We have $\sqrdnrm{A} = \sqrdnrm{x} \sum_{i = 1}^d (\mathbb{1}_{u_i \leq \frac{|x_i|}{\|x\|}} - \mathbb{1}_{u_i \leq \frac{|x_i|}{\|y\|}})^2 = \sqrdnrm{x} \sum_{i = 1}^d \mathbb{1}_{\frac{|x_i|}{\|y\|} \leq u_i \leq \frac{|x_i|}{\|x\|}}^2$ because $\sqrdnrm{y} \geq \sqrdnrm{x}$. Taking expectation, it gives $\fullexpec{\sqrdnrm{A}} = \sqrdnrm{x} \sum_{i = 1}^d \frac{|x_i|}{\|x\|} - \frac{|x_i|}{\|y\|} = \sqrdnrm{x} \| x \|_1 \frac{\|y\| - \|x\|}{\|y\|\|x\|}$. Now with triangular inequality, we have:
\[
\fullexpec{\sqrdnrm{A}} \leq \frac{\|x \|}{\|y\|} \| x \|_1 \|y - x\| \leq \| x \|_1 \|y - x\| \leq \sqrt{d} \| x \| \|y - x\| \,,
\]
and by symmetry $\fullexpec{\sqrdnrm{A}} \leq \sqrt{d} \min(\| x \|, \| y \|) \|y - x\|$.

\textit{Second term.}

We have $\sqrdnrm{B} = \|x\|^2  \sum_{i = 1}^d (\sign(x_i) \mathbb{1}_{u_i \leq \frac{|x_i|}{\|y\|}} - \sign(y_i) \mathbb{1}_{u_i \leq \frac{|y_i|}{\|y\|}} )^2 $. Let $i$ in $[d]$, if $\sign(x_i) = \sign(y_i)$, then:
\[
\FullExpec{\sqrdnrm{B}} = \|x\|^2  \sum_{i = 1}^d \FullExpec{\mathbb{1}_{\frac{\min(|x_i|, |y_i|)}{\|y\|} \leq u_i \leq \frac{\max(|x_i|, |y_i|)}{\|y\|}}^2} =  \frac{\|x\|^2}{\|y\|}  \sum_{i = 1}^d |y_i - x_i | \leq \| x\| \|x - y\|_ 1 \,. 
\]

If $\sign(x_i) \neq \sign(y_i)$, developping $(\sign(x_i) \mathbb{1}_{u_i \leq \frac{|x_i|}{\|y\|}} - \sign(y_i) \mathbb{1}_{u_i \leq \frac{|y_i|}{\|y\|}} )^2$, we have:
\begin{align*}
    \FullExpec{\| B \|^2} &= \|x\|^2  \sum_{i = 1}^d \frac{| x_i |}{\|y\|} + \frac{| y_i |}{\|y\|} - 2 \sign(x_i) \sign(y_i) \frac{\min(|x_i|, |y_i |)}{\|y\|} \\
    &= \frac{\|x\|^2}{\|y\|}  \sum_{i = 1}^d \max(|x_i|, |y_i |) + 3 \min(|x_i|, |y_i |) \,.
\end{align*}
Next, we have $\max(|x_i|, |y_i |) +  \min(|x_i|, |y_i |) = |x_i| + |y_i | \overset{\sign(x_i) \neq \sign(y_i)}{=} |x_i - y_i |$, which results to $\FullExpec{\| B \|^2} \leq 3 \frac{\|x\|^2}{\|y\|}  \sum_{i = 1}^d |y_i - x_i | \leq 3 \|x\| \| x -y \|_1  \leq 3\sqrt d \|x\| \| x -y \|$.

\textit{Third term.} We have $\sqrdnrm{C} = ( \|x \|- \|y \| )^2 \sum_{i = 1}^d \mathbb{1}_{u_i \leq \frac{|y_i|}{\|y\|}}^2$, taking expectation, it gives: 
\[
\fullexpec{\sqrdnrm{C}} = ( \|x \|- \|y \| )^2 \sum_{i = 1}^d \frac{|y_i|}{\|y\|} \leq  \sqrdnrm{x-y} \frac{\|y\|_1}{\|y\|} \leq \sqrt d \sqrdnrm{x-y}\,.
\]

\textit{Overall}, using \Cref{app:lem:two_inequalities}, we have:
\begin{align*}
    \fullexpec{\sqrdnrm{\Cquant(x) - \Cquant(y) }} &\leq 12 \sqrt{d} \min(\|x\|, \|y\|) \|x-y\|  + 3 \sqrt{d} \|x-y\|^2 \,,
\end{align*}
which allows to conclude as for $1$-quantization, we have $\omgC = \sqrt{d}$.
\end{proof}

We now compute the compressors' covariance given in \Cref{prop:covariance_formula,prop:covariance_formula_diagonal_case}. However, sketching requires more involved computations, they are provided in \Cref{app:subsec:sketching_cov}.

\fbox{
\begin{minipage}{0.97\textwidth}
\begin{proposition}[Structure of the compressor's covariance]
\label[proposition]{app:prop:covariance_formula}
The following formulas of compressors' covariance hold:
\begin{itemize}[topsep=2pt,itemsep=1pt,noitemsep]
    \item $\compCov[][\emptyset][M] = M$
    \item $\compCov[][q][M] \preccurlyeq M + \sqrt{\Tr M } \sqrt{\Diag{M}} - \Diag{M}$
    \item $\compCov[][s][M] = M + \frac{1 - p}{p} \Diag{M}$
    \item $\compCov[][\Phi][M] = \frac{1}{p} \bigpar{(\alpha - \beta) M + \beta \Tr{M} \Id_d}$ with $\alpha = \frac{\subdim +2}{d+2}$ and $\beta = \frac{d - \subdim}{(d - 1) (d + 2)}$
    \item $\compCov[][\mathrm{rd}h][M] = \frac{d(h-1)}{h (d-1)} M + \bigpar{\frac{d}{h} - \frac{d(h-1)}{h (d-1)}}  \Diag{M}$
    \item $\compCov[][\mathrm{PP}][M]  = \frac{1}{p} M$\,.
\end{itemize}
\end{proposition}
\end{minipage}
}

\begin{proof}

In this proof, we denote $\FF$ the $\sigma$-field generated by the random sampling of $E \sim p_M \in \mathcal{P}_M$, and $\GG$  the $\sigma$-field generated by the noise from the compression process.
Let $E \sim p_M \in \mathcal{P}_M$.

\textbf{Quantization.}
By definition, we have $\C_q(E) = \normeucl{E} \sign(E) \odot \chi$, with $\chi = \bigpar{\mathrm{Bern}(\frac{|E_i|}{\normeucl{E}})}_{i=1}^d$.
It follows that $\C_q(E)^\kcarre = \normeucl{E}^2 \sign(E)^\kcarre \odot \chi^\kcarre$.

Because:
\begin{align*}
\Expec{\chi^\kcarre}{\FF} = \left\{ 
    \begin{array}{l}
       \ffrac{|E_i |}{\normeucl{E}} \quad \text{if} \quad i=j \vspace{1em} \\ 
       \ffrac{|E_i |~|E_j |}{\normeucl{E}^2} \quad \text{else,} 
    \end{array}  \right.
\end{align*}
and considering that $\sign(E)^\kcarre = \begin{pmatrix}
1 &  & \sign(E_i) \sign(E_j) \\
& \ddots &  \\
\sign(E_i) \sign(E_j) &  & 1 
\end{pmatrix}$\,,
we have:
\begin{align*}
\Expec{\C_q(E)^\kcarre}{\FF} = \left\{ 
    \begin{array}{l}
       \normeucl{E} ~ |E_i | \quad \text{if} \quad i=j \,, \vspace{1em} \\ 
       E_i E_j \quad \text{else.} 
    \end{array}  \right.
\end{align*}

Taking the complete expectation gives:
\begin{align*}
\FullExpec{\C_q(E)^\kcarre} = \left\{ 
    \begin{array}{l}
       \FullExpec{\normeucl{E} ~ |E_i |} \quad \text{if} \quad i=j \vspace{1em} \\ 
       M_{ij} \quad \text{else.} 
    \end{array}  \right.
\end{align*}

Changing the diagonal to make appear $M$, we obtain:
\[
\FullExpec{\C_q(E)^\kcarre} = M + \FullExpec{\normeucl{E} \Diag{|E_i |}_{i=1}^d} - \FullExpec{\Diag{E_i^2}_{i=1}^d} \,.
\]

Furthermore, we first have that $\FullExpec{\Diag{E_i^2}_{i=1}^d} = \Diag{M}$ and secondly, by Cauchy-Schwarz \Cref{app:basic_ineq:cauchy_schwarz_cond} that:

\[
\FullExpec{\normeucl{E} \Diag{|E_i |}_{i=1}^d}^2 \preccurlyeq \FullExpec{\normeucl{E}^2 } \FullExpec{\Diag{E_i^2}_{i=1}^d} = \Tr{M} \Diag{M} \,,
\]

which finally gives $\FullExpec{\C_q(E)^\kcarre} \preccurlyeq M + \sqrt{\Tr{M}} \sqrt{\Diag{M}} - \Diag{M} \,.$

\textbf{Sparsification.}
By definition, we have $\C_s (E) = \frac{1}{p} B \odot E \in \R^d$, with $B \sim \bigpar{\text{Bern}(p)}_{i=1}^d$, thus $\C_s(E)^\kcarre = \frac{1}{p^2} B^\kcarre \odot E^\kcarre$.
Taking the expectation w.r.t. to the $\sigma$-filtration $\FF$, we have~ $\Expec{\C_s (E)^\kcarre}{\FF} = \frac{1}{p^2} P \odot E^\kcarre$ with~$P = \begin{pmatrix}
p &  & p^2 \\
& \ddots &  \\
p^2 &  & p 
\end{pmatrix}\,,$ because for all $i,j$ in $\llbracket 1, d \rrbracket$, we have $\Expec{B_i^2}{\FF} = p$ and $\Expec{B_i B_j}{\FF} = p^2$. This naturally gives: $\FullExpec{\C_s (E)^\kcarre} = \frac{1}{p^2} P \odot M$.

\textbf{Sketching.}
The proof is more complex and therefore is given separately, in \Cref{app:subsubsec:covariance_sketching}.

\textbf{Rand-$h$.}
By definition, we have $\Crandh (E) := \frac{d}{h} B(S) \odot E$ with $S\sim \mathrm{Unif}(\mathcal{P}_h([d]))$ and $B(S)_i = \mathbb{1}_{i \in S}$, thus $\Crandh(E)^\kcarre = \frac{1}{p^2} B^\kcarre \odot E^\kcarre$ ($p=h/d$).
We have that for any $i, j$ in $\{1, \dots, d\}$, $B_i$ and~$B_j$ are \textit{not} independent and that $B_i \sim \bigpar{\text{Bern}(p)}$, therefore we have that $\E[B_i^2] = p$ and that: $h^2 = \bigpar{\sum_{i=1}^d B_i}^2 = \sum_{i=1}^d B_i^2 + \sum_{i\neq j} B_i B_j$. Taking expectation, it gives $h^2 = h + d (d-1) \E[B_i B_j]$ i.e. $\E[B_i B_j] = \frac{h(h-1)}{d (d-1)}$. 
Taking the expectation w.r.t. to the $\sigma$-filtration $\FF$, we have : $$\Expec{\Crandh (E)^\kcarre}{\FF} = \frac{d(h-1)}{h (d-1)} E^{\kcarre} + \bigpar{\frac{d}{h} - \frac{d(h-1)}{h (d-1)}} \Diag{E^{\kcarre}}\,.$$

And taking full expectation allows conclusion.

\textbf{Partial Participation.} This result is straightforward.

\end{proof}

\subsection{Variance and covariance of sketching}
\label{app:subsec:sketching_cov}

In this Subsection, we compute the expectation, the variance, and the covariance of sketching. In \Cref{app:subsubsec:proof_principle_sketching}, we give the proof principle of our computation, in \Cref{app:subsubsec:expec_sketching}, we compute the expectation and the variance, and in \Cref{app:subsubsec:covariance_sketching}, we compute the covariance.

We thank Baptiste Goujaud (École polytechnique, CMAP) who greatly helped to prove the following. 

\subsubsection{Proof principle}
\label{app:subsubsec:proof_principle_sketching}

Let $y$ in $\R^d$ with $\SqrdNrm{y}=1$, and $x$ in $\R^d$. By \Cref{def:operators_compression}, for $\Phi$ in $\R^{h \times d}$, we have $\C_\Phi(x) = \frac{1}{p} \Phi^\dagger \Phi x$ with $\Phi^\dagger = \Phi^\top (\Phi \Phi^T)^{-1}$ and $p = \subdim / d$.

To compute the expectation, the variance, and the covariance of $\C_\Phi(x)$, the idea is to compute~$\E[y^\top C_\Phi(x)]$ and $\E[(y^\top C_\Phi(x))^2]$ by establishing \Cref{app:eq:formula_sketching} which allows controlling the randomness of sketching by using \Cref{app:eq:sketching_equalities_expectation}.
To establish \Cref{app:eq:formula_sketching}, first observe that~$p C_\Phi(\cdots)$ is a projector into a subspace of dimension $\subdim$, indeed we have $(p C_\Phi \odot p C_\Phi)(x) = p C_\Phi (x)$. Then there exists a random matrix $P$ in $\mathcal{O}_d$ s.t. $p C_\Phi(x) = P^\top J_{\subdim} Px$. It leads to:
\begin{align*}
y^\top C_\Phi(x) &=\frac{1}{p}y^\top P^\top J_{\subdim} P x = \frac{1}{p} (Py)^\top J_{\subdim} (Px) \,.   
\end{align*}

Now we note $X = Px/\|x\|$ and $Y = P y$, hence $y^\top C_\Phi(x) = \frac{\|x\|}{p} Y^\top J_{\subdim} X$, and because $P$ is in $\mathcal{O}_d$, we have:
\begin{align*}
\left\{
    \begin{array}{ll}
        \SqrdNrm{X} =  1 \\
        \SqrdNrm{Y} =  \SqrdNrm{y} = 1\\
        \PdtScl{X}{Y} = \PdtScl{x}{y} / \|x\| \,.
    \end{array}
\right.
\end{align*}

Furthermore, $P$ is a random projector, it follows that $X$ and $Y$ are sampled uniformly from the zero-center sphere of radius 1; i.e. $X \sim \mathrm{Unif}(\mathcal{S}_d(0,1))$ and $Y \sim \mathrm{Unif}(\mathcal{S}_d(0,1))$. However, $X$ and $Y$ are not independent, this is why, we consider that $X \sim \mathrm{Unif}(\mathcal{S}_d(0,1))$ and write $Y$ s.t. $Y = aX + bu$ with $u$ a random vector in $\R^d$ of norm $1$ orthogonal to $X$, that is to say, $u | X$ is uniformly sampled on a zero-centered hyper-sphere of radius $1$ orthogonal to the vector $X$ (see illustration on \Cref{app:fig:sphere_projecteur}). It comes that:
\begin{align}
\label{app:eq:formula_sketching}
y^\top C_\Phi(x) = \frac{\|x\|}{p} Y^\top J_{\subdim} X = \frac{\|x\|}{p} (a X^\top + b u^T) J_{\subdim} X = \frac{\|x\|}{p} (a X^\top J_{\subdim} X + b u^\top J_{\subdim} X) \,.
\end{align}

Observe that for any $i,j$ in $\{1, \cdots, d \}$, $X_i$, $X_j$ (resp. $u_i$, $u_j$) have the same law, it results to:
\begin{align}
\label{app:eq:sketching_equalities_expectation}
    \forall (i,j) \in \{1, \cdots, d \}^2,~\forall k \in \N,\qquad \E [X_i^k] = \E [X_j^k] \quad \text{and} \quad \E [u_i^k] = \E [u_j^k]\,.
\end{align}
This property is the key to compute the expectation, the variance, and the covariance of sketching. 

\begin{wrapfigure}[9]{r}{0.33\textwidth}
\vspace{-0.2cm}
  \begin{center}
    \includegraphics[width=0.7\linewidth]{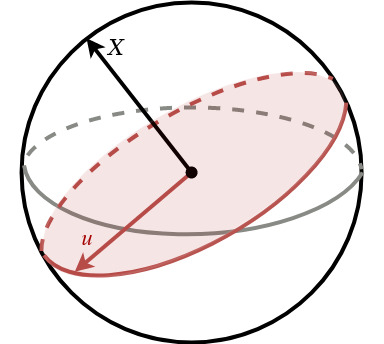}
  \end{center}
  \vspace{-0.5cm} 
  \caption{Sphere zero-center with radius 1: $X$ and $u$ are orthogonal.}
  \label{app:fig:sphere_projecteur}
\end{wrapfigure}

We now compute $a$ and $b$. First, by definition, we have:
\begin{align*}
    \frac{\PdtScl{x}{y}}{\|x\|} = \PdtScl{X}{Y} = a \SqrdNrm{X} = a \,,
\end{align*}

then we write that:
\begin{align*}
    1 = \SqrdNrm{Y} = \frac{\PdtScl{x}{y}^2}{\|x\|^4}
 \SqrdNrm{X} + b^2 \SqrdNrm{u} = \frac{\PdtScl{x}{y}^2}{\|x\|^2} + b^2 \,,
 \end{align*}
 which gives $b = \sqrt{1 - \frac{\PdtScl{x}{y}^2}{\|x\|^2}}$. 

At the end, we have: $Y =  aX + bu = \frac{\PdtScl{x}{y}}{\|x\|} X + \sqrt{1 - \frac{\PdtScl{x}{y}^2}{\|x\|^2}} u$.

\subsubsection{Expectation and variance of sketching}
\label{app:subsubsec:expec_sketching}

In this Subsection, we prove that sketching verifies \Cref{item:urvb} in \Cref{lem:compressor}; for this purpose, we show that it is unbiased, then we compute its variance. 

\begin{proposition}
Sketching is unbiased and its variance is relatively bounded, i.e., it verifies \Cref{item:urvb} in \Cref{lem:compressor} with $\omgC = (1-p)/p$ where $p = \subdim /d$.
\end{proposition}

\begin{proof}
Starting from \Cref{app:eq:formula_sketching}, we have $y^\top C_\Phi(x) = \frac{\|x\|}{p} (a X^\top J_{\subdim} X + b u^\top J_{\subdim} X)$. We first compute the expectation w.r.t. the $\sigma$-algebra $\sigma(\{X\})$ generated by the noise involved in the random vector $X$, it gives:
\[
\expec{y^\top C_\Phi(x)}{\sigma(\{X\})}= \frac{\|x\|}{p} 
 \sum_{i=1}^{\subdim} a X_i^2 + b X_i \Expec{u_i }{\sigma(\{X\})}\,.
\]

Because $u$ is sampled uniformly from the zero-center sphere of radius $1$ s.t. it is orthogonal to $X$, for any $i$ in $\{1, \cdots, d\}$, we have $\expec{u_i}{\sigma(\{X\})} = 0$, hence taking full expectation, we obtain:
\[
\fullexpec{y^\top C_\Phi(x)}= \frac{\|x\|}{p} 
 \sum_{i=1}^{\subdim} a \E[X_i^2]\,.
\]

Using \Cref{app:eq:sketching_equalities_expectation}, we have $\E[X_i^2] = \frac{1}{d} \sum_{j=1}^d \E[X_j^2]$, next recalling that $p=\subdim/d$ and $\|X\|^2 = 1$, it leads to $\fullexpec{y^\top C_\Phi(x)} = a \|x\| \fullexpec{\sum_{j=1}^d X_j^2} = a \|x\| \E[\|X\|^2] = a \|x\|$. And because $a = \PdtScl{x}{y} / \|x \|$, we have at the end that $\E [C_\Phi(x)] = x$.
Now we compute the variance:
\begin{align*}
    \E[C_\Phi(x)^\top C_\Phi(x)] = \frac{1}{p^2} \E[x^\top P^\top J_{\subdim} P P^\top J_{\subdim} P x] = \frac{1}{p^2} \E[x^\top P^\top J_{\subdim} P x] = \frac{\SqrdNrm{x}}{p^2} \E[X^\top J_{\subdim} X]\,.
\end{align*}

$\E[X^\top J_{\subdim} X]$ has been computed above and is equal to $p$, it results that $\E[C_\Phi(x)^\top C_\Phi(x)] = \SqrdNrm{x}/p$.
In the end, sketching verifies \Cref{lem:compressor} with $\omgC = (1-p)/p$.
\end{proof}

\subsubsection{Covariance of sketching.}
\label{app:subsubsec:covariance_sketching}

In this Subsection, we compute the covariance of sketching. For the sake of demonstration, we need to compute the $4^{\mathrm{th}}$-moment of $X_1$ and the $2^{\mathrm{nd}}$-moment of $u_1$. For any $i$ in $[d]$ and any vector $v$ in $\R^d$, we note $v_{-i} = (v_j)_{j\in[d], j\neq i}$ in $\R^{d-1}$.

\textit{Computing the $4^{\mathrm{th}}$-moment of $X_1$.}

The marginal density of $X_1$ is $f_{X_1} : x \mapsto B(\frac{d-1}{2}, \frac{1}{2})^{-1} (1 - x^2)^{(d-3)/2}$ where~$B$ is the beta function defined as~$B: x, y \mapsto \int_0^1 t^{x-1} (1-t)^{y-1} = 2 \int_{0}^{\pi/2} \sin^{2x-1}(t) \cos^{2y-1}(t) \mathrm{d} t$. This result can be obtained either by an application of the formula for the surface area of a sphere \citep{li2010concise,sidiropoulos2014n}, either by writing that $X_1 = \frac{Z_1}{\| Z \|}$ with $Z$ a Gaussian vector with $d$ components.
Therefore we have that:
\[
\E[X_1^4] = \frac{\int_{-1}^1 x^4 (1 - x^2)^{(d-3)/2} \mathrm{d} x}{2 \int_{0}^{\pi/2} \sin^{d-2}(t)  \mathrm{d} t} \stackrel{\mathrm{(i)}}{=} \frac{2\int_0^{\pi/2} \cos^4(t) \sin^{d-2}(t) \mathrm{d} t}{2 \int_{0}^{\pi/2} \sin^{d-2}(t)  \mathrm{d} t} \stackrel{\mathrm{(ii)}}{=} \frac{W_{d-2} - 2 W_{d} + W_{d+2}}{W_{d-2}} \,,
\]

where at (i) we set $x = \cos(t)$ and at (ii) we make appears the Wallis' integrals defined for any $n$ in $\N$ as $W_n = \int_0^{\pi/2} \sin^{n}(t) \mathrm{d} t$. Furthermore, we have the following recursion using integration by parts: $W_{d+2} = \frac{d+1}{d+2} W_{d}$, therefore, we have: 
\begin{align}
\label{app:eq:fourth_moment_X1}
    \E[X_1^4] = \bigpar{1- \frac{2(d-1)}{d} + \frac{(d-1)(d+1)}{d (d+2)}} = \frac{3}{d(d+2)} \,.
\end{align}

\textit{Computing the $2^{\mathrm{nd}}$-moment of $u_1$ w.r.t the $\sigma$-algebra $\sigma({X})$.}

\begin{wrapfigure}[10]{r}{0.55\textwidth}
\vspace{-0.5cm}
  \begin{center}
    \includegraphics[width=0.95\linewidth]{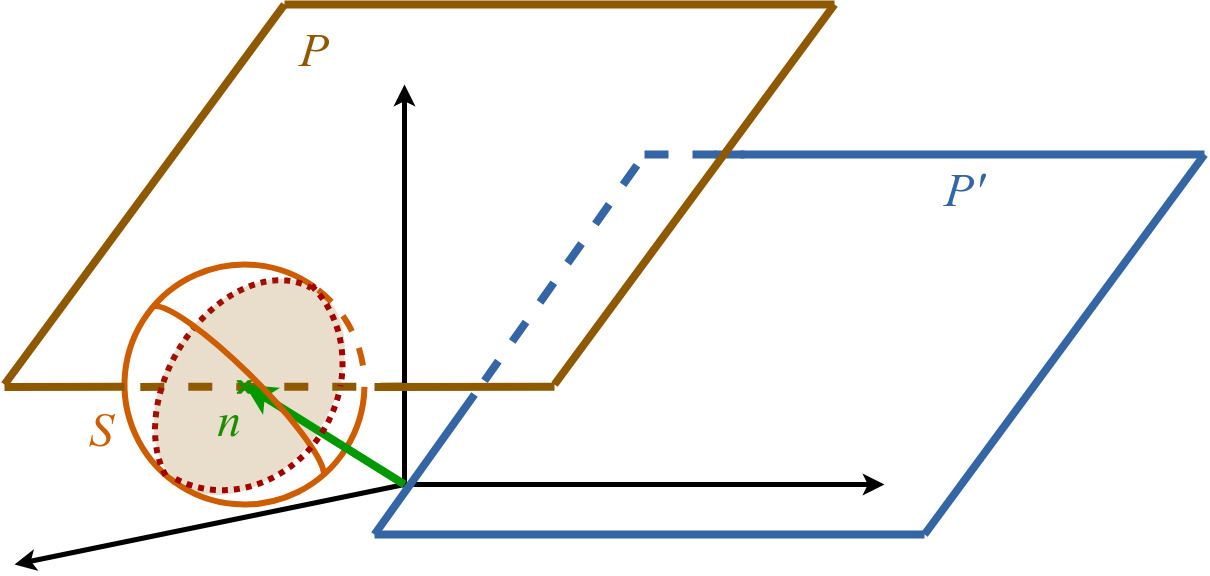}
  \end{center}
  \vspace{-0.5cm} 
  \caption{Parallel hyperplanes $P$ and $P'$ with the sphera $S$.}
  \label{app:fig:plans_and_sphere}
\end{wrapfigure}

We define three $(d-2)$--dimensional manifolds, two parallel hyperplanes $P,P'$ and a sphere~$S$, as follows:
\[
\left\{ 
    \begin{array}{l}
       P = \{\tilde{u} \in \R^{d-1} \mid \pdtscl{\tilde{u}}{X_{-i}} = - X_i u_i \} \\
       P' = \{\tilde{u} \in \R^{d-1} \mid \pdtscl{\tilde{u}}{X_{-i}} = 0 \}  \\
       S = S_{d-1}(0, \sqrt{1 - u_1^2})
    \end{array}  \right.    
\]
Obviously $u_{-i}$ is in $P \cap S$; then we decompose $u_{-i}$ in two terms $ n + v$, with~$v \sim \mathrm{Unif}(P')$ orthogonal to $X$ and independent of $u_i$: $n$ is the center of the sphere $S \cap P $ and $v$ is its radius, $n$ corresponds also to the normal vector of both $P, P'$ with norm equal to the distance between the two hyperplanes, hence $n = \frac{\pdtscl{u_{-i}}{X_{-i}}}{\|X_{-i}\|^2} X_{-i} = - \frac{u_i X_i}{\|X_{-i}\|^2} X_{-i} $.

First, because $u_{-1} \in S$, we have $\| n + v \|^2 = 1  - u_1^2$, next by Pythagorean theorem this is equivalent to $\|v\|^2 = 1 - u_1^2 - \|n\|^2 = 1 - \frac{u_1^2}{\|X_{-1}\|^2}$. Second, because $u_{-1} \in P$, we have $u_1 = \frac{-\pdtscl{u_{-1}}{X_{-1}}}{X_1}$, that is to say the probability density function of $u_1 \mid X$ is proportional to the number of possible values for $u_{-1}$, which corresponds to the surface area of the hypersphere $P \cap S$. This surface is proportional to the radius $\| v \|^{d-4} = (1 - \frac{u_1^2}{\|X_{-1}\|^2})^{(d-4)/2}$ given that $P \cap S$ is a $(d-3)$--dimensional manifold, therefore:
\begin{align*}
    \expec{u_1^2}{\sigma(\{X\})} = \frac{\int_{-\|X_{-1}\|}^{\|X_{-1}\|} x^2 \bigpar{1 - \frac{x^2}{\|X_{-i}\|^2}}^{(d-4)/2} \mathrm{d} x}{\int_{-\|X_{-1}\|}^{\|X_{-1}\|} \bigpar{1 - \frac{x^2}{\|X_{-i}\|^2}}^{(d-4)/2} \mathrm{d} x} &\stackrel{\mathrm{(i)}}{=} \frac{\|X_{-1}\|^2 \int_{-1}^1 y^2 \bigpar{1 - y^2}^{(d-4)/2} \mathrm{d} y}{\int_{-1}^1 \bigpar{1 - y^2}^{(d-4)/2} \mathrm{d} y} \\
    &\stackrel{\mathrm{(ii)}}{=} \|X_{-1}\|^2 \frac{W_{d-3} - W_{d-1}}{W_{d-3}} \,,
\end{align*}

where at (i) we set $y = \frac{x}{\|X_{-1}\|}$ and at (ii) we reuse the previous computations to make appear the Wallis' integral. In the end, we obtain:

\begin{align}
\label{app:eq:second_moment_u1}
    \expec{u_1^2}{\sigma(\{X\})} = (1 - \frac{d-2}{d-1}) \|X_{-1}\|^2 = \frac{\|X_{-1}\|^2}{d-1} \,.
\end{align}

Note that this result is consistent with the fact that $\sum_{i=1}^d \expec{u_i^2}{\sigma(\{X\})} = \frac{d - \sum_{i=1}^d X_i^2}{d-1} = 1$. Now we can compute the covariance of the sketching operator.

\begin{proposition}
Let $x$ in $p_M$, the covariance of sketching is equal to:
\[
\E[\C_\Phi(x)^{\kcarre}] = \frac{1}{p} ((\alpha - \beta) M + \beta \Tr{M} \Id_d) \,,\vspace{-0.5cm}
\]
with $\alpha = \frac{\subdim +2}{d+2}$ and $\beta = \frac{d - \subdim}{(d - 1) (d + 2)}$.
\end{proposition}

\begin{proof}

Let $x$ in $\R^d$ and $y$ in $\R^d$ with $\SqrdNrm{y}=1$, starting from \Cref{app:eq:formula_sketching}, we have:
\begin{align*}
    (y^\top C_\Phi(x))^2 &= \frac{\|x\|^2}{p^2} (a X^\top J_{\subdim} X + b u^\top J_{\subdim} X)^2 \\
    &= \frac{\|x\|^2}{p^2} \Bigg( a^2 (X^\top J_{\subdim} X)^2 + 2 ab (X^\top J_{\subdim} X u^\top J_{\subdim} X) + b^2 (u^\top J_{\subdim} X)^2 \Bigg)\,.
\end{align*}

\textbf{First term.} Taking expectation, we have $\E [(X^\top J_{\subdim} X)^2] = \sum_{i=1}^{\subdim} \bigpar{ \E[X_i^4] + \sum_{j=1, j\neq i}^{\subdim} \E[X_i^2 X_j^2] }$.
However:
\begin{align*}
    \sum_{j=1, j\neq i}^{\subdim} \E[X_i^2 X_j^2] &= \FullExpec{X_i^2 \sum_{j=1, j\neq i}^{\subdim}  X_j^2} \stackrel{\mathrm{(i)}}{=} \FullExpec{X_i^2 \sum_{j=1, j\neq i}^{\subdim}  \frac{1}{d-1} \sum_{k=1, k\neq i}^{d} X_k^2} \\
    &\stackrel{\mathrm{(ii)}}{=} \frac{\subdim - 1}{d-1} \FullExpec{X_i^2 (1 - X_i^2)} \,,
\end{align*}
where we use at line (i)  \Cref{app:eq:sketching_equalities_expectation} and at line (ii) $\sum_{i=1}^d X_i^2 = 1$. It follows that:
\begin{align*}
    \E [(X^\top J_{\subdim} X)^2] &= \sum_{i=1}^{\subdim} \bigpar{\frac{d- \subdim}{d-1}  \E[X_i^4] + \frac{\subdim - 1}{d-1} \FullExpec{X_i^2}}
    \\
    &\stackrel{\mathrm{(i)}}{=} \frac{\subdim(d- \subdim)}{d-1}  \E[X_1^4] + \frac{\subdim - 1}{d-1} \sum_{i=1}^{\subdim} \FullExpec{X_i^2} \\
    &\stackrel{\mathrm{(iii)}}{=} \frac{\subdim(d- \subdim)}{d-1}  \E[X_1^4] + \frac{\subdim(\subdim - 1)}{d(d-1)} \\
    &\stackrel{\mathrm{eq.~\ref{app:eq:fourth_moment_X1}}}{=} \frac{3 \subdim(d- \subdim)}{d(d-1)(d+2)} + \frac{\subdim(\subdim - 1)}{d(d-1)} = \frac{h (h+2)}{d(d+2)} := \alpha'\,.
\end{align*}
Where we considered at line (i) that for any $i$ in $\{1, \cdots, \subdim\}$, $\E[X_i^4] = \E[X_1^4]$, and at line (ii) that $\sum_{i=1}^{\subdim} \FullExpec{X_i^2} = \frac{\subdim}{d} \E[\|X\|^2] = \subdim /d$. 

\textbf{Second term.} We compute the expectation w.r.t. the $\sigma$-algebra $\sigma(\{X\})$ generated by the noise involved in the random vector $X$. It gives $\Expec{X^\top J_{\subdim} X u^\top J_{\subdim} X}{\sigma(\{X\})} = 0$, because $u | X$ is uniformly sampled on a zero-centered hyper-sphere, and thus for any $i$ in $\{1, \cdots, d\}$, we have~$\expec{u_i}{\sigma(\{X\})}~=~0$.

\textbf{Third term.} We have $(u^\top J_{\subdim} X)^2 = \sum_{i=1}^{\subdim} u_i^2 X_i^2 + \sum_{j=1, j\neq i}^{\subdim} u_i u_j X_i X_j$. 
On one side, we compute the expectation w.r.t. the $\sigma$-algebra $\sigma(\{X\})$ generated by the noise involved in the random vector $X$:
\begin{align*}
    \sum_{i=1}^{\subdim} \Expec{u_i^2 X_i^2}{\sigma(\{X\})} &= \sum_{i=1}^{\subdim} X_i^2 \Expec{u_i^2}{\sigma(\{X\})} \stackrel{\mathrm{eq.~\ref{app:eq:second_moment_u1}}}{=} \frac{1}{d-1} \sum_{i=1}^{\subdim} X_i^2 \sqrdnrm{X_{-i}}\,.
\end{align*}
Taking full expectation, we have $ \sum_{i=1}^{\subdim} \fullexpec{u_i^2 X_i^2} = \frac{1}{d-1}  \sum_{i=1}^{\subdim} \fullexpec{X_i^2 (1 - X_i^2)} = \frac{h}{d-1} (\frac{1}{d} - \fullexpec{X_1^4})$, because for any $i$ in $\{1, \dots, \subdim\}$, $\E[X_i^4] = \E[X_1^4]$ and $\sum_{i=1}^{\subdim} \FullExpec{X_i^2} = \frac{\subdim}{d} \E[\|X\|^2] = \subdim /d$.

Let $i$ in $[d]$, on the other side, we compute the expectation w.r.t. the $\sigma$-algebra $\sigma(\{X, u_i\})$ generated by the noise involved in the random vector $X$ and the random variable $u_i$, hence we requires to compute $\Expec{u_j}{\sigma(\{X, u_i\})}$. To do so, as before, we decompose $u_{-i}$ in two terms $ n + v$ (see \Cref{app:fig:plans_and_sphere}), with~$v \sim \mathrm{Unif}(P')$ orthogonal to $X$ and independent of $u_i$, hence $\Expec{v}{\sigma(\{X, u_i\})} = 0$. 
It gives that $\Expec{u_{-i}}{\sigma(\{X, u_i\})} = - \frac{u_i X_i}{\|X_{-i}\|^2} X_{-i} $. Thereby, replacing for any coordinate $j \neq i $ in $[d]$ the value of $u_{-i}$ and taking expectation w.r.t. the $\sigma$-algebra $\sigma(\{X\})$, we obtain:
\begin{align*}
\sum_{i=1}^{\subdim} \sum_{j=1, j\neq i}^{\subdim}  X_i X_j \Expec{u_i u_j}{\sigma(\{X\})} &= -\sum_{i=1}^{\subdim} \sum_{j=1, j\neq i}^{\subdim} \ffrac{1}{\|X_{-i} \|^2} X_i^2 X_j^2 \Expec{u_i^2}{\sigma(\{X\})} \\
&\stackrel{\mathrm{eq.~\ref{app:eq:second_moment_u1}}}{=} -\frac{1}{d-1} \sum_{i=1}^{\subdim} \sum_{j=1, j\neq i}^{\subdim} X_i^2 X_j^2 \\
&= -\frac{1}{d-1} \sum_{i=1}^{\subdim} \sum_{j=1, j\neq i}^{\subdim} X_i^2 \frac{1- X_i^2}{d-1} \,.
\end{align*}
Finally, we have: $\sum_{i=1}^{\subdim} \sum_{j=1, j\neq i}^{\subdim}  \E[X_i X_j u_i u_j] = - \frac{h(h-1)}{d(d-1)^2} (1  - \sum_{i=1}^d \E[X_i^4])$. Putting together the two terms, we have that:
\begin{align*}
    \FullExpec{(u^\top J_{\subdim} X)^2} &= \frac{h}{d-1} (\frac{1}{d} - 
    \E[X_i^4]) - \frac{h(h-1)}{d(d-1)^2} (1  - d \E[X_1^4]) \stackrel{\mathrm{eq.~\ref{app:eq:fourth_moment_X1}}}{=} \frac{h(d-h)}{d(d-1) (d+2)} := \beta'\,.
\end{align*}

\textbf{In the end,} we have $\E[ (y^\top C_\Phi(x))^2] = \frac{\|x\|^2}{p^2} (a^2 \alpha' + b^2 \beta')$.
And because $\|y\|^2 = 1$, $a = \PdtScl{x}{y} / \|x \|$ and~$b = \sqrt{1 - \PdtScl{x}{y}^2 / \|x \|^2}$, replacing them by their values gives:
\begin{align*}
     y^\top \E[ C_\Phi(x))^{\kcarre}] y &= \frac{\SqrdNrm{x}}{p^2} \bigpar{\alpha' \frac{\PdtScl{x}{y}^2}{\|x\|^2} + \beta' \bigpar{y^\top y - \frac{\PdtScl{x}{y}^2}{\|x\|^2}}} \,,
\end{align*}

hence $\E[ C_\Phi(x))^{\kcarre}] = \frac{1}{p^2} \bigpar{ (\alpha' - \beta') x x^\top + \beta' \SqrdNrm{x} \Id_d}$.
To conclude, we consider that $x$ is a random variable sampled from a distribution $p_M$, then taking expectation on this random variable we have:
$\E C_\Phi(x)^{\kcarre} = \frac{1}{p} \bigpar{ (\alpha - \beta) M  + \beta \Tr{M} \Id_d}$, with $\alpha = \frac{\alpha'}{p} = \frac{\subdim +2}{d+2}$ and $\beta = \frac{\beta'}{p} =  \frac{d - \subdim}{(d-1)(d +2)}$.
\vspace{-1cm}
\end{proof}

\subsection{Proof of Propositions~\ref{prop:particular_cases} and \ref{prop:particular_cases_quant}}
\label{app:subsec:proof_particular_case}

In this Subsection, we give the proof of \Cref{prop:particular_cases,prop:particular_cases_quant} which provides generic comparisons between the asymptotic convergence rate of compressors. 
We first give a lemma resulting from the Cauchy-Schwarz's inequality necessary to establish these proofs.

\begin{lemma}[Cauchy-Schwarz's inequality on matrices' traces]
\label[lemma]{app:lem:cs_matrix}
    For any matrix $M$ in $\R^{d \times d}$, we have~$\Tr{M} \Tr{M^{-1}}~\geq~d^2$, with strict inequalities if $M$ is not proportional to $\Id_d$. And if $M$ is with constant diagonal equal to $c$ in $\R$, we have $c\Tr{M^{-1}} \geq d$.
\end{lemma}

\begin{proof}
    Let $M$ in $\R^{d \times d}$, using the Cauchy-Schwarz inequality, we have: $$d^2 = \Tr{\Id_d}^2 = \Tr{M^{1/2} M^{-1/2}}^2 \stackrel{\mathrm{C.S}}{\leq} \Tr{M} \Tr{M^{-1}}\,,$$ and we have equality if $M$ is proportional to $\Id_d$.
\end{proof}

Now we give the demonstration of  \Cref{prop:particular_cases,prop:particular_cases_quant}. On \Cref{app:fig:compression_scatter_plot_std}, we complete the numerical illustration provided in \Cref{subsubsec:illustration_dim_2} by illustrating the scenario of standardized features, i.e., when the diagonal of $M$ is the identity.

\begin{figure}
\vspace{-1cm}
  \begin{center}
    \includegraphics[width=1\linewidth]{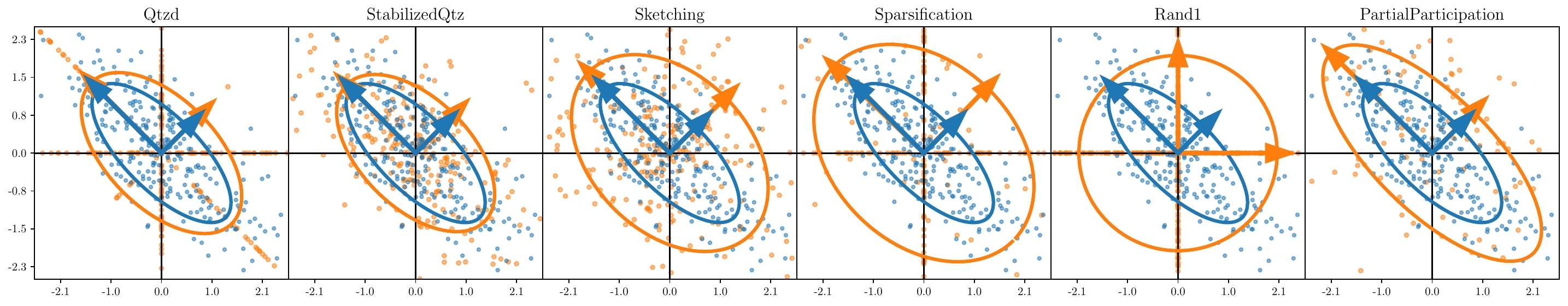}
  \end{center}
  \vspace{-0.5cm} 
  \caption{$\xCov$ not diagonal, scenario using features standardization. Scatter plot of \textcolor{tabblue}{$(x_k)_{i=1}^K$}/ \textcolor{taborange}{$(\C(x_k))_{i=1}^K$} with its ellipse \textcolor{tabblue}{$\mathcal{E}_{\Cov {x_k}}$}/\textcolor{taborange}{$\mathcal{E}_{\Cov {\C (x_k)}}$}.}
  \label{app:fig:compression_scatter_plot_std}
\end{figure}

\begin{proposition}[Comparison between ${\Cpp, \Cspars, \Crandh, \Csketch}$, $\omega=d/h- 1$] 
\label[proposition]{app:prop:particular_cases} We consider $\C$ in $\{\Cpp, \Cspars$, $\Crandh, \Csketch\}$ with $p=h/d$, such that $\C$ always satisfies \Cref{lem:compressor} with $\omgC = d/h- 1$. For any matrix~$M\in\R^{d \times d}$:
\begin{enumerate}[topsep=2pt,itemsep=1pt,leftmargin=0.5cm]
    \item \label{app:item:eq_trace_diag}  If $M$ is diagonal, then:
    \begin{itemize}
        \item $\compCov[][\mathrm{PP}][M]  = \compCov[][\mathrm{s}][M]  = \compCov[][\mathrm{rd}h][M]  = \frac{d}{h} M$, 
        \item $\Tr{\compCov[][\mathrm{PP}/\mathrm{s}/\mathrm{rd}h][M] M^{-1}} \leq \Tr{\compCov[][\Phi][M] M^{-1}}$.
    \end{itemize}
    \item \label{app:item:ineq_trace_diag} Moreover,  for any matrix $M$ with a \emph{constant diagonal} (e.g., after standardization), we have:
    $$\mathrm{Tr}(\compCov[][\mathrm{PP}][M] M^{-1}) \le \mathrm{Tr}(\compCov[][\Phi][M] M^{-1}) \le \mathrm{Tr}(\compCov[][s][M] M^{-1})\le \mathrm{Tr}(\compCov[][\mathrm{rd}h][M] M^{-1}) \,,$$
    with strict inequalities if $M$ is not proportional to $\Id_d$. 
\end{enumerate}
\end{proposition}

\begin{proof}

Let $M$ in $\R^{d \times d}$ and take $p = h/d$.

\paragraph{Proof of \Cref{item:eq_trace_diag} in \Cref{prop:particular_cases}.}
In the diagonal case, the first equalities are straightforward as we have $\compCov[][\mathrm{PP}][M]  = \compCov[][\mathrm{s}][M]  = \compCov[][\mathrm{rd}1][M]  = \frac{d}{h} M$.
Next, we have (regardless if $M$ is diagonal or not):
\begin{align*}
\Tr{\bigpar{\compCov[][\mathrm{\Phi}][M] - \compCov[][\mathrm{PP}][M]} M^{-1}} &=  (\frac{h+1}{d+2} + \delta_{hd} -1)  \frac{\Tr{\Id_d}}{p} + (1  - \frac{h-1}{d-1} ) \frac{\Tr{M} \Tr{M^{-1}}}{p(d+2)} \\
&\stackrel{\mathrm{\Cref{app:lem:cs_matrix}}}{\geq} \frac{d}{p} \bigpar{\frac{h+1}{d+2} + \delta_{hd} -1 + \frac{d}{d+2} (1  - \frac{h-1}{d-1} ) } \\
&= 0 \,.
\end{align*}

\paragraph{Proof of \Cref{item:ineq_trace_diag} in \Cref{prop:particular_cases}.}   

Suppose now that $\Diag{M} = c \Id_d$, then we have $\compCov[][\mathrm{PP}][M]  = \frac{d}{h} M $, $\compCov[][\mathrm{s}][M]  = M + (\frac{d}{h} - 1) c\Id_d $, $\compCov[][\mathrm{rd}h][M]  = \frac{d (h-1)}{h(d-1)} M + \frac{d}{h} (1 - \frac{h-1}{d-1}) c \Id_d$ and $\compCov[][\Phi][M] = \frac{d}{h} \bigpar{(\frac{h+1}{d+2} - \delta_{hd}) M + \bigpar{1 - \frac{h-1}{d-1}} \frac{\Tr{M} }{d+2}\Id_d}$.
Firstly, from previous item, we have $$ \Tr{\compCov[][\mathrm{PP}][M] M^{-1}} \leq \Tr{\compCov[][\mathrm{\Phi}][M]} M^{-1}\,.$$
Secondly, we write:
\begin{align*}
    \Tr{\bigpar{\compCov[][\Phi][M] - \compCov[][\mathrm{s}][M]} M^{-1}} &=  \frac{d}{p} \bigpar{ \frac{h+1}{d+2} + \delta_{hd} - \frac{h}{d}} \\
    &\qquad+ \frac{c \Tr{M^{-1}}}{p} \bigpar{ \frac{d}{d+2} (1 - \frac{h-1}{d-1}) - (1 - \frac{h}{d}) } \\
    &= \frac{d}{p} \bigpar{ \frac{h+1}{d+2} + \delta_{hd} - \frac{h}{d}} - \frac{c \Tr{M^{-1}}}{p} \cdot \frac{(d-2)(d-h)}{d(d-1)(d+2)} \\
    &\stackrel{\mathrm{\Cref{app:lem:cs_matrix}}}{\leq} \frac{d}{p} \bigpar{ \frac{h+1}{d+2} + \delta_{hd} - \frac{h}{d} - \frac{(d-2)(d-h)}{d(d-1)(d+2)}} = 0 \,. 
\end{align*}
Thirdly, we have:
\begin{align*}
    \Tr{\bigpar{\compCov[][\mathrm{rd}h][M] - \compCov[][\mathrm{s}][M]} M^{-1}} &=  \frac{h-d}{h(d-1)} \Tr{\Id_d} + \frac{d-h}{h(d-1)} c \Tr{M^{-1}} \\
    &\stackrel{\mathrm{\Cref{app:lem:cs_matrix}}}{\geq} \frac{d}{h} \bigpar{ \frac{h-d}{d-1} + \frac{d-h}{d-1} } = 0\,. 
\end{align*}
\end{proof}

\begin{proposition}[Comparison between ${\Cpp, \Cquant,  \Cspars}$, $\omega=\sqrt{d}$ ] 
\label[proposition]{app:prop:particular_cases_quant}

We consider $\C$ in \\ $\{\Cpp, \Cquant, \Cspars\}$ with $p=(\sqrt{d} + 1)^{-1}$, such that $\C$ always satisfies \Cref{lem:compressor} with $\omega = \sqrt{d}$. 
\begin{enumerate}[topsep=2pt,itemsep=1pt,leftmargin=0.5cm]
\item \label{app:item:ineq_trace_diag_quantiz} For any symmetric matrix $M$ diagonal, we have:
$$\Tr{\compCov[][\mathrm{PP}][M] M^{-1}} = \Tr{\compCov[][s][M] M^{-1}} \overset{\text{possib. } \ll }{\leq } \left(1 + \frac{1}{\sqrt{d}}\right)\Tr{ \widetilde{\mathfrak{C}}(\Cquant, M)M^{-1}}\,.$$
\item \label{app:item:ineq_trace_diag_quantiz_sparsif} If $M$ is not necessarily diagonal  but with a \emph{constant diagonal} (e.g., after standardization), then 
\begin{itemize}
    \item $\widetilde{\mathfrak{C}}(\Cquant, M) \preccurlyeq \compCov[][\mathrm{s}][M]$   
    \item $\Tr{\compCov[][\mathrm{PP}][M] M^{-1}} \leq  \left(1 + \frac{1}{\sqrt{d}}\right)\Tr{ \widetilde{\mathfrak{C}}(\Cquant, M)M^{-1}} $ \,.
\end{itemize}
\end{enumerate}
\end{proposition}

\begin{proof}

Let $M$ in $\R^{d \times d}$ and take $p = \frac{1}{1 + \sqrt{d}}$.
\paragraph{Proof of \Cref{item:ineq_trace_diag_quantiz} in \Cref{prop:particular_cases_quant}.} In the diagonal case with $p = \frac{1}{1 + \sqrt{d}}$, we have $\widetilde{\mathfrak{C}}(\Cquant, M) = \sqrt{\Tr{M} } \sqrt{M}$ and $\compCov[][\mathrm{PP}][M]  = (1 + \sqrt{d}) M $, hence $\Tr{ \widetilde{\mathfrak{C}}(\Cquant, M) M^{-1}} = \sqrt{\Tr{M}} \Tr{\sqrt{M^{-1}}} $ and $\Tr{\compCov[][\mathrm{PP}][M]  M^{-1}} = (1 + \sqrt{d}) d$. Noting $(\lambda_i)_{i \in [d]}$ the eigenvalues of $M$, and using the Cauchy-Schwarz inequality's, we have:
\begin{align*}
    d^2 &= \bigpar{\sum_{i=1}^d 1 }^2 = \bigpar{\sum_{i=1}^d \lambda_i^{1/4} \lambda_i^{-1/4} }^2 \stackrel{\mathrm{C.S}}{\leq} \bigpar{\sum_{i=1}^d \lambda_i^{1/2} } \bigpar{\sum_{i=1}^d \lambda_i^{-1/2} } \\
    &\stackrel{\mathrm{C.S}}{\leq} \sqrt{\sum_{i=1}^d \lambda_i } \sqrt{\sum_{i=1}^d 1 } \bigpar{\sum_{i=1}^d \lambda_i^{-1/2} } = \sqrt{d \Tr{M}} \Tr{M^{-1/2}} = \sqrt{d} \Tr{ \widetilde{\mathfrak{C}}(\Cquant, M) M^{-1}} \,.
\end{align*}

Which follows that $\Tr{ \widetilde{\mathfrak{C}}(\Cquant, M) M^{-1}} \geq d^{3/2} = \sqrt{d} (1 + \sqrt{d})^{-1} \Tr{\compCov[][\mathrm{PP}][M] M^{-1}}$ and it allows to conclude.

\paragraph{Proof of \Cref{item:ineq_trace_diag_quantiz_sparsif} in \Cref{prop:particular_cases_quant}.}
Suppose now that $\Diag{M} = c \Id_d$, then we have $\compCov[][\mathrm{PP}][M]  = (\sqrt{d} + 1) M $, $\widetilde{\mathfrak{C}}(\Cquant, M)  = M  + (\sqrt{d} - 1) c \Id_d$, and $\compCov[][\mathrm{s}][M]  = M + c \sqrt{d} \Id_d $.
Firstly, it follows that:
\begin{align*}
    \compCov[][\mathrm{s}][M] - \widetilde{\mathfrak{C}}(\Cquant, M) = \bigpar{M + \sqrt{d} c \Id_d} - \bigpar{M + (\sqrt{d} - 1) c \Id_d} = c \Id_d \succcurlyeq 0 \,,
\end{align*}
Secondly, we have $(1 + \frac{1}{\sqrt{d}})\widetilde{\mathfrak{C}}(\Cquant, M) - \compCov[][\mathrm{PP}][M] = - ( 1 - \frac{1}{\sqrt{d}} ) M + (\sqrt{d} - \frac{1}{\sqrt{d}}) c \Id_d$,
which gives:
\begin{align*}
    \Tr{\bigpar{ (\sqrt{d} - \frac{1}{\sqrt{d}}) \widetilde{\mathfrak{C}}(\Cquant, M) - \compCov[][\mathrm{PP}][M]} M^{-1}} &= (\sqrt{d} - \frac{1}{\sqrt{d}})  c\Tr{M^{-1}} - ( 1 - \frac{1}{\sqrt{d}} ) \Tr{\Id_d} \\
    &\geq (\sqrt{d} - \frac{1}{\sqrt{d}})  d - ( 1 - \frac{1}{\sqrt{d}} ) d \text{\quad(\Cref{app:lem:cs_matrix})} \\
    &\geq d (\sqrt{d} - 1) \geq 0 \,. 
\end{align*}
And the proof is concluded.

\end{proof}

\subsection{Empirical covariances computed on quantum and cifar10}
\label{app:subsec:plot_cov_matrix}

On \Cref{app:fig:cov_aniac}, for both \texttt{quantum} and \texttt{cifar-10}, we first plot the covariance matrix (1) without any processing and (2) with standardization. In this latter case, we then plot the covariances induced by quantization and sparsification for $\omgC = 1$ and $8$. For \texttt{quantum}, without standardization, only four points are visible; it is caused by some rows having extremely large values at features $27$ and $43$, resulting in a feature mean $100$ times greater than the others. 

Looking at the covariance induced by the compressors, we observe that for small $\omgC$, quantization better preserves the matrix structure compared to sparsification. This fact is consistent with \Cref{fig:real_dataset} where is given the trace of $\compCov[][M][] M^{-1}$ for these eight covariances: the traces for quantization are indeed smaller than for sparsification. This is also consistent with \Cref{fig:sgd_quantum,fig:sgd_cifar10} where $\omgC = 1$ and where quantization outperforms sparsification. 

\newcommand{\addCov}[2]{\includegraphics[align=c,width=0.45\textwidth]{pictures/real_dataset/cov_matrix/#1_#2_cov.pdf}}

\newcolumntype{C}{>{\centering\arraybackslash}m{0.45\textwidth}}

\newcommand{\splitCell}[2]{#1 \qqquad\qquad #2}

\begin{table}
\caption{(1) Data covariances for \texttt{quantum} and $\texttt{cifar-10}$. (2) Covariance $\compCov[][M][]$ w./w.o. standardization for quantization and sparsification; see \Cref{fig:real_dataset} to have the corresponding trace of $\compCov[][M][] M^{-1}$.}
\label{app:fig:cov_aniac}
\resizebox{\linewidth}{!}{
\begin{tabular}{m{1em}ccc}
\toprule
& $M$ & Quantization & Sparsification \\ 
 &raw-data \qqquad standardized  & \splitCell{$\omgC =1$}{$\omgC = 8$} & \splitCell{$\omgC = 1$}{$\omgC = 8$} \\ 
\cmidrule(lr){2-2}\cmidrule(l){3-3}\cmidrule(l){4-4}
\rotatebox{90}{\texttt{quantum}} &\addCov{quantum}{no_compr}& \addCov{quantum}{Qtzd} & \addCov{quantum}{Sparsification} \\ 
 \rotatebox{90}{\texttt{cifar-10}} &\addCov{cifar10}{no_compr}& \addCov{cifar10}{Qtzd} & \addCov{cifar10}{Sparsification}  \\ 
\bottomrule 
\end{tabular}}
\end{table} 	 
\section{Technical results on federated learning.}
\label{app:sec:bm_fl}

\subsection{Validity of the assumptions made on the random fields in the case of covariate-shift}
\label{app:subsec:validity_asu_random_fields_fl_sigma}

In this Subsection, we examine the setting of federated and compressed LSR under the scenario of covariate-shift (\Cref{subsec:fl_sigma}). Specifically, we consider the case where for any $i,j$ in $\llbracket 1, N \rrbracket$, we have heterogeneous covariances, i.e., $\xCov_i \neq \xCov_j$, but a unique optimal model i.e. $\ws^i = \ws$. 
We verify that all the assumptions on the random fields done in \Cref{subsec:def_ania_asu_field} are fulfilled in the setting. For this purpose, we redefine the filtration given in \Cref{app:subsec:validity_asu_random_fields} to align them with the FL setting. For $k$ in $\N^*$ and for $i$ in $\OneToN$, we note $u_k^i$ the noise that controls the compression $\C_k^i(\cdot)$ at round $k$.

\begin{definition}
We note $(\Fx_k)_{k \in \N}$ the filtration associated with the features noise, $(\Fy_k)_{k \in \N}$ the filtration associated with the label noise, and $(\Fsto_k)_{k \in \N}$ the filtration associated to the stochastic gradient noise, which is the union of the two previous filtrations. For $k \in \N^*$, we define $\Flast_{0} = \{ \varnothing \}$ and
\begin{align*}
    \Fx_k &= \sigma \bigpar{\Flast_{k-1} \cup \{(x_k^i)_\iN\} } \\
    \Fy_k &= \sigma \bigpar{\Flast_{k-1} \cup \{(\varepsilon_k^i)_\iN\} } \\
    \Fsto_k &= \sigma \bigpar{\Flast_{k-1} \cup \{ (x_k^i, \varepsilon_k^i)_\iN\}} \\
    \Flast_k &= \sigma \bigpar{\Flast_{k-1} \cup \{(x_k, \varepsilon_k^i, u_k^i)_\iN\} }\,.
\end{align*}
\end{definition}

Now we prove that all assumptions done in \Cref{sec:theoretical_analysis} are correct in this setting.

\begin{property}[Validity of the setting presented in \Cref{def:class_of_algo}]
\label{prop:zero_centered_noise_fl}
For \Cref{ex:dist_comp_LMS} in the context of \Cref{model:fed}, we have that the setting presented in \Cref{def:class_of_algo} is verified.
\end{property}

\begin{proof}
From~\Cref{ex:dist_comp_LMS}, we have for any $k$ in $\N^*$ and any $w$ in $\R^d$, $\xi_k(w - \ws) = \nabla F(w) - \frac{1}{N} \sum_\iN \C_k^i(\g_k^i(w))$.
Because $(\g_k^i)_{k\in\N^*, i\in\llbracket 1, N \rrbracket}$ and $(\C_k^i)_{k\in\N^*, i\in\llbracket 1, N \rrbracket}$ are by definition two sequences of i.i.d. random fields (\Cref{ex:dist_comp_LMS}), it follows that their composition is also i.i.d., \emph{hence $(\xi_k)_{k\in\N^*}$ is a sequence of i.i.d. random fields.}

Taking expectation w.r.t. the $\sigma$-algebra $\Fsto_k$ we have $\Expec{\C_k^i(\g_k^i(w))}{\Fsto_k} = \g_k^i(w)$ (\Cref{lem:compressor}), next with the $\sigma$-algebra $\Flast_{k-1}$, we have $\Expec{\g_k^i(w)}{\Flast_{k-1}} = \nabla F_i(w)$ (\Cref{eq:def_oracle}).
And because $\frac{1}{N} \sum_\iN \nabla F_i(w) = \nabla F(w)$, \emph{we obtain that the random fields are zero-centered.}

From \Cref{model:fed}, we have for any $k$ in $\N^*$ and any $w$ in $\R^d$ that:
\begin{align*}
    F(w) &= \frac{1}{2N} \sum_\iN \FullExpec{(\PdtScl{x_k^i}{w} - y_k^i)^2} \\
    &= \frac{1}{2N} \sum_\iN  \FullExpec{(w - w_*)^\top (x_k^i \otimes x_k^i) (w - w_*) - 2 \varepsilon_k^i \PdtScl{x_k^i}{w - w_*}+ (\varepsilon_k^i)^2} \\
    &=   \frac{1}{2N} \sum_\iN (w - \ws)^\top \xCov_i (w - \ws) + \sigma^2 =  \frac{1}{2} ((w - \ws)^\top \overline{\xCov} (w - \ws) + \sigma^2)\,.
\end{align*}

And we have from \Cref{model:fed}: $\Tr{\overline{\xCov}} = \frac{1}{N} \sum_\iN \Tr{\xCov_i} = \frac{1}{N} \sum_\iN  R_i^2 =:  \overline{R}^2$,
which concludes the verification.
\end{proof}

\begin{property}[Validity of \Cref{asu:main:bound_add_noise}]
\label{prop:add_noise_fl_sigma}
Consider \Cref{ex:dist_comp_LMS} and \Cref{model:fed} with \Cref{lem:compressor}, for any iteration $k$ in $\N^*$, the second moment of the additive noise $\xi_k^{\mathrm{add}}$ can be bounded by $(\omgC +1)  \overline{R}^2 \sigma^2 / N$ i.e. \Cref{asu:main:bound_add_noise} is verified.
\end{property}

\begin{proof}
Let $k$ in $\N^*$. Because we consider \Cref{ex:dist_comp_LMS}, with \Cref{def:class_of_algo,def:add_mult_noise}, we first have $\xi_k^{\mathrm{add}} = - \frac{1}{N} \sum_\iN \C_k^i(\gwkstar^i)$, hence taking expectation w.r.t the $\sigma$-algebra $\Fsto_k$ and because the $N$ compressions are independent (\Cref{ex:dist_comp_LMS}), using \Cref{lem:compressor}, we have that:
\begin{align*}
    \Expec{\|\xi_k^{\mathrm{add}}\|^2}{\Fsto_k} &= \frac{1}{N^2} \sum_\iN \Expec{\SqrdNrm{\C_k^i(\gwkstar^i)}}{\Fsto_k} + \frac{1}{N^2} \sum_{i \neq j} \PdtScl{\gwkstar^i}{\gwkstar^j}\\
    &\leq \frac{\omgC + 1}{N^2} \sum_\iN \SqrdNrm{\gwkstar^i} + \frac{1}{N^2} \sum_{i \neq j} \PdtScl{\gwkstar^i}{\gwkstar^j} \,.
\end{align*}

Next, we first have from \Cref{model:fed} and \Cref{eq:def_oracle} that for any $i$ in $\OneToN$, $\gwkstar^i = -\varepsilon_k^i x_k^i$, secondly because $\bigpar{(\varepsilon_k^i)_{k\in \OneToK, i\in\OneToN}}$ are independent from $\bigpar{(x_k^i)_{k\in\OneToK,i\in\OneToN}}$ (\Cref{model:fed}), we have that $\E[\|\varepsilon_k^i x_k^i\|^2] \leq \sigma^2  R_i^2$, hence $\Expec{\|\xi_k^{\mathrm{add}}\|^2}{\Flast_{k-1}} = \E[\|\xi_k^{\mathrm{add}}\|^2] = \frac{\omgC + 1}{N^2} \sum_\iN \sigma^2  R_i^2\,.$
\end{proof}

\begin{property}[Validity of \Cref{asu:main:bound_mult_noise}]
\label{prop:mult_noise_fl}
Consider \Cref{ex:dist_comp_LMS} in the context of \Cref{model:fed} with \Cref{lem:compressor}, for any iteration $k$ in $\N^*$, the second moment of the multiplicative noise $\xi_k^{\mathrm{mult}}(w)$ can be bounded for any $w$ in $\R^d$ by $2(\omgC +1)  \max_{i \in  \OneToN}(R_i^2) \SqrdNrm{\overline{\xCov}^{1/2} (w - \ws)} / N + 4 (\omgC + 1)  \overline{R}^2 \sigma^2 / N $ i.e. \Cref{asu:main:bound_mult_noise} is verified.
\end{property}

\begin{proof}
Let $k$ in $\N^*$, we note $\eta = w - \ws$.
Because we consider \Cref{ex:dist_comp_LMS}, with \Cref{def:class_of_algo,def:add_mult_noise}, we write $\xi_k^{\mathrm{mult}}(\eta) = \frac{1}{N} \sum_\iN \xi_k^{i, \mathrm{mult}}(\eta)$, where $\xi_k^{i, \mathrm{mult}}(\eta) = \xCov_i \eta - \C(\g_k^i(w)) + \C(\gwkstar^i)$ is the multiplicative noise on client $i$ in $ \OneToN$, hence developing the squared norm gives:
\begin{align*}
    \SqrdNrm{\xikmultFN[\eta]} = \SqrdNrm{\frac{1}{N} \sum_\iN \xi_k^{i, \mathrm{mult}}(\eta)} = \frac{1}{N^2} \sum_\iN \SqrdNrm{\xi_k^{i, \mathrm{mult}}(\eta)} + \frac{1}{N^2} \sum_{i \neq j} \PdtScl{\xi_k^{i, \mathrm{mult}}(\eta)}{\xi_k^{j, \mathrm{mult}}(\eta)} \,.
\end{align*}

Taking expectation w.r.t. the $\sigma$-algebra $\Flast_{k-1}$, using that the $N$ compressions are independent (\Cref{ex:dist_comp_LMS}) and that for any $i$ in $\OneToN$, $\expec{\xi_k^{i, \mathrm{mult}}(\eta)}{\Flast_{k-1}}=0$ (\Cref{lem:compressor}) results to have:
\begin{align*}
    \expec{\sqrdnrm{\xikmultFN[\eta]}}{\Flast_{k-1}} = \frac{1}{N^2} \sum_\iN \expec{\sqrdnrm{\xi_k^{i, \mathrm{mult}}(\eta)}}{\Flast_{k-1}}\,.
\end{align*}

Next, we use the result of \Cref{prop:mult_noise} for each client $i$ in $\OneToN$ and we obtain:
\begin{align*}
    \Expec{\|\xikmultFN[\eta]\|^2}{\Flast_{k-1}} &\leq \frac{1}{N^2} \sum_\iN \bigpar{2(\omgC +1)  R_i^2 \|\xCov_i^{1/2} (w - \ws)\|^2 + 4(\omgC +1)  R_i^2 \sigma^2} \\
    &\leq \frac{2(\omgC +1)  \max_{i \in  \OneToN}(R_i^2)}{N}  \|\overline{\xCov}^{1/2} (w - \ws)\|^2 + \frac{4(\omgC +1)  \overline{R}^2 \sigma^2}{N} \,,
\end{align*}
which allows concluding.

\end{proof}

\begin{property}[Validity of \Cref{asu:main:bound_mult_noise_holder}]
\label{prop:mult_noise_holder_fl}
Consider \Cref{ex:dist_comp_LMS} in the context of \Cref{model:fed} with \Cref{lem:compressor}, for any iteration $k$ in $\N^*$, the second moment of the multiplicative noise $\xi_k^{\mathrm{mult}}(w)$ can be bounded for any $w$ in $\R^d$ by $(\omgCOne \sigma  \max_{i \in  \OneToN}(R_i^2) \|\overline{\xCov}^{1/2} (w - \ws) \| + \omgCTwo  \max_{i \in  \OneToN}(R_i^2) \|\overline{\xCov}^{1/2} (w - \ws) \|^2) / N$ i.e. \Cref{asu:main:bound_mult_noise_holder} is verified.
\end{property}

\begin{proof}
Let $k$ in $\N^*$, we note $\eta = w - \ws$.
From \Cref{prop:mult_noise_fl}, taking expectation w.r.t. the $\sigma$-algebra $\Flast_{k-1}$, decomposing the multiplicative noise results to have:
\begin{align*}
    \expec{\sqrdnrm{\xikmultFN[\eta]}}{\Flast_{k-1}} = \frac{1}{N^2} \sum_\iN \expec{\sqrdnrm{\xi_k^{i, \mathrm{mult}}(\eta)}}{\Flast_{k-1}}\,.
\end{align*}
Next we use the result of \Cref{prop:mult_noise_holder} for each client $i$ in $\OneToN$ and we obtain:
\begin{align*}
    \expec{\sqrdnrm{\xikmultFN[\eta]}}{\Flast_{k-1}} &\leq \frac{1}{N^2} \sum_\iN \omgCOne R_i^2 \sigma \sqrt{ \|\xCov_i^{1/2} (w - \ws) \|^2} + \omgCTwo  R_i^2 \|\xCov_i^{1/2} (w - \ws) \|^2\,.
\end{align*}

With Jensen's inequality \ref{app:basic_ineq:jensen} used for concave function:
\begin{align*}
    \Expec{\SqrdNrm{\xikmultFN[\eta]}}{\Flast_{k-1}} &\leq \frac{\omgCOne \sigma \max_{i \in  \OneToN}(R_i^2)}{N} \sqrt{ \frac{1}{N} \sum_\iN \|\xCov_i^{1/2} (w - \ws) \|^2} \\
    &\qquad+ \frac{\omgCTwo  \max_{i \in  \OneToN}(R_i^2) }{N^2} \sum_\iN \|\xCov_i^{1/2} (w - \ws) \|^2  \\
    &\leq \frac{\omgCOne \sigma  \max_{i \in  \OneToN}(R_i^2)}{N} \sqrt{ \|\overline{\xCov}^{1/2} (w - \ws) \|^2} \\
    &\qquad+ \frac{1}{N} \omgCTwo  \max_{i \in  \OneToN}(R_i^2) \|\overline{\xCov}^{1/2} (w - \ws) \|^2 \,,
\end{align*}

which allows concluding.

\end{proof}

\begin{property}[Validity of \Cref{asu:main:bound_mult_noise_lin}]
\label{prop:mult_noise_lin_fl}
Consider \Cref{ex:dist_comp_LMS} and \Cref{model:fed} with \Cref{lem:compressor}, if the compressor $\C$ is linear, then for any iteration $k$ in $\N^*$, the multiplicative noise $\xi_k^{\mathrm{mult}}$ is linear, thus there exist a matrix $\Xi_k$ in $\R^{d\times d}$ such that for any $w$ in $\R^d$, $\xi_k^{\mathrm{mult}}(w) = \Xi_k w$. Furthermore the second moment of the multiplicative noise can be bounded for any $w$ in $\R^d$ by $(\omgC +1)  \max_{i \in  \OneToN}(R_i^2) \SqrdNrm{\overline{\xCov}^{1/2} (w - \ws)} / N$, hence \Cref{asu:main:bound_mult_noise_lin} is verified.
\end{property}

\begin{proof}
Let $k$ in $\N^*$, we note $\eta = w - \ws$.
Because we consider \Cref{ex:dist_comp_LMS}, with \Cref{def:class_of_algo,def:add_mult_noise}, we write $\xi_k^{\mathrm{mult}}(\eta) = \frac{1}{N} \sum_\iN \xi_k^{i, \mathrm{mult}}(\eta)$, where $\xi_k^{i, \mathrm{mult}}(\eta) = \xCov_i \eta - \C(\g_k^i(w)) + \C(\gwkstar^i)$ is the multiplicative noise on client $i$ in $ \OneToN$.
And because for any clients $i$ in $\{1, \cdots N\}$ the random mechanism $\C^i_k$ is linear, there exists a random matrix $\Pi_k^i$ in $\R^{d\times d}$ s.t. for any $z$ in~$\R^d$, we have $\C_k^i(z) = \Pi_k^i z$, it follows that:
\begin{align*}
\xikmultFN[\eta] = \nabla F(w) - \frac{1}{N} \sum_\iN \C_k^i(\g_k^i(w)) + \C_k^i(\gwkstar^i) = \bigpar{\overline{\xCov} - \frac{1}{N} \sum_\iN \Pi_k^i (x_k^i \otimes x_k^i) } \eta \,.
\end{align*}
Hence, the first part of \Cref{asu:main:bound_mult_noise_holder} is verified with $ \Xi_k = \frac{1}{N} \sum_\iN \xCov_i -\Pi_k^i (x_k^i \otimes x_k^i)$.
From \Cref{prop:mult_noise_fl}, taking expectation w.r.t. the $\sigma$-algebra $\Flast_{k-1}$, decomposing the multiplicative noise results to have:
\begin{align*}
    \Expec{\SqrdNrm{\xikmultFN[\eta]}}{\Flast_{k-1}} = \frac{1}{N^2} \sum_\iN \Expec{\SqrdNrm{\xi_k^{i, \mathrm{mult}}(\eta)}}{\Flast_{k-1}}\,.
\end{align*}

Next we use the result of \Cref{prop:mult_noise_lin} for each client $i$ in $\OneToN$ and we obtain:
\begin{align*}
    \Expec{\SqrdNrm{\xikmultFN[\eta]}}{\Flast_{k-1}} &\leq \frac{1}{N} \sum_\iN (\omgC +1)  R_i^2 \SqrdNrm{\xCov_i^{1/2} (w - \ws)} \\
    &\leq \frac{(\omgC +1)  \max_{i \in  \OneToN}(R_i^2)}{N^2}  \SqrdNrm{\frac{1}{N} \sum_\iN \xCov_i^{1/2} (w - \ws)} 
    \,,
\end{align*}
which allows concluding.

\end{proof}

\begin{property}[Validity of \Cref{asu:main:baniac_lin}]
\label{prop:baniac_lin_fl}
Considering \Cref{ex:dist_comp_LMS} under the setting of \Cref{model:centralized} with \Cref{rem:as_bounded_features,lem:compressor}, if the compressor $\C$ is linear, then 
for any $k$ in $\N^*$, we have $\aniac \preccurlyeq \sigma^2 \max_{i \in \OneToN} ( \Sha_{\xCov_i}) \overline{\xCov} / N$ and $\FullExpec{\Xi_k \Xi_k^\top} \preccurlyeq  \max_{i \in \OneToN} (R_i^2 \Sha_{\xCov_i}) \overline{H} / N$, with $(\Sha_{\xCov_i})_{i \in \OneToN}$ given in \Cref{cor:value_constants_in_thm}.
Overall, \Cref{asu:main:baniac_lin} is thus verified.
\end{property}

\begin{proof}

\textit{First inequality.} 

By \Cref{def:ania}, we have $    \aniac = \expec{\xikstaruk \otimes \xikstaruk}{\Flast_{k-1}} = \frac{1}{N^2}\sum_\iN \expec{\C_k^i(\gwkstar^i)^{\kcarre}}{\Flast_{k-1}} $, because for any client $i$ in $\OneToN$ $\bigpar{(\varepsilon_k^i)_{k\in \OneToK}}$ is independent from $\bigpar{(x_k^i)_{k\in\OneToK}}$ (\Cref{model:fed}) and using compressor linearity and \Cref{app:eq:bound_compressor_cov}, it gives:
\begin{align*}
    \aniac &= \sigma^2 \frac{1}{N^2}\sum_\iN \FullExpec{\C_k^i(x_k^i)^{\kcarre}} = \frac{\sigma^2}{N^2} \sum_\iN 
 \compCov[i][][] \preccurlyeq \frac{\sigma^2}{N^2} \sum_\iN \Sha_{\xCov_i}\xCov \\
 &\preccurlyeq \frac{\sigma^2 \max_{i \in  \OneToN}(\Sha_{\xCov_i}) }{N} \overline{\xCov} \,.
\end{align*}

\textit{Second inequality.}

Using \Cref{prop:mult_noise_lin_fl}, because the random mechanism $\C^i$ is linear, there exists two matrices $\Pi_k^i, \Xi_k^i$ in $\R^{d\times d}$ s.t. for any $z$ in~$\R^d$, we have $\C_k^i(z) = \Pi_k^i z$ and $\xi_k^\mathrm{mult, i}(z) = \Xi_k^i z = (\xCov_i -\Pi_k^i (x_k^i \otimes x_k^i) z$, which gives that $\Xi_k = \frac{1}{N} \sum_\iN \xCov_i -\Pi_k^i (x_k^i \otimes x_k^i)$.
It follows that:
\begin{align*}
    \Xi_k\Xi_k^\top &= \frac{1}{N^2} \sum_\iN (\Xi_k^i) (\Xi_k^i)^\top + \frac{1}{N^2} \sum_{i \neq j } (\Xi_k^i) (\Xi_k^j)^\top \,.\nonumber 
\end{align*}

Taking the $\sigma$-algebra $\Flast_{k-1}$, using that the $N$ compressions are independent (\Cref{ex:dist_comp_LMS}) and that for any $i$ in $\OneToN$, $\Expec{\xi_k^{i, \mathrm{mult}}}{\Flast_{k-1}}=0$ (\Cref{lem:compressor}) results to have $\expec{\Xi_k\Xi_k^\top}{\Flast_{k-1}} = \frac{1}{N^2} \sum_\iN \Expec{(\Xi_k^i) (\Xi_k^i)^\top}{\Flast_{k-1}}$.
Now, we can reuse the computations given in \Cref{prop:baniac_lin} to obtain $\Expec{(\Xi_k^i) (\Xi_k^i)^\top}{\Flast_{k-1}} \preccurlyeq   R_i^2 \Sha_{H_i} \xCov_i$. 
Therefore, we have $\Expec{\Xi_k \Xi_k^\top}{\Flast_{k-1}} \preccurlyeq \max_{i \in \OneToN} (R_i^2 \Sha_{\xCov_i}) \overline{\xCov} / N$,
which concludes the second part of the verification of \Cref{asu:main:baniac_lin}.

\end{proof} 
\subsection{Heterogeneous optimal point}
\label{app:subsec:fl_wstar}

In this section, we explore further the scenario of concept-shift by adding a memory mechanism \citep{mishchenko_distributed_2019}. This mechanism has been shown by \citet{philippenko_artemis_2020} to improve the convergence in the case of heterogeneous clients. We give below the updates equation defining the algorithm of distributed compressed LSR with memory.

\begin{algo}[Distributed compressed LMS with control variates]\label{ex:dist_comp_LMS_artemis}
Each client $i \in \OneToN$ maintains a sequence $(h_k^i)_{i \in \OneToN}$ in $\R^d$, observes at any step $k \in \OneToK$ an oracle $\g_k^i(\cdot)$ on the gradient of the local objective function $F_i$ and applies an independent random compression mechanism $\C_k^i(\cdot)$ to the difference $\g_k^i-h_k^i$. And for any step-size $\gamma>0$, any $k \in \N^*$, the sequence of iterates $(w_k)_{k \in\N}$ satisfies:
\begin{equation}
\label{eq:def_artemis}
\left\{\begin{aligned}
& w_k = \wkm - \frac{\gamma}{N} \sum_\iN \C_k^i(\gwki - h_{k-1}^i) + h_{k-1}^i\\
& h_{k}^i = h_{k-1}^i + \alpha \C_k^i(\gwki - h_{k-1}^i) \,,
\end{aligned}\right.
\end{equation}
with $\alpha = 1 / 2(\omgC + 1)$.
\end{algo}

The counterpart of adding memory is that the random fields are no more identically distributed, thus \Cref{def:class_of_algo} is not fulfilled, and results from \Cref{sec:theoretical_analysis} cannot be applied, especially because $\FullExpec{\xi_k^{\mathrm{add}} \otimes \xi_k^{\mathrm{add}}}$ changes along iterations.
To remedy this problem, we  define here the \emph{limit} of the covariance of the additive noise i.e. $\aniac^{\infty} = \underset{k \to +\infty}{\lim} \FullExpec{\xi_k^{\mathrm{add}} \otimes \xi_k^{\mathrm{add}}}$. In the following result, we establish an asymptotic result on the convergence, similar  to \Cref{thm:bm2013_with_nonlinear_operator_compression}.

\begin{theorem}[CLT for concept-shift heterogeneity]
\label{thm:asymptotic_normality_artemis}
Consider \Cref{ex:dist_comp_LMS_artemis} under \\ \Cref{model:fed}~with~$\mu > 0$ and \Cref{lem:compressor}, for any step-size $(\gamma_k)_{k\in \N^*}$ s.t. $\gamma_k = 1 / \sqrt{k}$. Then 

\begin{enumerate}[topsep=2pt,itemsep=0pt,leftmargin=0.5cm]
    \item $(\sqrt{K} \etabarkm)_{K>0} \xrightarrow[K \to +\infty]{\mathcal{L}} \mathcal{N}(0, \Fhess^{-1} \aniac^\infty \Fhess^{-1})$,
    \item $\aniac^\infty = \overline{\mathfrak{C}((\C^i, p_{\Theta_i'})_\iN)}$, where $p_{\Theta_i'}$ is the distribution of $\gwkstar^i - \nabla F_i(\ws)$.
\end{enumerate}
\end{theorem}

\Cref{thm:asymptotic_normality_artemis} shows that when using memory, in the settings of heterogeneous optimal points~$(\ws^i)_\iN$, convergence is still impacted by heterogeneity but with a smaller additive noise's covariance as  $\Theta_i' \prec \Theta_i$. In particular, in the case of deterministic gradients (batch case), we case $\Theta_i'\equiv 0$. 
From a technical standpoint, it shows that we recover asymptotically  the results stated by \Cref{thm:bm2013_with_nonlinear_operator_compression,thm:bm2013_with_linear_operator_compression} in the general setting of i.i.d. random fields $(\xik)_{k\in\N^*}$.
To prove this theorem, we show that the assumptions required by \Cref{app:thm:polyak_juditsky_92} are fulfilled by this framework.

\begin{proof}

For the sake of demonstration, we define a Lyapunov function $V_k$ \cite[as in][]{mishchenko_distributed_2019,liu_double_2020,philippenko_artemis_2020}, with $k$ in $\llbracket 1, K \rrbracket$:
\[
V_k = \SqrdNrm{\eta_k} + 2 \gamma_k^2 \cst \frac{1}{N} \sum_\iN \SqrdNrm{h_{k-1}^i - \nabla F_i(\ws)} \,,
\]
with $\cst$ in $\R^*$ being a Lyapunov constant that verifies some conditions given in Theorem S6 in \citet{philippenko_artemis_2020}. For any $k$ in $\N$, the Lyapunov function is defined combining two terms: (1) the distance from parameter $w_k$ to the optimal parameter $\ws$, (2) for any client $i$ in $\OneToN$, the distance between the memory term $h_{k-1}^i$ and the true gradient $\nabla F_i(\ws)$. 

\textbf{First, we have that in the case of decreasing step size s.t. for any $k$ in $\N$, $\gamma_k = k^{-\alpha}$, we have $\eta_{K} \xrightarrow[K \to +\infty]{L^2} 0$ and $h_{K}^i \xrightarrow[K \to +\infty]{L^2} \nabla F_i(\ws)$.}

Let $k$ in $\N^*$, from the demonstration of the \Artemis~algorithm with memory, we have from Theorem S6 in \citet[][]{philippenko_artemis_2020} (see page 41-45) that (1) combining Equations (S14) and (S15) in the form (S14)$+2\gamma_k^2 C $(S15), (2) and applying strong-convexity,
allows to obtain Equation (S17) but adapted to decreasing step-size:
\begin{align*}
    \Expec{V_{k}}{\Flast_{k-1}} &\leq \left(1 - 2 \gamma_k \mu \Box_k \right) \SqrdNrm{w_{k-1} - \ws} + \frac{2\gamma_k^2 \cst \diamondsuit}{N} \sum_{i=1}^N \lnrm h_{k-1}^i - \nabla F_i(\ws) \rnrm^2 +   \frac{2 \gamma_k^2 \sigma \triangle}{N}\,,
\end{align*}

with $\Box_k, \diamondsuit, \triangle$ being three constants in $\R$ whose exact expression can be found on pages 42-43 in \citet{philippenko_artemis_2020}. Furthermore, in the same article, they verify that to obtain a $(1- \gamma_k \mu)$ convergence, the following condition on $\Box_k, \diamondsuit, \triangle$ are fulfilled for any $k$ in $\N^*$: $\Box_k \leq 1/2$ and $\diamondsuit \leq 1 - \gamma_k \mu$.

These properties are valid under some conditions on the Lyapunov constant $\cst$, the step-size $\gamma_k$, and the learning rate $\alpha$; these conditions are provided in the statement of Theorem S6 from \citep{philippenko_artemis_2020} and we don't recall them here.
Hence, we can write that we have:
\begin{align*}
    \Expec{V_{k}}{\Flast_{k-1}} &\leq (1 - \gamma_k \mu) \bigpar{\SqrdNrm{w_{k-1} - \ws} + \frac{2\gamma_k^2 \cst}{N} \sum_{i=1}^N \lnrm h_{k-1}^i - \nabla F_i(\ws) \rnrm^2 } +   \frac{2 \gamma_k^2 \sigma^2 \triangle}{N}\,,
\end{align*}
and because for any $k$ in $N$, the step-size is decreasing, we have $\gamma_k \leq \gamma_{k-1}$, which makes to recover the Lyapunov function $V_{k-1}$ at step $k-1$: $    \Expec{V_{k}}{\Flast_{k-1}} \leq (1 - \gamma_k \mu) V_{k-1} +   \frac{2 \gamma_k^2 \sigma^2 \triangle}{N}$.
Taking full expectation and unrolling the sequence $(V_k)_{k\in \N}$, we obtain:
\begin{align*}
    \E V_{k} &\leq \prod_{i=1}^k (1 - \gamma_i \mu) V_{0} + \frac{2 \sigma^2 \triangle}{N} \sum_{j=1}^k \gamma_j^2 \prod_{i=j+1}^k (1 - \gamma_i \mu) \,.
\end{align*}

To show that each part of the bound given in the above equation tends to zero when $k$ grows to infinity is classical computations and can be find for instance in lectures notes of \citet[][pages 164-167 and 182]{bach_lecturesnotes_2022}, and \cite{kushner_stochastic_2003}.

To apply Theorem 1 from \citet[][recalled in \Cref{app:thm:polyak_juditsky_92}]{polyak_acceleration_1992}, which gives the desired result, it suffices to prove the convergence in probability of the covariance of the noise $\xik$ towards $\aniac$, as $k\to \infty$.

\textbf{In the following, we show that $\underset{k \to +\infty}{\lim} \Expec{ \xik \xik^\top}{\Flast_{k-1}} \stackrel{\mathbb{P}}{=}\aniac^\infty  $.}
Let $k$ in $\N^*$, for this purpose, we consider the following additive/multiplicative noise decomposition:
\begin{equation}
\label{eq:fl_wstar_decomposition_clt}
\left\{\begin{aligned}
& \xi_{k,*}^{\mathrm{A}} = - \frac{1}{N} \sum_\iN \C_k^i(\gwkstar^i - \nabla F_i(\ws)) \\
& \xi_k^{\mathrm{M}}(\eta_k) = \Fhess \eta_k - \frac{1}{N} \sum_\iN \C_k^i(\gwki - h_{k-1}^i) + \frac{1}{N} \sum_\iN \C_k^i(\gwkstar - \nabla F_i(\ws)) + h_{k-1}^i \,.
\end{aligned}\right.
\end{equation}

Furthermore, we have that $\xi_k^{\mathrm{add}} \xrightarrow[k \to +\infty]{L^2} \xi_{k,*}^{\mathrm{A}}$ because of the Hölder-inequality (\Cref{lem:compressor}) and because we shown that $h_{K}^i \xrightarrow[K \to +\infty]{L^2} \nabla F_i(\ws)$; thereby $\E[\xi_k^{\mathrm{add}} \otimes \xi_k^{\mathrm{add}}] \xrightarrow[k \to +\infty]{L^1} \aniac^\infty$. 
Next, from \Cref{eq:fl_wstar_decomposition_clt}, we write:
\begin{align*}
    \xik \xik^\top &= (\xi_{k,*}^{\mathrm{A}} - \xi_k^{\mathrm{M}}(\etakm)) (\xi_{k,*}^{\mathrm{A}} - \xi_k^{\mathrm{M}}(\etakm))^\top \\
    &= \xi_{k,*}^{\mathrm{A}} \otimes \xi_{k,*}^{\mathrm{A}} - \xi_{k,*}^{\mathrm{A}} \xi_k^{\mathrm{M}}(\etakm)^\top - \xi_k^{\mathrm{M}}(\etakm) (\xi_{k,*}^{\mathrm{A}})^\top + \xi_k^{\mathrm{M}}(\etakm) \otimes \xi_k^{\mathrm{M}}(\etakm) \,.
\end{align*} 

(i) Developing the covariance of the additive noise and considering \Cref{model:fed,ex:dist_comp_LMS} results to $\expec{\xi_{k,*}^{\mathrm{A}} \otimes \xi_{k,*}^{\mathrm{A}}}{\Flast_{k-1}} = \frac{1}{N^2} \sum_\iN \expec{\C_k^i(\gwkstar^i - \nabla F_i(\ws))^{\kcarre}}{\Flast_{k-1}}$.
For any iteration~$k$ in~$\N^*$ and any client~$i$ in~$\OneToN$, we note $\Theta_i'$ the covariance of $\gwkstar^i - \nabla F_i(\ws)$, then $\gwkstar^i - \nabla F_i(\ws)$ is an i.i.d. zero-centered random vectors draw from a distribution $p_{\Theta_i'}$, hence we have for any iteration~$k$ in~$\N^*$,~$\aniac^\infty = \expec{\xi_{k,*}^{\mathrm{A}} \otimes \xi_{k,*}^{\mathrm{A}}}{\Flast_{k-1}} = \overline{\mathfrak{C}(\C^i, (p_{\Theta_i'})_\iN)} \,.$

(ii) Second, we show that $\Expec{\xi_k^{\mathrm{M}}(\etakm)^{\kcarre}}{\Flast_{k-1}}$ converge to 0 in probability: it is sufficient to show that $\| \xi_k^{\mathrm{M}}(\etakm)^{\kcarre} \|_F$ tends to 0.
To do so, we use the fact that $\| \xi_k^{\mathrm{M}}(\etakm)^{\kcarre} \|_F = \|\xi_k^{\mathrm{M}}(\etakm)\|^2_ 2$, it results to the following decomposition:
\begin{align*}
    \| \xi_k^{\mathrm{M}}(\etakm)^{\kcarre} \| &\leq 3 \SqrdNrm{\xCov \etakm} + 3\SqrdNrm{ \frac{1}{N} \sum_\iN \C_k^i(\gwki - h_{k-1}^i) - \C_k^i(\gwkstar^i - \nabla F_i(\ws))} \\
    &\qqquad + 3\SqrdNrm{ \frac{1}{N} \sum_\iN h_{k-1}^i - \nabla F_i(\ws)} \,.
\end{align*}
Considering the Hölder inequality given 
in \Cref{item:holder_compressor} from \Cref{lem:compressor}, because $\eta_k \xrightarrow[k \to +\infty]{L^2} 0$ and $h_k^i \xrightarrow[k \to +\infty]{L^2} \nabla F_i(\ws)$, we deduce that $\Expec{ \xi_k^{\mathrm{M}}(\etakm)^{\kcarre}}{\Flast_{k-1}}$ tends to $0$ in $L^1$-norm. 

(iii) Third, it remains to show that $\expec{\xi_k^{\mathrm{M}}(\etakm) (\xi_{k,*}^{\mathrm{A}})^\top}{\Flast_{k-1}} \xrightarrow[k \to +\infty]{\mathbb{P}} 0$.
We use the Cauchy-Schwarz inequality's \ref{app:basic_ineq:cauchy_schwarz_cond} for conditional expectation:
\begin{align*}
    \Expec{\xi_k^{\mathrm{M}}(\etakm) (\xi_{k,*}^{\mathrm{A}})^\top \|_F}{\Flast_{k-1}}^2 &=  \Expec{\xi_k^{\mathrm{M}}(\etakm)\|_2\| (\xi_{k,*}^{\mathrm{A}})^\top \|_2}{\Flast_{k-1}}^2 \\
    &\leq \|\Expec{\xi_k^{\mathrm{M}}(\etakm)\|^2_2}{\Flast_{k-1}} \Expec{ \|(\xi_{k,*}^{\mathrm{A}})^\top\|^2_2}{\Flast_{k-1}} \,.
\end{align*}
The sequence of random vectors $(\xi_{k,*}^{\mathrm{A}})_{k\in \N}$ is i.i.d., and moreover we have shown previously that $\xi_k^{\mathrm{M}}(\etakm)^{\kcarre}$ tends to 0, hence $\expec{\xi_k^{\mathrm{M}}(\etakm) (\xi_{k,*}^{\mathrm{A}})^\top}{\Flast_{k-1}}$ converges to $0$ in distribution.
Consequently, noting $\Theta_i' = \fullexpec{\gwkstar^i - \nabla F_i(\ws))^{\kcarre}}$ we can state that:
\begin{align*}
    \Expec{\xik^{\kcarre}}{\Flast_{k-1}} \xrightarrow[k \to +\infty]{\mathbb{P}} \aniac^\infty =  \overline{\mathfrak{C}(\C^i, (p_{\Theta_i'})_\iN)}  \,.
\end{align*}
\end{proof} 
 
\end{document}